\documentclass{article} % For LaTeX2e
\usepackage{arxiv}

\usepackage{amsmath,amsfonts,bm}

\def\eqref#1{equation~\ref{#1}}
\def\1{\bm{1}}

\DeclareMathAlphabet{\mathsfit}{\encodingdefault}{\sfdefault}{m}{sl}
\SetMathAlphabet{\mathsfit}{bold}{\encodingdefault}{\sfdefault}{bx}{n}

\usepackage{hyperref}       % hyperlinks
\usepackage{url}            % simple URL typesetting
\usepackage{booktabs}       % professional-quality tables
\usepackage{amsfonts}       % blackboard math symbols
\usepackage{nicefrac}       % compact symbols for 1/2, etc.
\usepackage{microtype}      % microtypography
\usepackage{xcolor}         % colors

\usepackage{amsmath}
\usepackage{amssymb}
\usepackage{amsthm}
\usepackage{bm}
\usepackage{multirow}
\usepackage{color}
\usepackage{tcolorbox}
\usepackage[numbers,sort&compress]{natbib}
\usepackage{cite}
\usepackage{algorithm}
\usepackage{algorithmic}

\usepackage{enumitem}
\usepackage{balance}
\usepackage{smile}
\usepackage{ulem}
\usepackage{tabularx}
\usepackage{graphicx}
\usepackage{subcaption}

\usepackage{booktabs}
\usepackage{diagbox}
\usepackage{wrapfig}
\usepackage{graphicx}

\newcommand{\AtC}{\texttt{AtC}}

\definecolor{darkcerulean}{rgb}{0.03, 0.27, 0.49}
\definecolor{darkviolet}{rgb}{0.58, 0.0, 0.83}
\definecolor{iris}{rgb}{0.35, 0.31, 0.81}
\definecolor{mint}{rgb}{0.24, 0.71, 0.54}
\definecolor{awesome}{rgb}{1.0, 0.13, 0.32}
\definecolor{ufogreen}{rgb}{0.24, 0.82, 0.44}

\newcommand{\offer}[1]{\textcolor{black}{#1}}

\title{Aggregate-then-Calibrate for Human-centered Assessment with Theoretical Guarantees}

\author{
  \noindent \textbf{Zejun Xie}$^1$\thanks{\quad Equal contribution and shared co-first authorship.}~, ~\textbf{Xintong Li}$^2${\footnotemark[1]}~, ~\textbf{Guang Wang}$^3$\thanks{\quad Corresponding author.}~, ~\textbf{Desheng Zhang}$^1$ \\
  $^1$ Rutgers University \;\;\;
  $^2$ Renmin University of China  \;\;\;
  $^3$ Florida State University \\
  $^1$ \texttt{\{zx180, dz220\}@cs.rutgers.edu} \;\;\;
  $^2$ \texttt{lixintong@ruc.edu.cn}  \;\;\;
  $^3$ \texttt{guang.wang@fsu.edu}\\
}

\begin{document}
\maketitle
% \footnotetext[1]{Equal contribution.}

\begin{abstract}
Human-centered assessment tasks, which are essential for systematic decision-making, rely heavily on human judgment and typically lack verifiable ground truth. Existing approaches face a dilemma: methods using only human judgments suffer from heterogeneous expertise and inconsistent rating scales, while methods using only model-generated scores must learn from imperfect proxies or incomplete features. We propose Aggregate-then-Calibrate (\AtC), a two-stage framework that combines these complementary sources. Stage-1 aggregates heterogeneous comparative judgments into a consensus ranking $\hat{\pi}$ using a rank-aggregation model that accounts for annotator reliability. Stage-2 calibrates any predictive model’s scores by an isotonic projection onto the order $\hat{\pi}$, enforcing ordinal consistency while preserving as much of the model’s quantitative information as possible. Theoretically, we show: (1) modeling annotator heterogeneity yields strictly more efficient consensus estimation than homogeneity; (2) isotonic calibration enjoys risk bounds even when the consensus ranking is misspecified; and (3) \AtC\ asymptotically outperforms model-only assessment. Across semi-synthetic and real-world datasets, \AtC\ consistently improves accuracy and robustness over human-only or model-only assessments. Our results bridge judgment aggregation with model-free calibration, providing a principled recipe for human-centered assessment when ground truth is costly, scarce, or unverifiable. 
\end{abstract}

\section {Introduction}

Human-centered assessment tasks are essential for systematic decision-making in domains where ground truth is costly, unobservable, or only available in the future. Consider two scenarios: (1) Delivery platforms seek to estimate worker workload for fair compensation, yet true energy expenditure, though measurable via wearables, is impractical to collect at scale, necessitating reliance on worker judgments~\citep{cite-key,8663457,ho2012online,RAICA}. (2) Conference committees must evaluate paper quality for acceptance decisions, yet the true quality only becomes observable years later through future impact, requiring committees to aggregate reviewer assessments in the present~\citep{NEURIPS2021_eaf76caa,dolati2019systematic,10.1145/3490486.3538235,tran2021an}. These tasks share key characteristics: they seek standardized assessments from human judgments rather than individual preference satisfaction~\citep{10.1145/1102351.1102369,chau2022learning,hu2022explaining}, and while ground truth exists conceptually, such as workload or future impact, it cannot be directly observed when decisions are needed~\citep{alur2024human,NEURIPS2024_2375085c}.

Existing approaches tend to rely either solely on human judgments or solely on model predictions. Methods that use only annotator-provided comparative judgments, such as judgment aggregation algorithms~\citep{10.1145/2433396.2433420,ZhaoLWKMSX18,jin2020rank}, can weight annotators by expertise but remain confounded by inconsistent rating scales. For example, a lenient expert and a strict novice might agree on an item's relative quality yet assign very different absolute scores~\citep{ito2024mitigatingcognitivebiasesmulticriteria,khurana2024crowdcalibrator}. Conversely, methods that use only model-generated scores from proxy labels face a supervision crisis: the true target values may correlate with latent factors that are impractical to measure at scale, such as cognitive load in task difficulty estimation~\citep{pmlr-v48-xua16,guo2023towards,wang2023gcrl,lai2022loyalty}. This forces models to learn from noisy surrogates, propagating systematic biases into the assessments.

Intuitively, human judgments are easy to solicit but lack a universal scale, whereas model predictions are consistent by design but require supervised signals~\citep{cowgill2020bias, clinical-v-actuarial-1989}. Our goal is to combine the strengths of both sources: use the scale from the model assessment and the structure from human judgments~\citep{Hemmer02112025}. A key \textbf{insight} is that people are typically more reliable at comparing items than at absolute scoring (also known as Weber–Fechner laws in psychophysics~\citep{Weber2019}). We therefore extract only the ordinal information from the noisy human inputs, rather than trusting their raw rating values or aggregated scores.

We propose a two-stage framework called Aggregate-then-Calibrate (\AtC) that integrates the complementary advantages of human judgments and model assessments. In the aggregation stage, \AtC\ employs a heterogeneous rank aggregation approach~\citep{jin2019coride,10.1145/2433396.2433420,10.14778/3407790.3407855} to derive a consensus ranking of items from multiple annotators, explicitly accounting for each annotator’s individual scale and reliability. In the calibration stage, \AtC\ uses isotonic regression~\citep{6e0673c1-388c-3c59-aece-5a7736061f87,10.2307/2345150,NEURIPS2021_eaf76caa} to adjust an initial predictive model’s assessments so that they align with the consensus ranking. This two-step approach ensures that the final assessments are monotonic with respect to the human consensus judgments. By combining the ordinal structure (from human) with the scoring scale (from the model), we theoretically and empirically demonstrate that \AtC\ leverages the best of both: the human judgments provide a reliable ranking backbone, and the model provides consistent scoring, resulting in assessments that outperform either approach alone.

Our contributions are summarized as follows:
\begin{itemize}[leftmargin=*]
    \item \textbf{Conceptual:} We formalize a specific class of human-centered assessment problems where judgment aggregation (not preference optimization) is central for systematic decision-making. \AtC\ addresses these via a human-model complementarity: aggregated comparisons provide ordinal constraints, while models estimate metrically scaled scores. Moreover, it can be generalized to any predictive model and thus enables the use of any off-the-shelf predictive model without modification.
    \item \textbf{Theoretical:} We provide a comprehensive analysis for \AtC\ from three perspectives. First, to the best of our knowledge, Theorem~\ref{thm:cov-compare} offers the first proof that heterogeneous rank aggregation models are strictly more statistically efficient than homogeneous methods when annotator abilities vary, validating our approach for the initial aggregation stage. Second, the risk analysis in Theorem~\ref{thm:Robustness} makes novel contributions to isotonic regression theory beyond prior work~\citep{10.1214/aos/1021379864,chatterjee2015risk,bellec2018sharp} by accounting for projection onto a random cone and managing biased effective noise, which are two necessary conditions for our problem setting. Finally, these components culminate in our optimality guarantee in Theorem~\ref{thm:optimality}, formally demonstrating the power of integrating human ordinal judgments with model-based scores.
    \item \textbf{Empirical:} 
    \AtC\ consistently outperforms human-only and model-only assessments across semi-synthetic and real-world datasets, demonstrating superior accuracy and robustness under varying levels of data degradation. 
\end{itemize}
\section{Aggregate-then-Calibrate Framework}
\subsection{Overview}
We consider a human-centered assessment task for $n$ items using inputs from $m$ human annotators and a predictive model $p$. Let $\mathbf s \in \mathbb{R}^n$ denote the true (unobserved) latent scores of the $n$ items, representing their ideal ground-truth quality (e.g., the true proficiency of a student’s answer or the true merit of a product). Since $\mathbf s$ cannot be directly observed, we collect two complementary sources of information to estimate it:
\begin{itemize}[leftmargin=*]
\item \textbf{Human Judgments:} Annotators provide comparative assessments (e.g., pairwise preferences or ratings) for various subsets of the items. These judgments are aggregated to produce a consensus score vector, which we denote by $\tilde{\mathbf s} \in \mathbb{R}^n$. The vector $\tilde{\mathbf s}$ represents the aggregated consensus over item scores derived purely from the given human input. We model $\tilde{\mathbf s}$ as a noisy observation of the true scores $\mathbf s$; specifically, we assume $\tilde{\mathbf s} = \mathbf s + \tilde{\boldsymbol \epsilon}$, where $\tilde{\boldsymbol \epsilon} \sim \mathcal{N}(0, \tilde{\sigma}^2 I_n)$. In practice, $\tilde{\mathbf s}$ is not observed directly; instead, we estimate this consensus via the Stage-1 procedure. Let ${\mathbf s}^* \in \mathbb{R}^n$ denote the Stage-1 estimate of $\tilde{\mathbf s}$ (as described in Stage-1 below). Let $\hat{\pi}$ be the consensus ranking of the $n$ items, obtained by sorting the entries of ${\mathbf s}^*$ from lowest to highest. This ranking $\hat{\pi}$ reflects the collective ordering of items according to the annotators’ comparative judgments.
\item \textbf{Model Assessments:} A predictive model $p$ (e.g., a machine learning model analyzing item features) produces an initial score vector $\mathbf s_p \in \mathbb{R}^n$ for the items. We do not assume that $\mathbf s_p$ is perfectly calibrated or aligned with the human judgments; indeed, the model may have been trained on proxy labels or may even be unsupervised due to the lack of ground truth. We model $\mathbf s_p$ as a biased observation of the true scores, subject to an unknown systematic offset: specifically, $\mathbf s_p = \mathbf s + \boldsymbol{\nu}$, where $\boldsymbol{\nu} \in \mathbb{R}^n$ is an unknown error (or bias) vector. Given the relationship $\tilde{\mathbf s} = \mathbf s + \tilde{\boldsymbol \epsilon}$ from above, we can relate the model scores to the consensus by writing $\mathbf s_p = \tilde{\mathbf s} - \tilde{\boldsymbol \epsilon} + \boldsymbol{\nu}$. For simplicity, we assume that the model’s noise is independent of the annotators’ noise (i.e., $\boldsymbol{\nu}$ is independent of $\tilde{\boldsymbol \epsilon}$).
\end{itemize}
Our goal is to produce a final calibrated score vector $\hat{\mathbf s} \in \mathbb{R}^n$ for the $n$ items. The calibrated scores $\hat{\mathbf s}$ should (i) respect the consensus ordering $\hat{\pi}$ (to maintain agreement with collective human judgment), and (ii) remain as close as possible to the model’s original scores $\mathbf s_p$ (to preserve the model’s useful quantitative information). We formalize this goal by defining the AtC estimator as the projection of $\mathbf s_p$ onto the set of score vectors that are consistent with $\hat{\pi}$. Let
$\hat{\mathcal{M}} := \{\mathbf{y}\in\mathbb{R}^n:~y_{\hat{\pi}(1)} \le y_{\hat{\pi}(2)} \le \cdots \le y_{\hat{\pi}(n)}\}$ 
be the set of all length-$n$ vectors that are nondecreasing according to the ranking $\hat{\pi}$. The AtC estimate is then given by:
\begin{equation}\label{eq:projection}
\hat{\mathbf{s}} \;=\; \Pi_{\hat{\mathcal{M}}}(\mathbf{s}_p) \;=\; \arg\min_{\mathbf{y}\in \hat{\mathcal{M}}} \;\|\mathbf{y} - \mathbf{s}_p\|_2^2\,,
\end{equation}
where $\Pi_{\hat{\mathcal{M}}}(\mathbf{s}_p)$ denotes the Euclidean projection of $\mathbf s_p$ onto the constraint set $\hat{\mathcal{M}}$. By construction, $\hat{\mathbf s}$ is the closest vector to $\mathbf s_p$ (in $\ell_2$-distance) that does not violate the consensus ordering. In other words, $\hat{\mathbf s}$ is obtained by an isotonic regression of the model’s scores onto the ordering $\hat{\pi}$. 
Figure~\ref{fig:main_fig} illustrates the roles of $\tilde{\mathbf s}$, $\mathbf s_p$, and $\hat{\mathbf s}$ in \AtC.
% , with pseudocode provided in Appendix~\ref{app:alg}.

% \vspace*{-0.1in}
\begin{figure}[h]
\centering
\includegraphics[width=\textwidth]{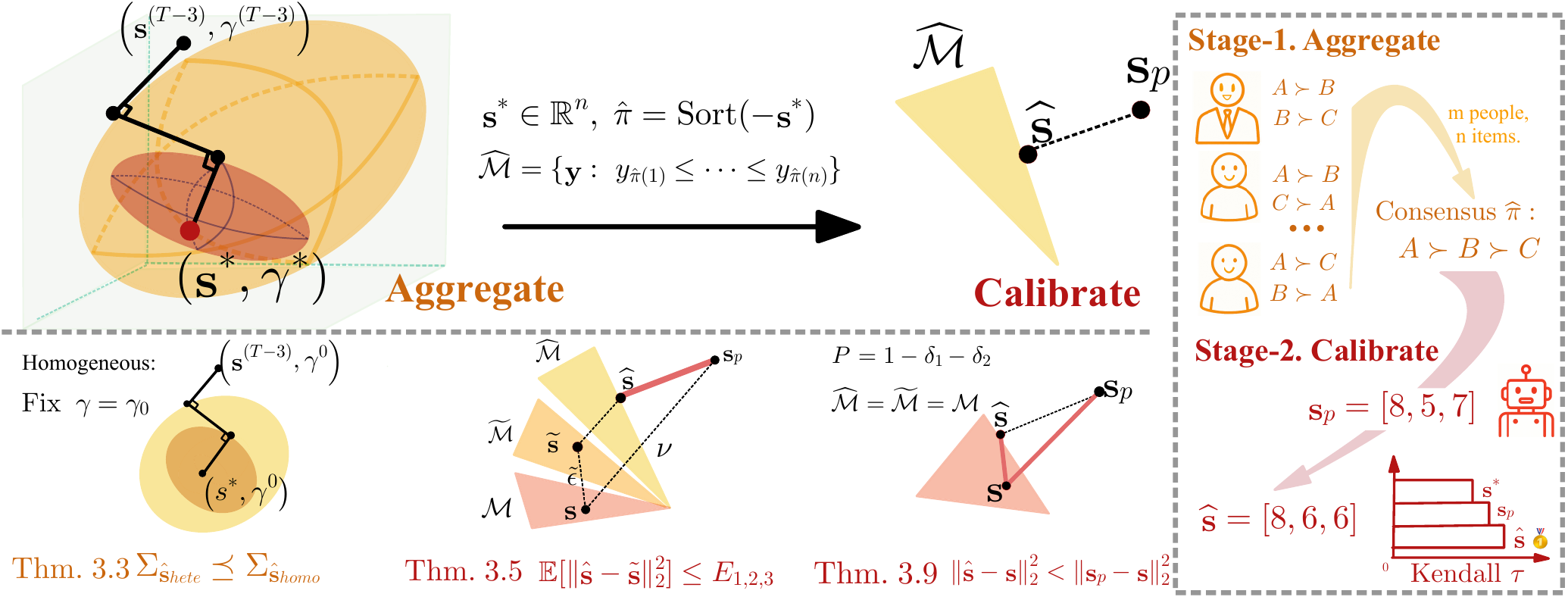}
\caption{\offer{\AtC\ Framework and Its Theoretical Guarantees.\\
\small Stage-1 aggregates human judgments into a ranking $\hat{\pi}$; Stage-2 calibrates model scores $\mathbf s_{p}$ via isotonic projection onto $\mathcal{M}$ to obtain $\hat{\mathbf s}$. Theorems~\ref{thm:cov-compare}, ~\ref{thm:Robustness}, and~\ref{thm:optimality} provide efficiency, risk, and superiority guarantees.}}
\label{fig:main_fig}
% \vspace*{-0.3in}
\end{figure}

\subsection{Stage-1: Judgment Aggregation under Heterogeneous Thurstone Model}
In Stage-1, we aggregate the human judgments into a single consensus ranking. We formulate this rank aggregation problem as finding a latent score vector (the consensus) that best explains the annotators’ comparative judgments. Each annotator $u \in \{1,\dots,m\}$ provides pairwise preferences over certain item pairs (for example, $i \succ j$ indicates annotator $u$ prefers item $i$ over item $j$). \offer{A key design choice is how to handle observed inconsistencies when different annotators judge the same items. When couriers assess delivery route difficulty, their judgments differ—but this reflects heterogeneity in expertise rather than fundamental subjective preferences. Rather than filtering novice opinions, our approach learns each annotator's reliability $\gamma_u$ from data and applies optimal weighting to produce consensus closer to the latent target.}
\paragraph{Preliminary: Heterogeneous Thurstone Model (HTM).}We assume these preferences follow a Thurstone choice model, extended to allow annotator-specific noise levels—this is HTM \citep{jin2020rank}. In particular, given any two items $i$ and $j$ that annotator $u$ compares, the probability that $u$ prefers $i$ to $j$ is modeled as: $\Pr\{u:~i \succ j\} \;=\; F\!(\gamma_{u}\,(\mathbf s_i - \mathbf s_j))\,$, where $F$ is a fixed symmetric cumulative distribution function (e.g., the standard normal CDF in Thurstone’s original formulation, or a logistic CDF in a Bradley–Terry model), and $\gamma_u > 0$ is annotator $u$’s precision (or consistency) parameter. A larger $\gamma_u$ means annotator $u$ is less noisy (more reliable) in their pairwise comparisons, whereas a smaller $\gamma_u$ implies higher variance (lower consistency) for that annotator. 
To infer the consensus scores from the observed comparisons, maximum-likelihood estimation (MLE) is performed under the HTM. 

Let $\mathcal{D}_{u}$ denote the set of observed pairwise comparisons from annotator $u$, where $i \succ j \in \mathcal{D}_{u}$ means $u$ preferred item $i$ over $j$. The log-likelihood of all annotators' data, given parameters $(\mathbf{s}, \mathbf{\gamma})$, is:
$\ell(\mathbf{s}, \mathbf{\gamma}) \;=\; \sum_{u=1}^{m} \;\sum_{\,i \succ j\,\in\,\mathcal{D}_{u}} \log\,F\!(\gamma_{u}\,(\,\mathbf s_i - \mathbf s_j\,))\,$,
with appropriate constraints on $(\mathbf{s}, \mathbf{\gamma})$ for identifiability (e.g., one can fix $\frac{1}{n}\sum_{i=1}^{n} \mathbf s_i = 0$ or constrain $\frac{1}{m}\sum_{u=1}^{m}\gamma_u = 1$). 
The consensus score estimates $\mathbf{s^*} = (\mathbf s^*_1,\ldots,\mathbf s^*_n)$ and the annotator reliabilities $\mathbf{\gamma}$ are obtained by maximizing this log-likelihood:
$\max_{\mathbf{s},\,\mathbf{\gamma}} \;\; \ell(\mathbf{s}, \mathbf{\gamma})\,$,
i.e. by finding the parameters that best explain the observed comparisons. This optimization can be carried out with iterative methods; for example, one can alternate between updating the score estimates $\mathbf{s}$ and the annotator precisions $\mathbf{\gamma}$ until convergence (each update can use gradient ascent on $\ell$ or other numerical routines).

The result of Stage-1 is an estimated consensus score vector $\mathbf s^{*}$. Sorting the components of $\mathbf s^{*}$ from lowest to highest yields the consensus ranking $\hat{\pi}$. In other words, $\hat{\pi}$ is the permutation of item indices such that $\mathbf s^{*}_{\hat{\pi}(1)} \le \mathbf s^{*}_{\hat{\pi}(2)} \le \cdots \le \mathbf s^{*}_{\hat{\pi}(n)}$. This Stage-1 outcome $\hat{\pi}$ reconciles the varying, potentially inconsistent human judgments into a single ranked list of items. 

\begin{algorithm}[H]
\label{alg:HRA}
\caption{Heterogeneous Rank Aggregation}
\begin{algorithmic}[1]
\STATE \textbf{Input:} Comparative judgment data $\{i \succ j \in \mathcal{D}_u:\; u=1,\ldots,m\}$ from $m$ annotators.
\STATE \textbf{Output:} Consensus score vector $\tilde{\mathbf{s}}$ and consensus ranking $\hat{\pi}$.
\STATE Initialize item scores $\mathbf s_i^{(0)} \leftarrow 0$ for all $i \in [n]$, annotator parameters $\gamma_u^{(0)} \leftarrow 1$ for all $u \in [m]$.
\FOR{$t = 1,2,\dots$ \textbf{until} convergence}
    % \STATE Compute  $\nabla_{\mathbf{s}}\,\ell\!\Big(\mathbf{s}^{(t-1)}, \gamma^{(t-1)}\Big)$ and $\nabla_{\gamma}\,\ell\!\Big(\mathbf{s}^{(t-1)}, \gamma^{(t-1)}\Big)$.
    \STATE Update score estimates: \\
    $\mathbf{s}^{(t)} \leftarrow \mathbf{s}^{(t-1)} \,+\, \eta_s \,\nabla_{\mathbf{s}}\ell\!\Big(\mathbf{s}^{(t-1)}, \gamma^{(t-1)}\Big)$.
    \STATE Identifiability: $\mathbf{s}^{(t)}=(\Ib-\one\one^{\top}/n) \mathbf{s}^{(t)}$.
    \STATE Update annotator parameters: \\ $\gamma_u^{(t)} \leftarrow \gamma_u^{(t-1)} \,+\, \eta_\gamma \,\nabla_{\gamma_u}\ell\!\Big(\mathbf{s}^{(t-1)}, \{\gamma_u^{(t-1)}\}\Big)$ for all $u \in [m]$.
    \STATE (Optionally, project $\{\gamma_u^{(t)}\}$ onto its feasible set, e.g., $\frac{1}{m}\sum_{u}\gamma_u^{(t)} = 1$.)
\ENDFOR
\STATE $\mathbf{s}^* \leftarrow \mathbf{s}^{(t_{\text{final}})}$ \hfill 
\STATE $\hat{\pi} \leftarrow \text{rank\_order}(\mathbf{s}^*)$ \hfill 
\STATE \textbf{return} $\hat{\pi}$.
\end{algorithmic}
\end{algorithm}

\subsection{Stage-2: Assessment Calibration via Isotonic Regression}
In Stage-2, we leverage the consensus ranking $\hat{\pi}$ from Stage-1 to calibrate the model’s raw scores $\mathbf s_p$. The key idea is to adjust the model’s scores so that the final outputs $\hat{\mathbf s}$ respect the consensus ordering $\hat{\pi}$ while remaining as faithful as possible to $\mathbf s_p$. We use the term \emph{calibration} here to mean rescaling or aligning the model-generated scores with the human-derived ranking constraints. This notion of calibration differs from its conventional meaning in probability estimation or psychology: instead of aligning predicted probabilities with observed frequencies, here we enforce an ordinal consistency between $\mathbf s_p$ and the consensus ranking. In practice, we implement this calibration by projecting $\mathbf s_p$ onto the monotonicity constraint set defined by $\hat{\pi}$, as formalized in Eq.~(\ref{eq:projection}).

\paragraph{Preliminary: Isotonic regression.} As shown in Eq.~(\ref{eq:projection}), the calibrated output can be obtained by solving an isotonic regression problem: we seek the adjusted score vector $\hat{\mathbf s}$ in $\hat{\mathcal{M}}$ that minimizes the squared error to $\mathbf s_p$. Equivalently, $\hat{\mathbf s}$ is the solution to
$\min_{\mathbf{x} \in \hat{\mathcal{M}}} \sum_{i=1}^n (x_i - \mathbf s_{p,i})^2$,
which is exactly the projection defined in Eq.~(\ref{eq:projection}). This optimization finds the closest score vector to $\mathbf s_p$ that does not violate the consensus ranking order. The solution to this isotonic regression is efficiently computed via the Pool-Adjacent-Violators (PAV) algorithm \citep{de2010isotone}. PAV iteratively scans through the scores sorted by $\hat{\pi}$ to detect any adjacent pair that is out of order (i.e., violates the required non-decreasing condition), and fixes such violations by averaging the offending scores. This process repeats until no more violations remain. The result is the calibrated score vector $\hat{\mathbf s}$, which by construction satisfies $\hat{\mathbf s}_{\hat{\pi}(1)} \le \hat{\mathbf s}_{\hat{\pi}(2)} \le \cdots \le \hat{\mathbf s}_{\hat{\pi}(n)}$ (i.e., it respects the consensus ranking).

Notably, if the model’s score vector $\mathbf s_p$ already happens to be perfectly consistent with $\hat{\pi}$, the isotonic regression will leave it unchanged. Otherwise, the model’s scores are adjusted in a minimally invasive way (with respect to squared-error) to remove any ranking inconsistencies. The calibrated scores $\hat{\mathbf s}$ can be interpreted as a revised set of model-generated scores that combine the model’s quantitative estimates with the human-derived ordinal constraints.

    \begin{algorithm}[H]
    \label{alg:IRC}
    \caption{Isotonic Regression Calibration}
    \begin{algorithmic}[1]
    \STATE \textbf{Input:} Model score vector $\mathbf{s}_p = (\mathbf s_{p,1},\ldots,\mathbf s_{p,n})$; consensus ranking $\hat{\pi}$ over $n$ items.
    \STATE \textbf{Output:} Calibrated score vector $\hat{\mathbf{s}} = (\hat{\mathbf s}_1,\ldots,\hat{\mathbf s}_n)$ that is monotonic w.rt.\ $\hat{\pi}$.
    \STATE Relabel item indices of $\mathbf{s}_p$ according to $\hat{\pi}$ (so that index $1$ corresponds to $\hat{\pi}(1)$, etc.). Let $(y_1,\ldots,y_n)$ be the reordered scores, where $y_1 = \mathbf s_{p,\hat{\pi}(1)}$, $y_2 = \mathbf s_{p,\hat{\pi}(2)}$, ..., $y_n = \mathbf s_{p,\hat{\pi}(n)}$.
    \STATE Initialize $\hat{\mathbf s}_i \leftarrow y_i$ for $i=1,\ldots,n$.
    \REPEAT 
        \FOR{$i = 1$ to $n-1$}
            \IF{$\hat{\mathbf s}_i > \hat{\mathbf s}_{i+1}$}
                \STATE $\hat{\mathbf s}_i, \hat{\mathbf s}_{i+1} \leftarrow \frac{\hat{\mathbf s}_i + \hat{\mathbf s}_{i+1}}{2}$  \hfill 
            \ENDIF
        \ENDFOR
    \UNTIL{$\hat{\mathbf s}_1 \le \hat{\mathbf s}_2 \le \cdots \le \hat{\mathbf s}_n$}
    \STATE Undo the reordering: for each item $j$, set $\hat{\mathbf s}_j$ to the value of the calibrated score assigned to item $j$'s position in the sorted order.
    \STATE \textbf{return} $\hat{\mathbf{s}}$.
    \end{algorithmic}
    \end{algorithm}
    
\section{Main Theoretical Results}
\label{sec:theory}
As illustrated in Figure~\ref{fig:main_fig}, \AtC's intuitive design offers strong theoretical guarantees for human-centered assessment across three perspectives: heterogeneity, robustness, and optimality. In this section, we present the main results, with related work discussions in \offer{Section~\ref{sec:related}} and detailed assumptions and proofs in Appendix~\ref{Appendix-of-cov}, ~\ref{Appendix:th2}, and~\ref{sec:improvement_probability}.

\subsection{Heterogeneous Efficiency Guarantee}

We now demonstrate the efficiency advantage of the HTM estimator over an estimator derived from a misspecified homogeneous model. When annotators exhibit varying levels of reliability, a correctly specified HTM is expected to yield score estimates with lower asymptotic variance compared to a homogeneous model that incorrectly assumes uniform annotator reliability. The following lemmas formalize the asymptotic covariance structures of the two estimators.  Specifically, $\hat{\theta}_{hete} = (\hat{\mathbf{s}}_{hete}, \hat{\gamma}_{hete})$ is the MLE for the HTM. $\hat{\mathbf{s}}_{homo}$ is the QMLE for the homogeneous model with fixed $\gamma_0=1$.
To establish the asymptotic properties of the estimators $\hat{\mathbf{s}}_{hete}$ and $\hat{\mathbf{s}}_{homo}$, we assume:

\begin{assumption}[Independent Observations]
\label{ass:thm33_independence}
The pairwise comparison outcomes $Y_{ul}$ are independent conditional on the true parameters.
\end{assumption}

\begin{assumption}[True Data Generating Process]
\label{ass:thm33_dgp}
The data are generated by a heterogeneous Thurstone model (HTM) with true parameters $\theta_0 = (\mathbf{s}^*, \gamma^*) \in \mathbb{R}^{n+m}$. To ensure identifiability of $\mathbf{s}^*$, we impose the constraint $\mathbf{1}^\top \mathbf{s}^* = 0$, where $\mathbf{1}$ denotes the all-ones vector. The annotator-specific parameter vector $\gamma^* = (\gamma_1^*, \ldots,\gamma_m^*)^\top$ contains at least two distinct values. The link function $F(\cdot)$ is assumed to be known and twice continuously differentiable.
\end{assumption}

\begin{assumption}[Regularity Conditions]
\label{ass:thm33_regularity}
Both the HTM and the homogeneous model satisfy the standard regularity conditions in \citep[Assumptions 1--6]{white1982maximum}, which ensure the consistency and asymptotic normality of the MLE and QMLE, respectively. 
These conditions include compactness of the parameter space, measurability and smoothness of the log-likelihood functions, existence of the required expectations, unique identifiability of the parameters (or QMLE limits), and non-singularity of the relevant matrices on the identifiable parameter subspace. 
Such conditions are typically satisfied by common link functions, such as the logistic CDF and the Gaussian CDF, provided that the parameter space is well-defined and the comparison design is non-degenerate.
Furthermore, for the correctly specified HTM, of \citep[Assumption 7]{white1982maximum} is assumed to hold at $\theta_0$, ensuring the information matrix equivalence.
\end{assumption}

\begin{lemma}
\label{lem:hetero-cov}
Under conditions in Assumption~\ref{ass:thm33_independence},~\ref{ass:thm33_dgp} and~\ref{ass:thm33_regularity}, the score vector estimator $\hat{\mathbf{s}}_{hete}$ obtained from the MLE of the HTM has an asymptotic covariance matrix given by:
\begin{align}
{\Sigma_{\hat{\mathbf{s}}_{hete}}} \;=\; \frac{1}{N}\Big(I_{ss} - I_{s\gamma}\,I_{\gamma\gamma}^{-1} I_{\gamma s}\Big)^{+}\!,
\end{align}
where $N=\sum_{u = 1}^m k_u$ is the total number of observations. The matrices $I_{ss}$, $I_{\gamma\gamma}$ and $I_{s\gamma}$ are the Fisher information sub-matrices for $\mathbf{s}$ and $\boldsymbol{\gamma}$. $(I_{ss} - I_{s\gamma}\,I_{\gamma\gamma}^{-1} I_{\gamma s})^{+}$ is the Schur complement of the full Fisher information with respect to $\mathbf{s}$, representing the effective Fisher information for $\mathbf{s}$ after accounting for the estimation of $\boldsymbol{\gamma}$. The pseudoinverse addresses the identifiability constraint $\mathbf{1}^\top \mathbf{s} = 0$.
\end{lemma}

\begin{lemma}
\label{lem:homo-cov}
Under conditions in Assumption~\ref{ass:thm33_independence},~\ref{ass:thm33_dgp} and~\ref{ass:thm33_regularity}, consider a homogeneous Thurstone model assuming a known constant annotator accuracy $\gamma_0$. The quasi-MLE $\hat{\mathbf{s}}_{homo}$, derived from this potentially misspecified model, has an asymptotic covariance matrix given by the classic sandwich form:
\begin{align}
{\Sigma_{\hat{\mathbf{s}}_{homo}}} \;=\; \frac{1}{N}\Big(A_{homo}(\mathbf{s}_*)^{+} B_{homo}(\mathbf{s}_*) A_{homo}(\mathbf{s}_*)^{+}\Big)\!,
\end{align}
where $\mathbf{s}_*$ is the probability limit of $\hat{\mathbf{s}}_{homo}$. The matrix $A_{homo}(\mathbf{s}_*)$ is the negative of the expected Hessian and $B_{homo}(\mathbf{s}_*)$ is the expected outer product of the scores from the misspecified model's log-likelihood, with expectations taken under the true heterogeneous data-generating process. The pseudoinverse also addresses identifiability.
\end{lemma}
Lemmas~\ref{lem:hetero-cov} and \ref{lem:homo-cov} allow us to formally compare the asymptotic efficiency of the two estimation strategies. The following theorem states that the estimator from the correctly specified HTM is asymptotically superior. The proof is given in Appendix~\ref{app:proof_lemma3.1} and~\ref{app:proof_lemma3.2}.

\begin{theorem}
\label{thm:cov-compare}
\textbf{(Heterogeneous Efficiency Guarantee)}.
Let $\Sigma_{\hat{\mathbf{s}}_{hete}}$ and $\Sigma_{\hat{\mathbf{s}}_{homo}}$ be the asymptotic covariance matrices of the score estimators as defined in Lemmas~\ref{lem:hetero-cov} and \ref{lem:homo-cov}. Then,
\begin{align}
\Sigma_{\hat{\mathbf{s}}_{hete}} \;\preceq\; \Sigma_{\hat{\mathbf{s}}_{homo}} \!,
\end{align}
where $\preceq$ denotes the Loewner order. This implies that the matrix ${\Sigma_{\hat{\mathbf{s}}_{homo}}} - {\Sigma_{\hat{\mathbf{s}}_{hete}}}$ is positive semi-definite. Furthermore, if the true annotator accuracies $\boldsymbol{\gamma}^*$ are not all equal (i.e., genuine heterogeneity exists), then this inequality is strict within the identifiable subspace, meaning ${\Sigma_{\hat{\mathbf{s}}_{homo}}} - {\Sigma_{\hat{\mathbf{s}}_{hete}}}$ is positive definite on this subspace.
\end{theorem}

\paragraph{Proof Sketch.}
The proof compares the asymptotic covariance of two estimators under the true heterogeneous Thurstone data-generating process. For the correctly specified heterogeneous model, the MLE $(\hat{\mathbf{s}}_{\mathrm{hete}}, \hat{\gamma}_{\mathrm{hete}})$ is asymptotically normal and achieves the Cram\'er--Rao lower bound. In Lemma~\ref{lem:hetero-cov}, since the score parameters $\mathbf{s}$ and the user-specific accuracy parameters $\gamma$ are estimated jointly, the effective information for $\mathbf{s}$ is given by the Schur complement $S = I_{ss} - I_{s\gamma} I_{\gamma\gamma}^{-1} I_{\gamma s}$, and the asymptotic covariance of $\hat{\mathbf{s}}_{\mathrm{hete}}$ on the identifiable subspace $\{\mathbf{v}:\mathbf{1}^\top \mathbf{v}=0\}$ is $S^{+}$.
While in Lemma~\ref{lem:homo-cov}, the homogeneous estimator $\hat{\mathbf{s}}_{\mathrm{homo}}$ is a QMLE under model misspecification, because it ignores genuine user-level heterogeneity in $\gamma_u$. By White's misspecification theory, it converges to a pseudo-true parameter $\mathbf{s}_*$ rather than the true score vector $\mathbf{s}^*$ in general, and its asymptotic covariance takes the sandwich form $A_{\mathrm{homo}}^{+} B_{\mathrm{homo}} A_{\mathrm{homo}}^{+}$,
instead of the inverse Fisher information. The failure of the information matrix equality under misspecification is precisely what creates the efficiency gap.
Comparing these two covariance expressions shows that the misspecified homogeneous estimator cannot be more efficient than the correctly specified heterogeneous estimator in the Loewner order. When the annotator accuracies are genuinely non-identical, this dominance is strict on the identifiable subspace, which yields the theorem. The full proof is in Appendix~\ref{app:proof_thm3.3}.

Theorem~\ref{thm:cov-compare} demonstrates that when annotator heterogeneity is present, the correctly specified HTM yields a strictly more efficient estimator of the true scores than the misspecified homogeneous model. Explicitly modeling annotator-specific differences provides a demonstrable advantage in estimation precision. 

\subsection{Robustness Guarantee under Model Misspecification}

We next analyze how errors from the Stage-1 rank aggregation affect the final calibrated output in Stage-2: an upper bound on the expected squared error $\mathbb{E}[\|\hat{\mathbf{s}} - \tilde{\mathbf{s}}\|_2^2]$. This bound quantifies the performance of the two-stage procedure, where $\hat{\mathbf{s}}$ is the projection of model predictions $\mathbf{s}_p$ onto the human-consensus-derived cone $\hat{\mathcal{M}}$, aiming to estimate the subjective optimal scores $\tilde{\mathbf{s}}$. Our analysis accounts for the misspecification of $\hat{\mathcal{M}}$ and the systematic bias $\boldsymbol{\nu}$ in $\mathbf{s}_p$ relative to $\tilde{\mathbf{s}}$.

We first establish a sharp oracle inequality for least squares estimators, leveraging a lemma adapted from ~\citep[Proposition 2.1]{bellec2018sharp}. This result bounds the estimation error based on the effective noise properties and the constraint set's geometry.

\begin{lemma}[Oracle Inequality with Biased Effective Noise]
\label{lem:adapted_bellec_prop2.1_main}
Let $\tilde{\mathbf{s}} \in \mathbb{R}^n$ be the target score vector (itself random, with $\mathbb{E}[\tilde{\mathbf{s}}] = \mathbf{s}$). Let $\hat{\mathcal{M}} \subset \mathbb{R}^n$ be a closed convex set. The observations for projection are $\mathbf{s}_p$, and the effective noise relating $\mathbf{s}_p$ to $\tilde{\mathbf{s}}$ is $\boldsymbol{\xi} = \mathbf{s}_p - \tilde{\mathbf{s}} \sim \mathcal{N}(\boldsymbol{\nu}, \tilde{\sigma}^2 I_n)$. The least squares estimator is $\hat{\mathbf{s}} = \Pi_{\hat{\mathcal{M}}}(\mathbf{s}_p)$.
For any $\boldsymbol{u} \in \hat{\mathcal{M}}$, the squared estimation error satisfies, almost surely:
\begin{align}
    \|\hat{\mathbf{s}} - \tilde{\mathbf{s}}\|_2^2 - \|\boldsymbol{u} - \tilde{\mathbf{s}}\|_2^2 \le \frac{1}{n}\left( \sup_{\boldsymbol{\theta} \in \mathcal{T}_{\hat{\mathcal{M}}, \boldsymbol{u}} : \|\boldsymbol{\theta}\|_2 \leq 1} \boldsymbol{\xi}^T \boldsymbol{\theta} \right)^2,
\end{align}
where $\mathcal{T}_{\hat{\mathcal{M}}, \boldsymbol{u}}$ is the tangent cone to $\hat{\mathcal{M}}$ at $\boldsymbol{u}$.
\end{lemma}

Taking expectations and choosing $\boldsymbol{u} = \Pi_{\hat{\mathcal{M}}}(\tilde{\mathbf{s}})$ leads to a decomposition of the total expected risk into three principal components:
\begin{align}
\mathbb{E}[\|\hat{\mathbf{s}} - \tilde{\mathbf{s}}\|_2^2] \le \underbrace{\mathbb{E}_{\hat{\mathcal{M}}, \tilde{\mathbf{s}}} \left[ \|\Pi_{\hat{\mathcal{M}}}(\tilde{\mathbf{s}}) - \tilde{\mathbf{s}}\|_2^2 \right]}_{\text{Projection Error:=$\mathbb{E}_1$}} + \underbrace{\mathbb{E}_{\hat{\mathcal{M}}, \tilde{\mathbf{s}}} \left[ \frac{2\tilde{\sigma}^2}{n} \delta(\mathcal{T}_{\hat{\mathcal{M}}, \Pi_{\hat{\mathcal{M}}}(\tilde{\mathbf{s}})}) \right]}_{\text{Statistical Error:=$\mathbb{E}_2$}} + \underbrace{\mathbb{E}_{\hat{\mathcal{M}}, \tilde{\mathbf{s}}} \left[ \frac{2}{n} \|\Pi_{\mathcal{T}_{\hat{\mathcal{M}}, \Pi_{\hat{\mathcal{M}}}(\tilde{\mathbf{s}})}}(\boldsymbol{\nu})\|_2^2 \right]}_{\text{Bias Error:=$\mathbb{E}_3$}}.
\end{align}
Term $\mathbb{E}_1$ reflects errors from projecting $\tilde{\mathbf{s}}$ onto the potentially incorrect cone $\hat{\mathcal{M}}$. Term $\mathbb{E}_2$ arises from the zero-mean component of $\boldsymbol{\xi}$, where $\delta(\cdot)$ is the statistical dimension of the tangent cone. Term $\mathbb{E}_3$ captures error due to the systematic bias $\boldsymbol{\nu}$ (details are in Appendix~\ref{Appendix:th2_Preliminaries}). Further bounding these components yields our main robustness theorem.

\begin{theorem}
\label{thm:Robustness}
\textbf{(Robustness Guarantee)}.
The total expected squared error of $\hat{\mathbf{s}}$ for $\tilde{\mathbf{s}}$ is bounded by:
\begin{align}
\mathbb{E}[\|\hat{\mathbf{s}} - \tilde{\mathbf{s}}\|_2^2] \;\le \; \mathbb{E}_{\tilde{\mathbf{s}}} \left[n \text{Var}(\tilde{\mathbf{s}})\right] \mathbb{E}\left[\mathrm{Inv}(\hat{\pi}, \tilde{\pi})\right] \;+\; \frac{2\tilde{\sigma}^2 (\ln(n) + \gamma_E + O(1/n))}{n}\;+\; \frac{2\|\boldsymbol{\nu}\|_2^2}{n}.
\end{align}
Here, $\mathbb{E}[\mathrm{Inv}(\hat{\pi}, \tilde{\pi})]$ is the expected ranking inversions from Stage-1. $\tilde{\sigma}^2$ is the variance of $\tilde{\mathbf{s}}$, $\boldsymbol{\nu}$ is the fixed systematic bias, and $\gamma_E$ is the Euler-Mascheroni constant.
\end{theorem}

\paragraph{Proof Sketch.}
The key idea is to view Stage-2 calibration as least-squares projection from $\mathbf s_p$ onto the random isotonic cone $\hat{\mathcal M}$ induced by the Stage-1 ranking $\hat \pi$. Since the calibration target is $\tilde{\mathbf s}$ while the model output is $\mathbf s_p = \tilde{\mathbf s} + \boldsymbol{\nu} - \tilde{\boldsymbol{\epsilon}}$, the effective noise is $\boldsymbol{\xi} = \mathbf s_p - \tilde{\mathbf s} \sim \mathcal N(\boldsymbol{\nu}, \tilde{\sigma}^2 I_n)$, which is both biased and centered around a random cone constraint. We start from an oracle inequality for least-squares projection onto a closed convex set in Lemma~\ref{lem:adapted_bellec_prop2.1_main}. Taking expectations then decomposes the risk into three parts: (1) the projection error caused by using a misspecified cone $\hat{\mathcal M}$, (2) a statistical term controlled by the statistical dimension of the tangent cone, and (3) a bias term induced by the systematic offset $\boldsymbol{\nu}$. The first term is governed by the probability that Stage-1 produces an incorrect ordering, which can be related to the expected number of inversions. The second term is bounded by the statistical dimension of an isotonic cone, which scales as the harmonic number $H_n = \log n + O(1)$. The third term is bounded by non-expansiveness of Euclidean projection. Combining these bounds yields the theorem. The full proof is in Appendix~\ref{app:proof_thmRob}.

\paragraph{Remark.}
The implications in Theorem  are two-fold. Firstly, Theorem~3.5 shows that the robustness of \AtC\ is controlled by three qualitatively different error sources. The first is \emph{structural}: if Stage-1 returns a ranking close to the target ordering, then the projection error remains small, whereas severe ranking mistakes directly enlarge the first term through the expected inversion count. The second is \emph{statistical}: even with the correct cone, calibration still incurs uncertainty from the zero-mean component of the effective noise, but this complexity is mild because isotonic projection only pays the statistical dimension of the tangent cone. The third is \emph{systematic}: a nonzero bias in the model scores cannot be removed for free, yet its contribution is still controlled and decays as $1/n$ after projection.

Next, this theorem also clarifies why our setting is more delicate than standard isotonic regression. Classical analyses typically assume a fixed cone and zero-mean noise relative to the target \citep{10.1214/aos/1021379864, JMLR:v16:bellec15a, chatterjee2015risk}, whereas in \AtC\ both assumptions fail: the cone is random because it is inferred from Stage-1, and the effective noise is biased because the model scores may contain systematic distortion. The theorem shows that isotonic calibration remains well-behaved even under these two departures, which is the regime needed for human-centered assessment.

The first term in the bound of Theorem \ref{thm:Robustness}, which accounts for errors due to Stage-1 ranking inaccuracies, is directly related to the probability of the estimated ranking $\hat{\pi}$ differing from the subjective optimal ranking $\tilde{\pi}$. This probability can be further bounded as follows:

\begin{corollary}[Bound on Ranking Misspecification Probability]
\label{cor:ranking_misspec_prob}
For any two items $j,k$ such that $\tilde{\mathbf s}_j < \tilde{\mathbf s}_k$, let $\Delta_{kj} = \tilde{\mathbf s}_k - \tilde{\mathbf s}_j$ be the true score gap. Let $\Sigma_{\hat{\mathbf{s}}}$ denote the covariance matrix of the estimation error ${\hat{\mathbf{s}}} - \tilde{\mathbf{s}}$ in Lemma ~\ref{lem:hetero-cov}, and $\sigma_{X_{jk}}^2 = (\boldsymbol{e}_j - \boldsymbol{e}_k)^\top \Sigma_{\hat{\mathbf{s}}} (\boldsymbol{e}_j - \boldsymbol{e}_k)$. Let $A^c$ be the event $A^c := \{\hat{\pi} \neq \tilde{\pi}\}$. The probability that the estimated ranking $\pi$ differs from the subjective optimal ranking $\tilde{\pi}$ (i.e., $\hat{\mathcal{M}} \neq \tilde{\mathcal{M}}$)  is bounded by:
\begin{align}
\delta_1:=P(A^c) \;\le\; \sum_{j,k:\,\tilde{\mathbf s}_j < \tilde{\mathbf s}_k} \frac{\sigma_{X_{jk}}}{\Delta_{kj}\sqrt{2\pi}}\; \exp\!\left(-\,\frac{(\Delta_{kj})^2}{2\,\sigma_{X_{jk}}^2}\right)\,.
\label{eq:ranking_error_prob_bound_corrected}
\end{align}
\end{corollary}

Corollary~\ref{cor:ranking_misspec_prob} makes the first error term in Theorem~\ref{thm:Robustness} more explicit by translating the risk of cone misspecification into pairwise ranking inversion probabilities. In particular, it shows that the probability of using an incorrect isotonic cone can be controlled through the uncertainty of the Stage-1 estimator and the pairwise score gaps of the target ordering. 
Besides, it also serves as a technical ingredient for later analysis in Theorem~\ref{thm:optimality}. It provides a tractable way to upper bound the contribution of Stage-1 ranking errors, which will be invoked in the subsequent proof when establishing conditions under which the two-stage AtC pipeline yields an improvement over the raw predictor.

% To conclude, the bound highlights in Theorem \ref{thm:Robustness} three key error sources: (i) Stage-1 ranking inaccuracies, quantified by the expected number of inversions $\mathbb{E}[\mathrm{Inv}(\hat{\pi}, \tilde{\pi})]$ and scaled by the target's variability, which rapidly diminishes with larger score gaps relative to first-stage noise; (ii) statistical uncertainty from the zero-mean effective noise, decaying around $(\ln n)/n$; (iii) systematic observation bias $\boldsymbol{\nu}$, decaying as $1/n$. Notably, even perfect first-stage ranking does not eliminate errors from $\tilde{\sigma}^2$ and $\boldsymbol{\nu}$, while poor initial rankings can significantly inflate the total error, though mitigation of these factors can still lead to effective error reduction. The proof is in Appendix~\ref{app:proof_thmRob}.

\subsection{Optimality Guarantee}

Finally, we establish that the \AtC\ yields a calibrated output that exceeds the uncalibrated model’s output in accuracy with high probability. Intuitively, if the Stage-1 aggregation produces an almost-correct ordering of items and Stage-2 calibration correctly removes the bias, then the final result will be closer to the ground truth than using the original model predictions alone. Recall from Corollary~\ref{cor:ranking_misspec_prob} shows that $\delta_1$ decays exponentially in the gap‐to‐noise ratio $\frac{\sigma_{X_{jk}}}{\Delta_{kj}}$ for each pair. If every true gap is not too small compared to the associated noise, then $\delta_1$ becomes negligible.
The following proposition shows that Stage-1 produces the correct order of items with high probability. 
\begin{proposition}
\label{lem:separation}
Let $\displaystyle \Delta_{\min}=\min_{j<k}\Delta_{kj}$ and 
$\displaystyle \sigma_{\max}=\max_{j<k}\sigma_{X_{jk}}$.  
If
\(
\displaystyle 
\frac{\sigma_{\max}}{\Delta_{\min}}
\le \varepsilon
\)
for some sufficiently small $\varepsilon>0$
(equivalently, the Stage--1 sample size satisfies
$mk\gg 1/(\Delta_{\min}\varepsilon)^{2}$),
then $\delta_1 = o(1)$; i.e.\ the estimated ranking space $\hat{\mathcal{M}}$ will equal the true ranking space ${\mathcal{M}}$. 
\end{proposition}
At the same time, Stage-2 may still fail if the projection is badly perturbed by the
noise $\boldsymbol{\tilde \epsilon}\sim \mathcal{N}(0,\tilde \sigma^{2}I_n)$.
Let's define the event
\(B=\{{\mathcal{M}}=\tilde{\mathcal{M}}\}\), similarly we have:
\begin{lemma}
\label{lem:delta2}
Let $\mathbf s \in \mathbb R^n$ be the ground truth score and $\tilde{\mathbf s} = \mathbf{s} + \tilde{\boldsymbol \epsilon}$, where $\boldsymbol{\tilde \epsilon}\sim \mathcal{N}(0,\tilde \sigma^{2}I_n)$. Let $B^c$ be the event $B^c := \{\pi(\tilde{\mathbf{s}}) \neq \pi(\mathbf{s})\}$. Then we have,
\begin{equation}\label{eq:delta2}
\delta_2
\;:=\;
P(B^{\mathsf{c}})
\;\le\;
\sum_{j,k:\,s_k>s_j}
\frac{\tilde \sigma}{\Delta_{kj}\sqrt{\pi}}
\exp\left(-\tfrac{\Delta_{kj}^{2}}{4\tilde \sigma^{2}}\right).
\end{equation}
\end{lemma}

Lemma~\ref{lem:delta2} shows that the chance of calibration error
decays once the post–aggregation score gaps exceed the noise
level~$\sigma$.
Because the pairwise comparisons used in Stage-1 are
independent of the Gaussian perturbations $\boldsymbol{\epsilon}$ in
Stage-2, the events $A, B$ in Corollary~\ref{cor:ranking_misspec_prob} and Lemma~\ref{lem:delta2}
are independent.

\begin{theorem}
\label{thm:optimality}
\textbf{(Optimality Guarantee)}.
Let $\delta_1$ and $\delta_2$ be as in Corollary~\ref{cor:ranking_misspec_prob} and Lemma ~\ref{lem:delta2}.  
Then, with probability at least\/ $1-\delta_1-\delta_2$,
\begin{align}
\bigl\|\hat{\mathbf{s}}-\mathbf{s}\bigr\|_2^{2}
\;<\;
\bigl\|\mathbf{s}_p-\mathbf{s}\bigr\|_2^{2}.
\end{align}
i.e.\ the calibrated output $\hat{\mathbf{s}}$
is strictly closer to the ground truth $\mathbf{s}$
than the uncalibrated model prediction $\mathbf{s}_p$.
As the Stage--1 sample size grows
and\/ $\Delta_{\min}$ dominates the noise,
both $\delta_1$ and $\delta_2$ approach\/ $0$,
so the probability of improvement approaches\/ $1$.
\end{theorem}
% \end{tcolorbox}

\paragraph{Proof Sketch.}
The proof is based on a simple projection geometry argument combined with a decomposition of ranking errors across the two stages. Previously, we define $A$ is the event that Stage-1 recovers the correct ranking for its own target, while $B$ is the event that the subjective target $\tilde{\mathbf s}$ preserves the true ranking of $\mathbf s$.
On the intersection $A \cap B$, the estimated isotonic cone coincides with the true cone induced by $\mathbf s$, and hence $\mathbf s \in \hat{\mathcal M}.$ Once $\mathbf s$ belongs to the projection set, the improvement claim follows from the Pythagorean identity for Euclidean projection onto a closed convex set $\|\mathbf s_p - \mathbf s\|_2^2 \ge \|\mathbf s_p - \hat{\mathbf s}\|_2^2 + \|\hat{\mathbf s} - \mathbf s\|_2^2$. Therefore, whenever $\mathbf s_p \notin \hat{\mathcal M}$, the first term on the right-hand side is strictly positive, which implies $\|\hat{\mathbf s} - \mathbf s\|_2^2 < \|\mathbf s_p - \mathbf s\|_2^2$. It remains to lower bound the probability of this favorable event. A union bound gives $P(A \cap B) \ge 1-\delta_1-\delta_2$. Therefore, with probability at least $1-\delta_1-\delta_2$, projection onto the estimated cone yields a strictly better estimate than the raw predictor. The full proof is in Appendix~\ref{sec:improvement_probability}.

Theorem~\ref{thm:optimality} confirms that, once Stage-1 has gathered enough pairwise comparisons to rank items reliably, the monotone projection in Stage-2 invariably drives the (possibly biased) model scores toward the ground truth. In effect, \AtC\ protects against misspecification: even if both aggregated consensus and predictive model are imperfect, enforcing the human–derived rank guarantees a strictly better estimator with high probability. 

Overall, in this section we establishes the theoretical foundation of AtC from three complementary perspectives: efficiency, robustness, and optimality. Theorem~\ref{thm:cov-compare} shows that heterogeneous judgment aggregation is statistically preferable to homogeneous modeling when annotator reliability truly varies. Theorem~\ref{thm:Robustness} shows that isotonic calibration remains robust under an estimated ranking constraint and biased effective noise. Theorem~\ref{thm:optimality} further demonstrates that these two ingredients combine to yield a strict improvement over the raw predictor with high probability. Together, these results formalize the central intuition behind AtC: reliable ordinal information from human judgments can be aggregated and then converted, through calibration, into more accurate final scores.
\section{Experiments}\label{sec:experiments}

We evaluate \AtC\ on both semi-synthetic and real-world datasets, guided by 6 research questions: 
\begin{itemize}[leftmargin=*]
    \item (\textbf{RQ1}) Does \AtC's output $\hat{\mathbf{s}}$ outperform human-only assessment $\mathbf{s}^{*}$ and model-only assessment $\mathbf{s}_p$?
    \item (\textbf{RQ2}) Do heterogeneous rank aggregation methods outperform homogeneous ones, and is \AtC's design of using the ranking $\hat{\pi}$ from $\mathbf{s}^{*}$ rather than directly using the $\mathbf{s}^{*}$ values correct?
    \item (\textbf{RQ3}) Is \AtC\ robust to aggregation errors?
    \item (\textbf{RQ4}) Is \AtC\ robust to model assessments ($\mathbf s_p$) errors?
    \item (\textbf{RQ5}) Does \AtC's assumption that rankings are more reliable than ratings hold true in practice?
    \item (\textbf{RQ6}) How well does \AtC\ perform on a real-world task?
\end{itemize}

% \vspace*{-.1in}
\begin{table}[!bht]\small
\centering
\caption{Semi-synthetic results on Reading Level dataset. \\
\small \textbf{Bold} indicate the best performance among human-only assessment ($\mathbf{s}^{*}$), model-only assessment ($\mathbf{s}_p$), and \AtC\ assessment ($\hat{\mathbf{s}}$), respectively. The \textcolor{red}{\textbf{red}} denotes the optimal performance across all methods for the given metric.}
\label{tab:semi-synthetic}
\begin{tabular}{p{1.1cm}cccc}
\toprule
\multirow{2}{*}{\textbf{Stage-1}} & \textbf{Kendall $\tau$}$\uparrow$ & \textbf{Wasserstein}$\downarrow$ & \textbf{KS}$\downarrow$ & \textbf{MSE}$\downarrow$ \\
\cmidrule(lr){2-2}\cmidrule(lr){3-3}\cmidrule(lr){4-4}\cmidrule(lr){5-5}
\textbf{Method}& $\mathbf{s}^{*}/\ \mathbf{s}_p/\ \hat{\mathbf{s}}$ & $\mathbf{s}^{*}/\ \mathbf{s}_p/\ \hat{\mathbf{s}}$ & $\mathbf{s}^{*}/\ \mathbf{s}_p/\ \hat{\mathbf{s}}$ & $\mathbf{s}^{*}/\ \mathbf{s}_p/\ \hat{\mathbf{s}}$ \\
\midrule
HRA-G & 0.375 / 0.399 / \textcolor{red}{\textbf{0.410}} & 2.250 / 2.831 / \textbf{0.839} & 0.500 / 0.300 / \textcolor{red}{\textbf{0.163}} & 8.658 / 29.00 / \textbf{8.122} \\
HRA-E & 0.375 / 0.399 / \textbf{0.403} & 2.243 / 2.831 / \textbf{0.827} & 0.498 / 0.300 / \textcolor{red}{\textbf{0.163}} & 8.658 / 29.00 / \textbf{8.191} \\
HRA-N & 0.368 / 0.399 / \textbf{0.399} & 2.351 / 2.831 / \textcolor{red}{\textbf{0.738}} & 0.563 / 0.300 / \textbf{0.192} & 8.985 / 29.00 / \textbf{7.919} \\
\midrule
CrowdBT & 0.354 / \textbf{0.399} / 0.399 & 2.150 / 2.831 / \textbf{0.843} & 0.455 / 0.300 / \textbf{0.269} & 8.301 / 29.00 / \textcolor{red}{\textbf{7.555}} \\
CrowdTCV & 0.339 / \textbf{0.399} / 0.372 & 2.272 / 2.831 / \textbf{1.015} & 0.506 / \textbf{0.300} / 0.312 & 8.689 / 29.00 / \textbf{7.809} \\
\midrule
BTL & 0.340 / \textbf{0.399} / 0.373 & 2.186 / 2.831 / \textbf{0.894} & 0.461 / 0.300 / \textbf{0.300} & 7.764 / 29.00 / \textbf{8.097} \\
TCV & 0.338 / \textbf{0.399} / 0.373 & 2.343 / 2.831 / \textbf{0.914} & 0.547 / 0.300 / \textbf{0.300} & 8.943 / 29.00 / \textbf{8.030} \\
\bottomrule
\end{tabular}
% \vspace*{-.25in}
\end{table}

\textbf{Datasets.} We consider two datasets in our experiments. \textbf{(1) Reading-level (semi-synthetic dataset)}~\citep{10.1145/2433396.2433420} contains pairwise comparisons of text documents based on reading difficulty. 490 documents with known reading levels, 624 annotators providing 12,728 pairwise judgments.
\textbf{(2) Dots-activity (real-world dataset)}~\citep{kemmer2020enhancing}, a benchmark for collective judgment aggregation. The task involves estimating the number of dots in images. The dataset contains two types judgments (ranking and rating) from 300 participants on 30 distinct images, which yield 8700 pairwise comparisons.

\offer{\textbf{Baselines and Metrics.} For the semi-synthetic dataset, we compare 7 different aggregation methods: 3 HRA variants based on different noise distribution assumptions (HRA-G/E/N), 2 heterogeneous methods (CrowdBT/TCV)~\citep{10.1145/2433396.2433420}, and 2 homogeneous baselines (BTL/TCV)~\citep{BTL}. For the real-world dataset, we additionally compare against 3 human-centered assessment methods: GPPL~\citep{10.1145/1102351.1102369}, Rank-SVM~\citep{ranksvm}, and BARCW~\citep{barcw}. We evaluate using 4 metrics: Kendall $\tau$ measures ranking accuracy; Wasserstein distance and KS statistic quantify distributional alignment between predicted and ground-truth scores; MSE captures absolute prediction error.} Detailed descriptions are provided in Appendix~\ref{sec:appendix-exp}.

\begin{wrapfigure}{r}{0.3\textwidth}
    \centering
    % \vspace*{-.2in}
    \includegraphics[width=\linewidth]{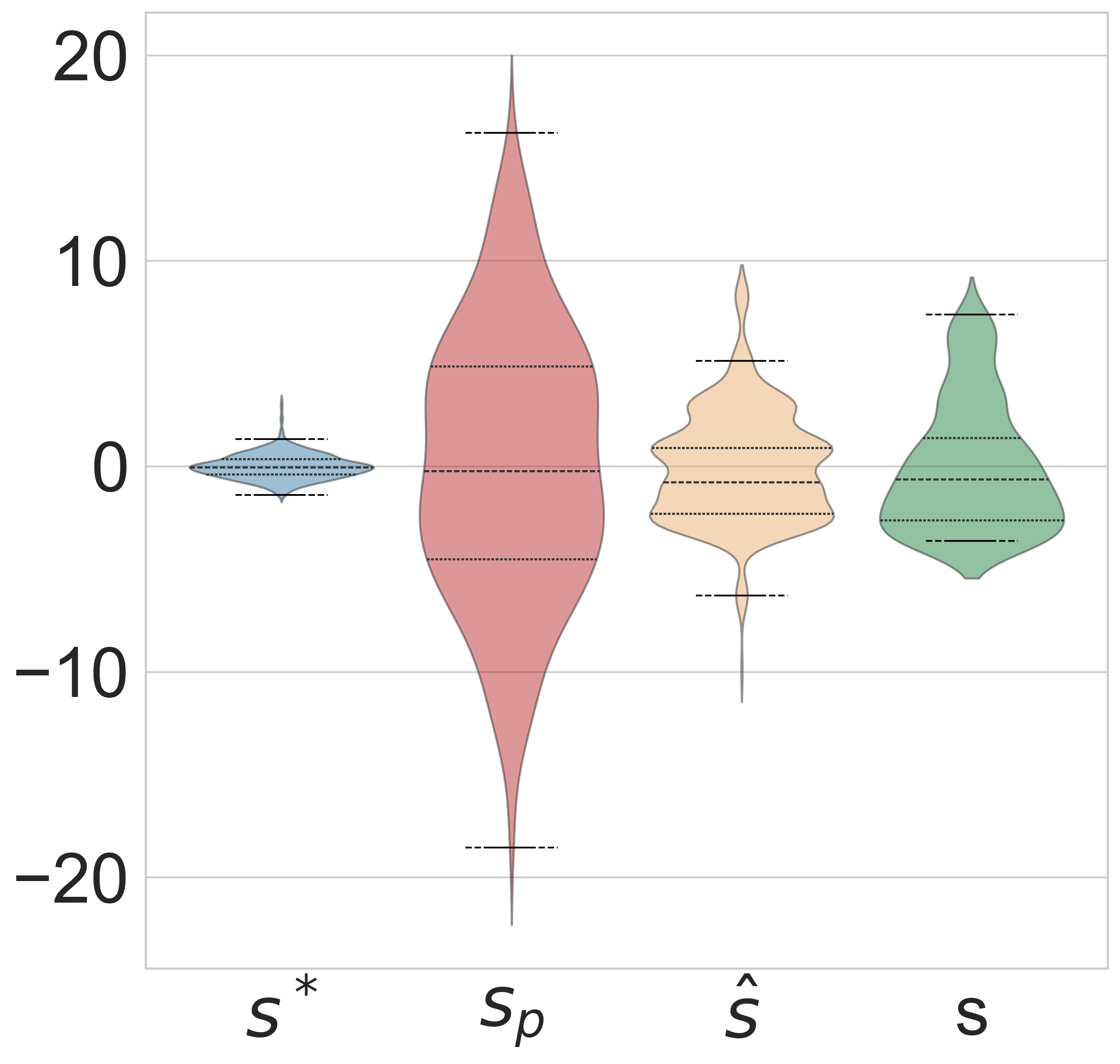}
    \renewcommand{\arraystretch}{1}
    \setlength{\abovecaptionskip}{-0.25cm}
    \caption{Score Distribution of HRA-E}
    \setlength{\belowcaptionskip}{-0.1cm}
    \label{fig:violin}
    % \vspace*{-.2in}
\end{wrapfigure}

\textbf{Semi-Synthetic Evaluation.} We first evaluate on semi-synthetic setting where ground-truth item scores are available. We simulate $m$ annotators who provide pairwise comparisons, apply various Stage-1 aggregation methods (w/ or w/o modeling annotator heterogeneity) to obtain a consensus score $\mathbf{s}^{*}$, and then train a predictive model to produce an initial model score $\mathbf{s}_p$ for each item (using a portion of the ground truth with added noise to emulate model error). We calibrate $\mathbf{s}_p$ via \AtC\ to obtain the output $\hat{\mathbf{s}}$.

In Table~\ref{tab:semi-synthetic}, we observe that $\hat{\mathbf{s}}$ consistently outperforms both $\mathbf{s}^{*}$ and $\mathbf{s}_p$ across nearly all metrics and methods (the $\hat{\mathbf{s}}$ entries are almost always \textbf{bolded} as best).\footnote{\offer{The Kendall's $\tau$ values for $\hat{\mathbf{s}}$ and $\mathbf{s}^{*}$ differ due to tie-creation operation; see Appendix~\ref{app:metrics} for details.}} 
Figure~\ref{fig:violin} further illustrates this improvement: the $\hat{\mathbf{s}}$ values align much more closely with the true distribution of $s$, whereas the uncalibrated $\mathbf{s}^{*}$ and $\mathbf{s}_p$ scores deviate significantly (same conclusion for all methods in Appendix~\ref{sec:score_visualizations}). 
\offer{To verify that this improvement stems from $\mathbf s_p$ rather than simple rescaling, we conducted experiments matching $\mathbf{s}^{*}$'s range to ground truth (Appendix~\ref{app:scaling-analysis}).}
In summary, the \AtC\ score dominates the human-only and model-only assessment, confirming that combining the two sources produces more accurate judgments (answering \textbf{RQ1}). 

We examine the effect of modeling annotator heterogeneity on the consensus ranking quality. In Table~\ref{tab:semi-synthetic}, the methods that account for annotator heterogeneity (HRA-G/E/N, and CrowdBT/TCV) produce better (or the \textcolor{red}{\textbf{best}}) calibrated scores $\hat{\mathbf{s}}$ than the homogeneous models (BTL, TCV). Furthermore, we observe that for $\mathbf{s}^{*}$, the heterogeneous HRA models achieve more accurate rankings (higher Kendall $\tau$), despite performing worse on distance metrics compared to 4 baselines. Notably, after calibration with $\mathbf{s}_p$, the resulting $\hat{\mathbf{s}}$ scores from heterogeneous models surpass 4 baselines across Wasserstein distance and KS statistic, confirming that our design choice to prioritize ranking information from $\mathbf{s}^{*}$ is indeed effective (answering \textbf{RQ2}).

\begin{wrapfigure}{r}{0.3\textwidth}
    \centering
    % \vspace*{-.3in}
    \renewcommand{\arraystretch}{1}
    \includegraphics[width=1\linewidth]{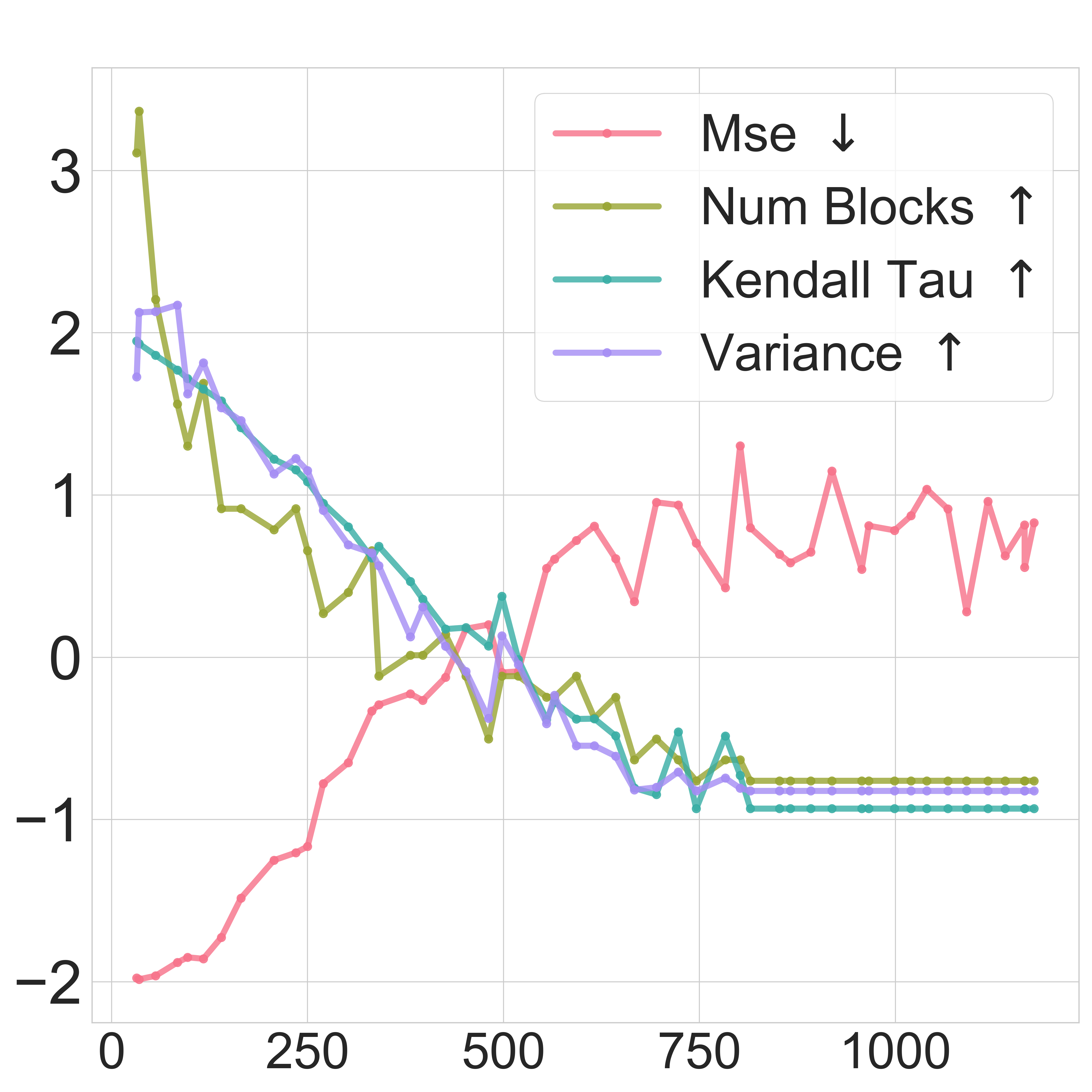}
    \setlength{\abovecaptionskip}{-0.2cm}
    \caption{Robustness Analysis}
    \setlength{\belowcaptionskip}{-0.2cm}
    \label{fig:robu}
    % \vspace*{-.3in}
\end{wrapfigure}
Figure~\ref{fig:robu} shows \AtC's robustness to ranking errors in $\mathbf{s}^{*}$. Our normalized results show that as we introduce pairwise judgment inversions, \AtC\ degrades gradually until a critical threshold. With moderate noise (up to 400 inversions), Kendall $\tau$ decreases steadily while MSE increases, but \AtC\ still produces meaningful outputs by leveraging model signals. However, beyond approximately 500 inversions, performance collapses completely—Kendall $\tau$ becomes negative, variance approaches zero, and the output distribution flattens (shown by declining block count). This synchronized degradation across all metrics confirms that while \AtC\ can tolerate considerable noise in consensus rankings, extremely corrupted inputs will eventually render the calibration ineffective (answering \textbf{RQ3}).

\textbf{Real-World Evaluation.} We evaluate the validity and robustness of the \AtC\ framework on the Dots-activity dataset, a benchmark for counting dots in images. We assume models can only process corrupted images due to various constraints, while humans can still make reliable inferences based on experience, making this a valuable human-centered assessment problem. We introduce various image corruptions to simulate real-world scenarios with noisy or partially occluded data, then generate predictive scores $\mathbf{s}_p$ using OpenCV's contour detection~\citep{SUZUKI198532} as an imperfect predictor. As illustrated in Figure~\ref{fig:damage-examples}, we apply 4 corruption types~\citep{hendrycks2018benchmarking}: global Gaussian blur and three localized corruptions (blur, noise, and whiteout), with degradation intensity controlled by a hyperparameter to evaluate performance across different corruption levels.

\begin{figure}[htbp]
    % \vspace*{-.25in}
    \centering
    \subfigure[Original]{
        \includegraphics[width=0.17\textwidth]{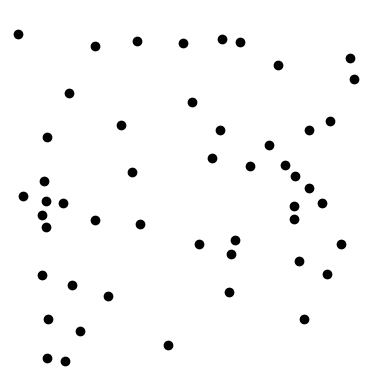}
    }
    \subfigure[BLUR]{
        \includegraphics[width=0.17\textwidth]{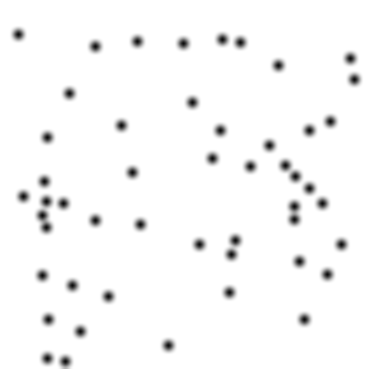}
    }
    \subfigure[LOCAL\_BLUR]{
        \includegraphics[width=0.17\textwidth]{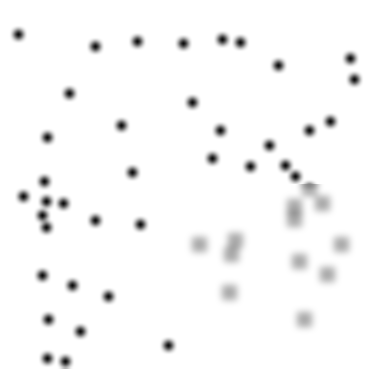}
    }
    \subfigure[LOCAL\_NOISE]{
        \includegraphics[width=0.17\textwidth]{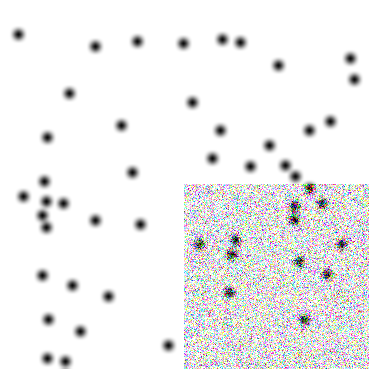}
    }
    \subfigure[WHITEOUT]{
        \includegraphics[width=0.17\textwidth]{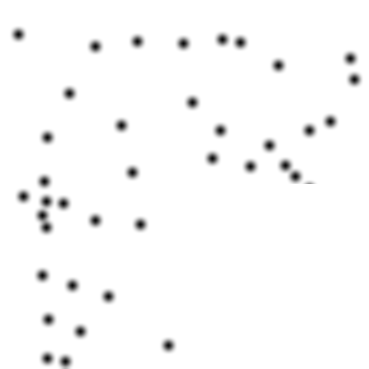}
    }
    % \vspace*{-.15in}
    \caption{Examples of image corruption types applied to the Dots-activity dataset.}
    % \vspace*{-.2in}
    \label{fig:damage-examples}
\end{figure}

\offer{Table~\ref{tab:real-world-dots} summarizes performance across \AtC\ and 3 human-centered assessment baselines (GPPL, Rank-SVM, BARCW); comprehensive results including all Stage-1 methods appear in}
Figure~\ref{fig:hra-e-radars},~\ref{fig:radar_HRA-G}, and~\ref{fig:radar_HRA-N} presents the core robustness results for HRA-E, HRA-G and HRA-N, with similar trends for other baseline methods shown in Appendix~\ref{app:appendix-radars}. The radar charts intuitively visualize the degradation of Kendall's Tau correlation as damage intensity increases (clockwise from the top axis). Each chart compares the initial model ($\mathbf s_p$), the uncalibrated human consensus ($\mathbf s^*$), and the calibrated \AtC\ outputs based on rankings ($\hat{\mathbf s}$ (Rank)) and ratings ($\hat{\mathbf s}$ (Rate)).
Across the 4 types of corruption, our proposed primary method $\hat{\mathbf s}$ (Rank) demonstrates remarkable resilience. While the performance of the raw objective model $\mathbf s_p$ degrades sharply with increasing noise, particularly under localized corruptions. This shows that by anchoring the model's scores to the stable structure provided by human consensus, \AtC\ effectively mitigates the impact of noise on the predictor (answering \textbf{RQ4}). Furthermore, these results underscore the superiority of using ordinal (ranking) information for calibration. In every scenario, $\hat{\mathbf s}$ (Rank) outperforms $\hat{\mathbf s}$ (Rate), the alternative calibrated using aggregated cardinal ratings. This suggests that the consensus ranking ($\mathbf s^*$) provides a more robust calibration framework than aggregated numerical estimates, which are susceptible to individual biases and inconsistent scales (answering \textbf{RQ5}). In summary, the \AtC\ framework demonstrates strong real-world performance by maintaining high accuracy even under severe input degradation (answering \textbf{RQ6}).

\begin{table}[tbp]\small
% \vspace*{-.1in}
\centering
\caption{Real-world results on Dots dataset.}
\label{tab:real-world-dots}
\begin{tabular}{p{0.3cm}p{1.15cm}p{2.55cm}p{2.1cm}p{2.7cm}p{2.55cm}}
\toprule
& \multirow{2}{*}{\textbf{Stage-1}} & \textbf{Kendall $\tau$}$\uparrow$ & \textbf{Wasserstein}$\downarrow$ & \textbf{KL}$\downarrow$ & \textbf{MSE}$\downarrow$ \\
\cmidrule(lr){3-3}\cmidrule(lr){4-4}\cmidrule(lr){5-5}\cmidrule(lr){6-6}
& \textbf{Method} & $\mathbf{s}^{*}/\ \mathbf{s}_p/\ \hat{\mathbf{s}}$ & $\mathbf{s}^{*}/\ \mathbf{s}_p/\ \hat{\mathbf{s}}$ & $\mathbf{s}^{*}/\ \mathbf{s}_p/\ \hat{\mathbf{s}}$ & $\mathbf{s}^{*}/\ \mathbf{s}_p/\ \hat{\mathbf{s}}$ \\
\midrule
\multirow{3}{*}{\textbf{AtC}}
& HRA-G & 0.917 / 0.923 / \textbf{0.940} & 6.97 / 2.53 / \textcolor{red}{\textbf{2.53}} & 123.13 / 0.881 / \textbf{0.860} & 65.08 / 11.75 / \textbf{9.61} \\
& HRA-E & 0.922 / 0.922 / \textcolor{red}{\textbf{0.943}} & 6.98 / 2.53 / \textcolor{red}{\textbf{2.53}} & 123.16 / 0.881 / \textbf{0.861} & 65.16 / 11.75 / \textcolor{red}{\textbf{9.59}} \\
& HRA-N & 0.917 / 0.923 / \textbf{0.934} & 7.05 / 2.53 / \textcolor{red}{\textbf{2.53}} & 125.50 / 0.881 / \textcolor{red}{\textbf{0.859}} & 66.44 / 11.75 / \textbf{9.75} \\
\midrule
\multicolumn{2}{l}{GPPL} & -- / -- / 0.931 & -- / -- / 64.50 & -- / -- / 126.62 & -- / -- / 4220.36 \\
\multicolumn{2}{l}{Rank-SVM} & -- / -- / 0.923 & -- / -- / 61.20 & -- / -- / 133.70 & -- / -- / 3814.49 \\
\multicolumn{2}{l}{BARCW} & -- / -- / 0.940 & -- / -- / 64.62 & -- / -- / 24.09 & -- / -- / 4236.16 \\
\bottomrule
\end{tabular}
% \vspace*{-.25in}
\end{table}

% \begin{figure}[htbp]
% % \vspace*{-.1in}
%     \centering
%     \subfigure[BLUR]{
%         \includegraphics[width=0.205\textwidth]{figure/radar_plots/radar_plot_HRA-E_GAUSSIAN_BLUR.pdf}
%     }
%     \subfigure[LOCAL\_BLUR]{
%         \includegraphics[width=0.205\textwidth]{figure/radar_plots/radar_plot_HRA-E_LOCAL_BLUR.pdf}
%     }
%     \subfigure[LOCAL\_NOISE]{
%         \includegraphics[width=0.205\textwidth]{figure/radar_plots/radar_plot_HRA-E_LOCAL_NOISE.pdf}
%     }
%     \subfigure[WHITEOUT]{
%         \includegraphics[width=0.31\textwidth]{figure/radar_plots/radar_plot_HRA-E_WHITEOUT.pdf}
%     }
%     % \vspace*{-.15in}
%     \caption{HRA-E's performance (Kendall $\tau$$\uparrow$) under 4 corruption types as damage intensity increases.}
%     % \vspace*{-.1in}
%     \label{fig:hra-e-radars}
% \end{figure}

\begin{figure}[htbp]
% \vspace*{-.1in}
    \centering
    \subfigure[BLUR]{
        \includegraphics[width=0.23\textwidth]{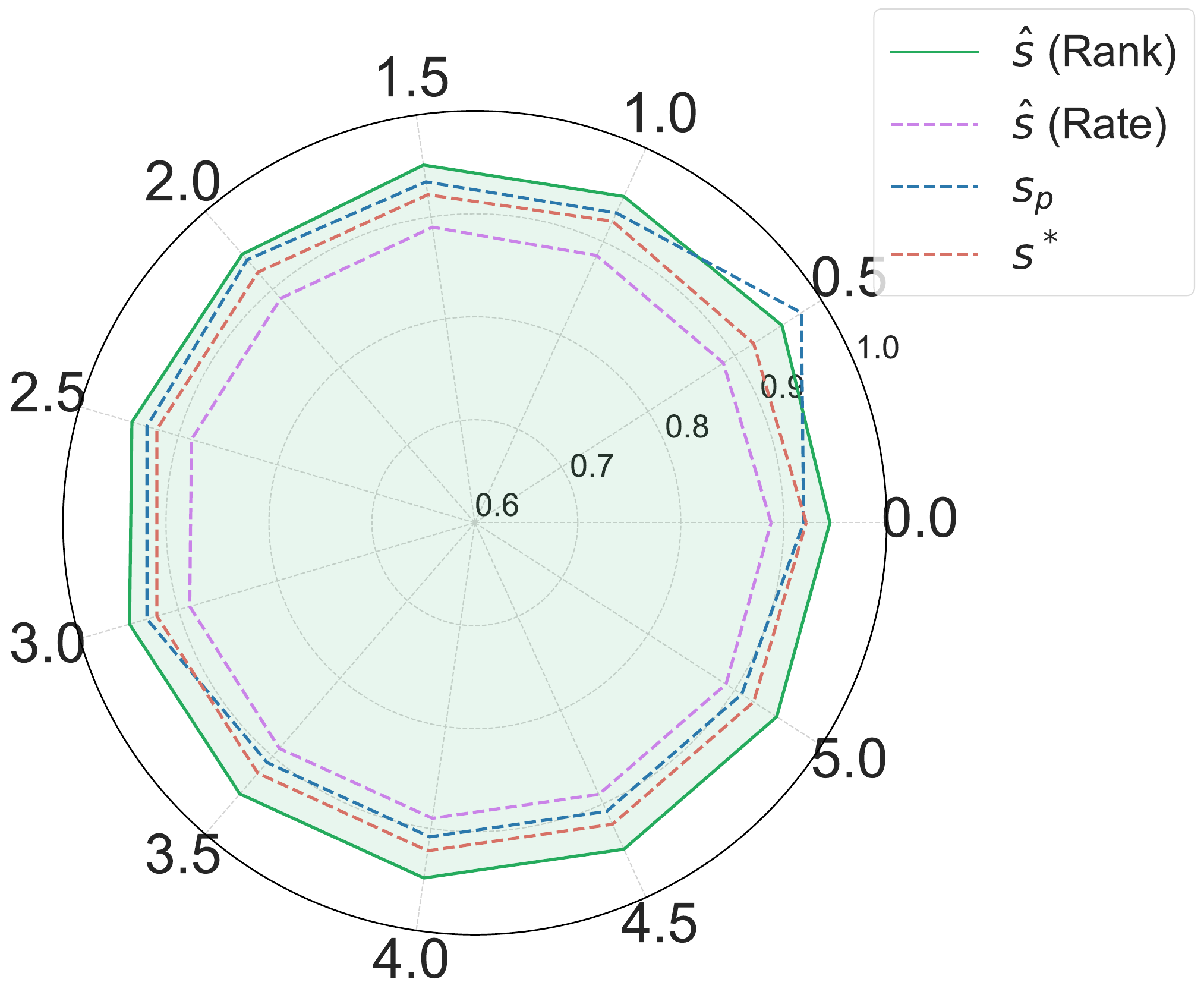}
    }
    \subfigure[LOCAL\_BLUR]{
        \includegraphics[width=0.23\textwidth]{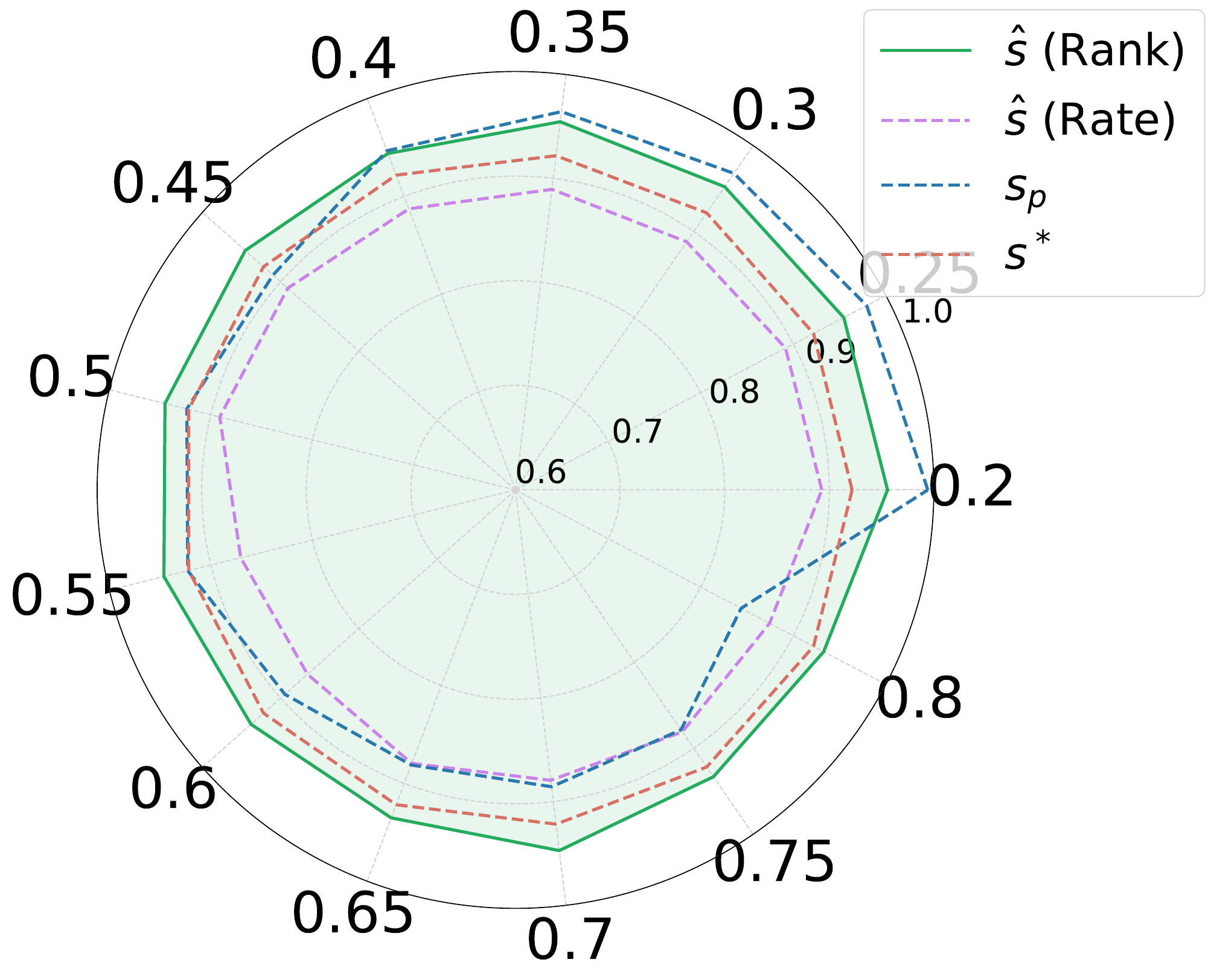}
    }
    \subfigure[LOCAL\_NOISE]{
        \includegraphics[width=0.23\textwidth]{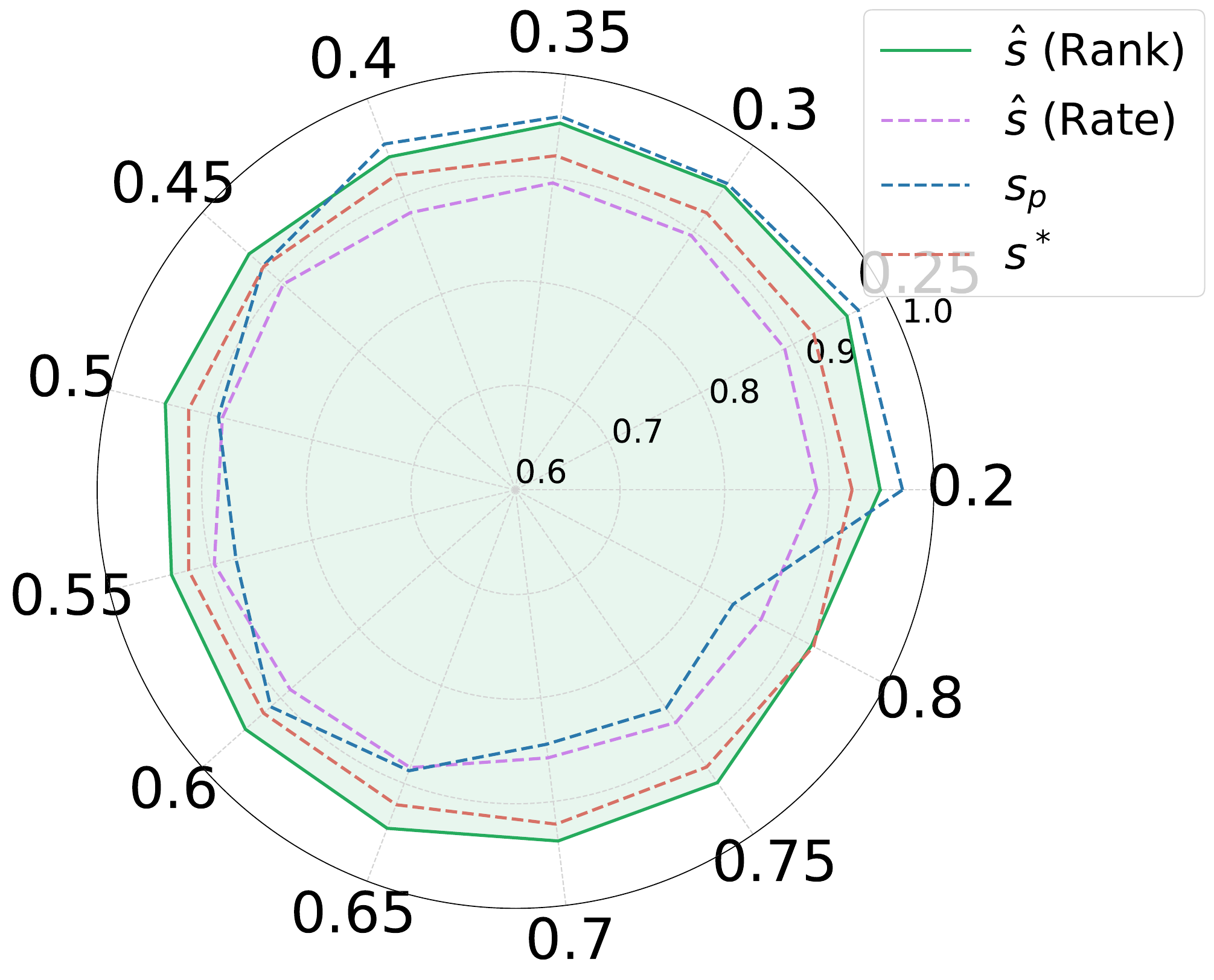}
    }
    \subfigure[WHITEOUT]{
        \includegraphics[width=0.23\textwidth]{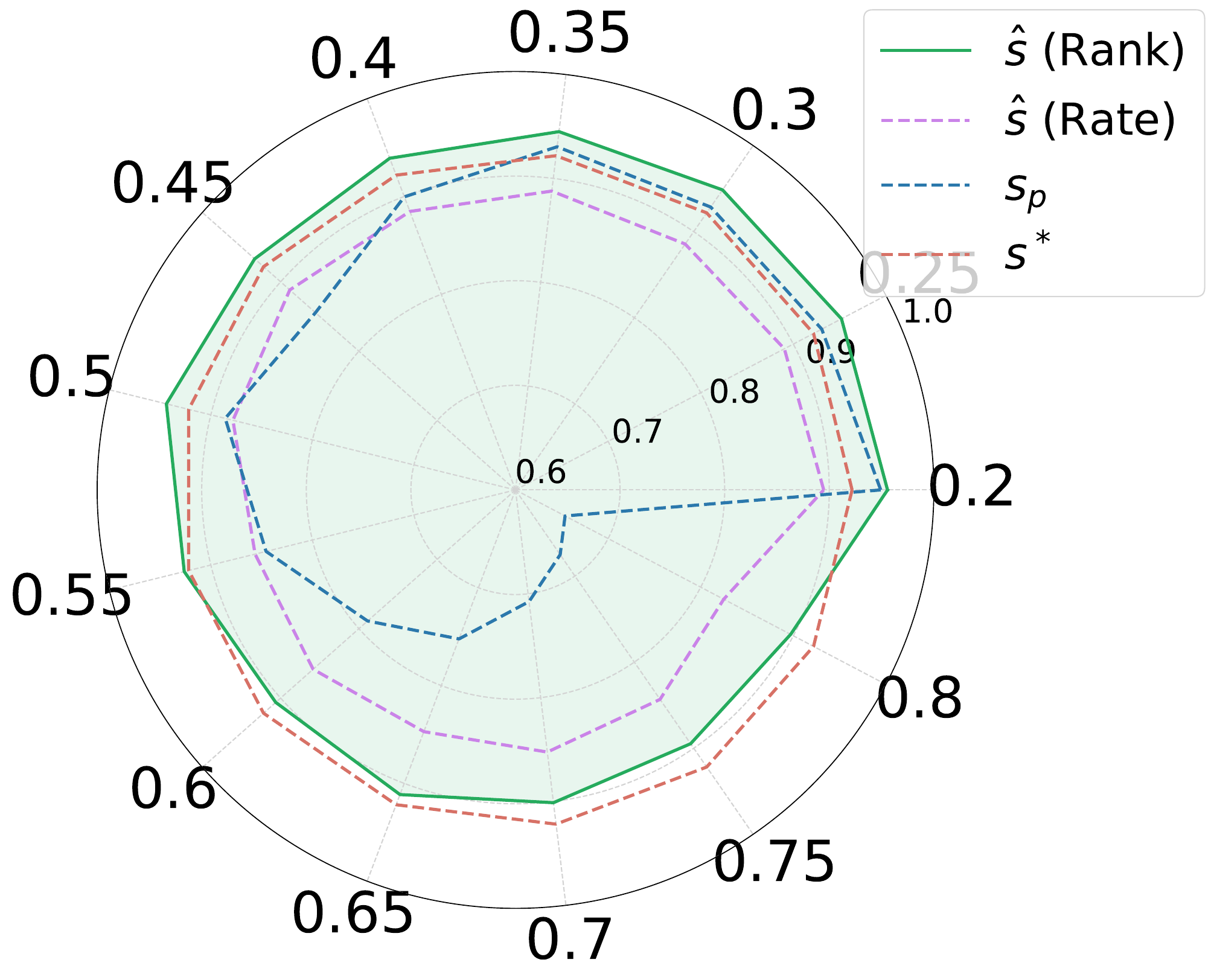}
    }
    % \vspace*{-.15in}
    \caption{Radar plots under different noise conditions (HRA-E).}
    % \vspace*{-.1in}
    \label{fig:hra-e-radars}
\end{figure}

\begin{figure}[htbp]
    \centering
    \subfigure[BLUR]{
        \includegraphics[width=0.23\textwidth]{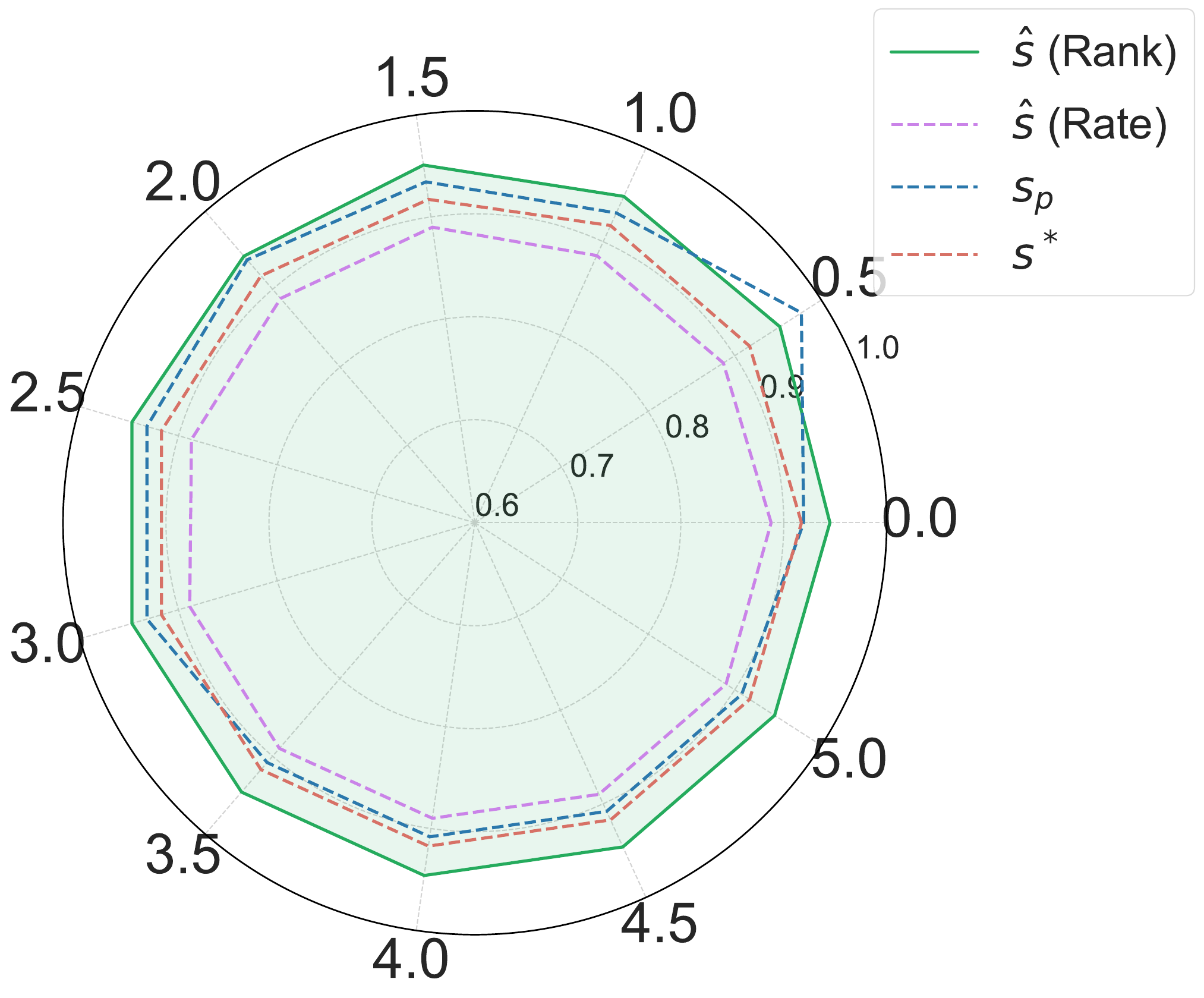}
    }
    \subfigure[LOCAL\_BLUR]{
        \includegraphics[width=0.23\textwidth]{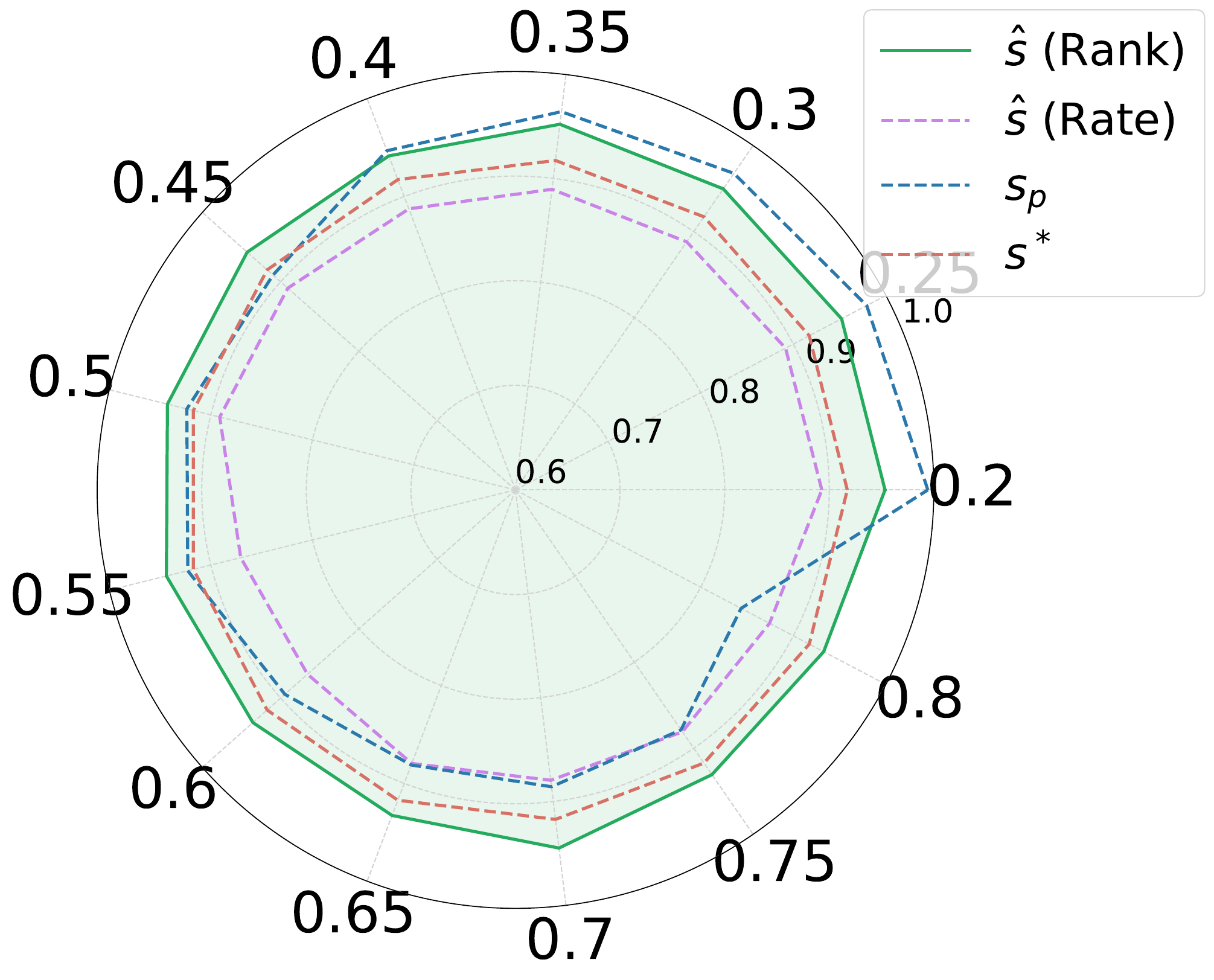}
    }
    \subfigure[LOCAL\_NOISE]{
        \includegraphics[width=0.23\textwidth]{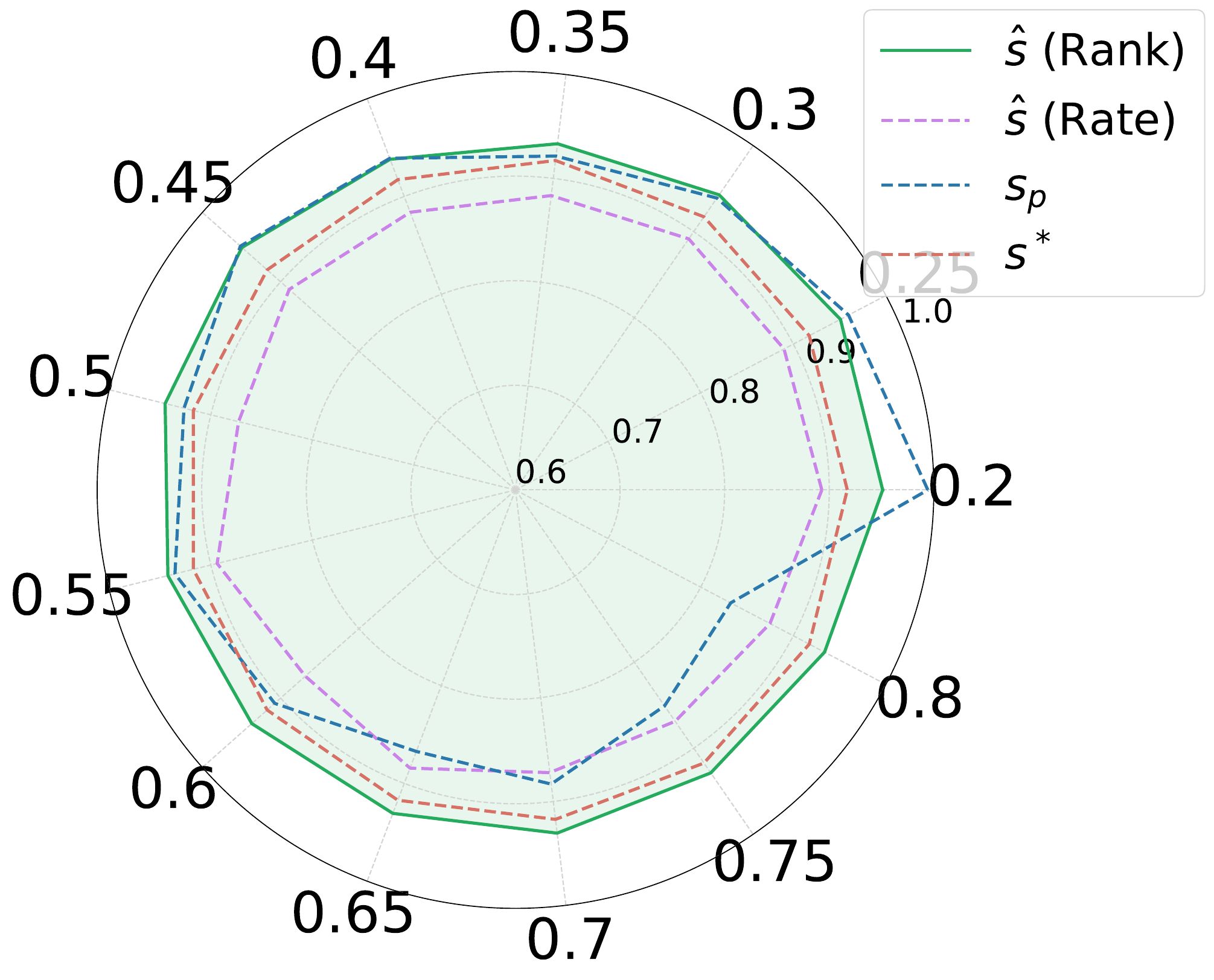}
    }
    \subfigure[WHITEOUT]{
        \includegraphics[width=0.23\textwidth]{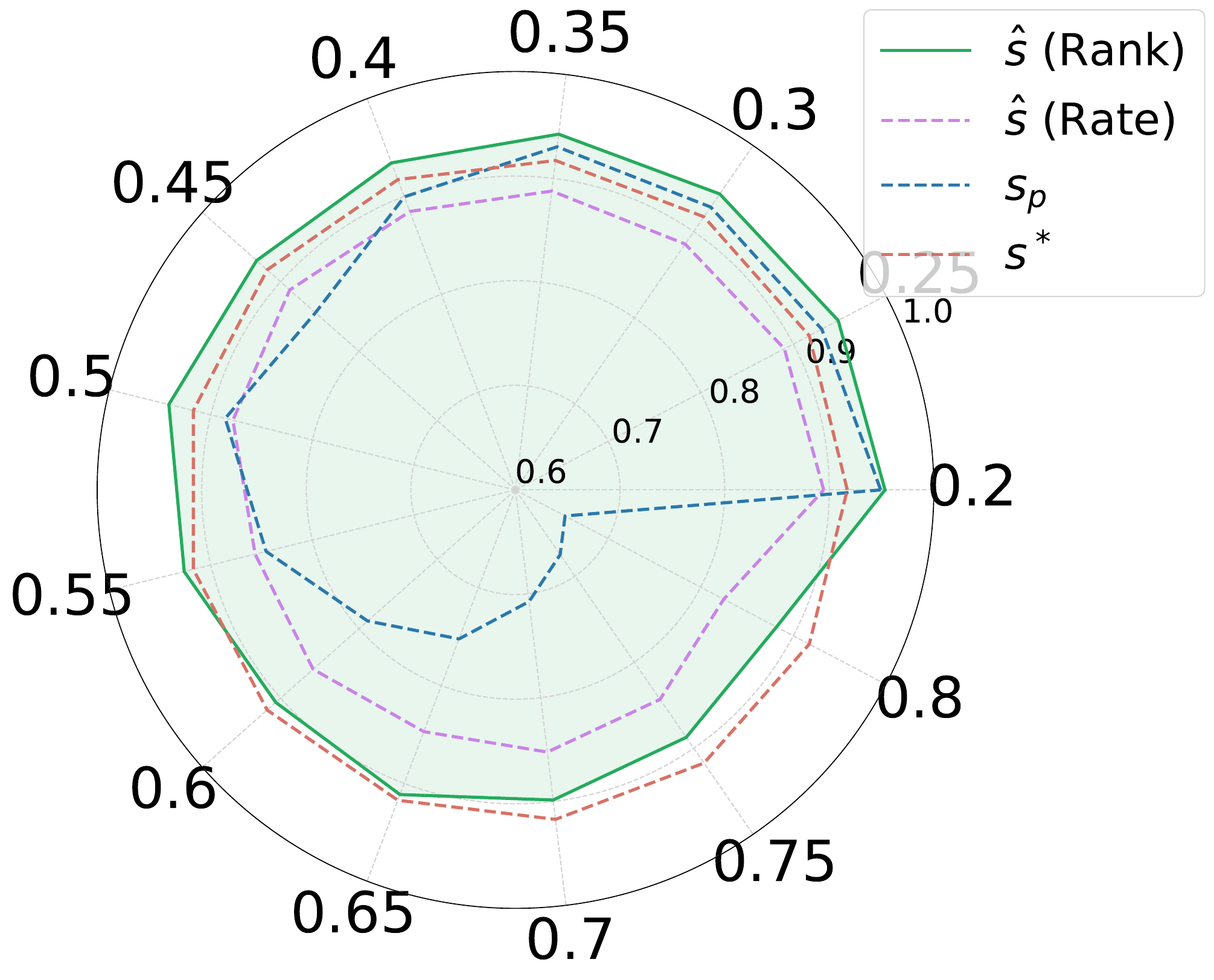}
    }
    \caption{Radar plots under different noise conditions (HRA-G).}
    \label{fig:radar_HRA-G}
\end{figure}

\begin{figure}[htbp]
    \centering
    \subfigure[BLUR]{
        \includegraphics[width=0.23\textwidth]{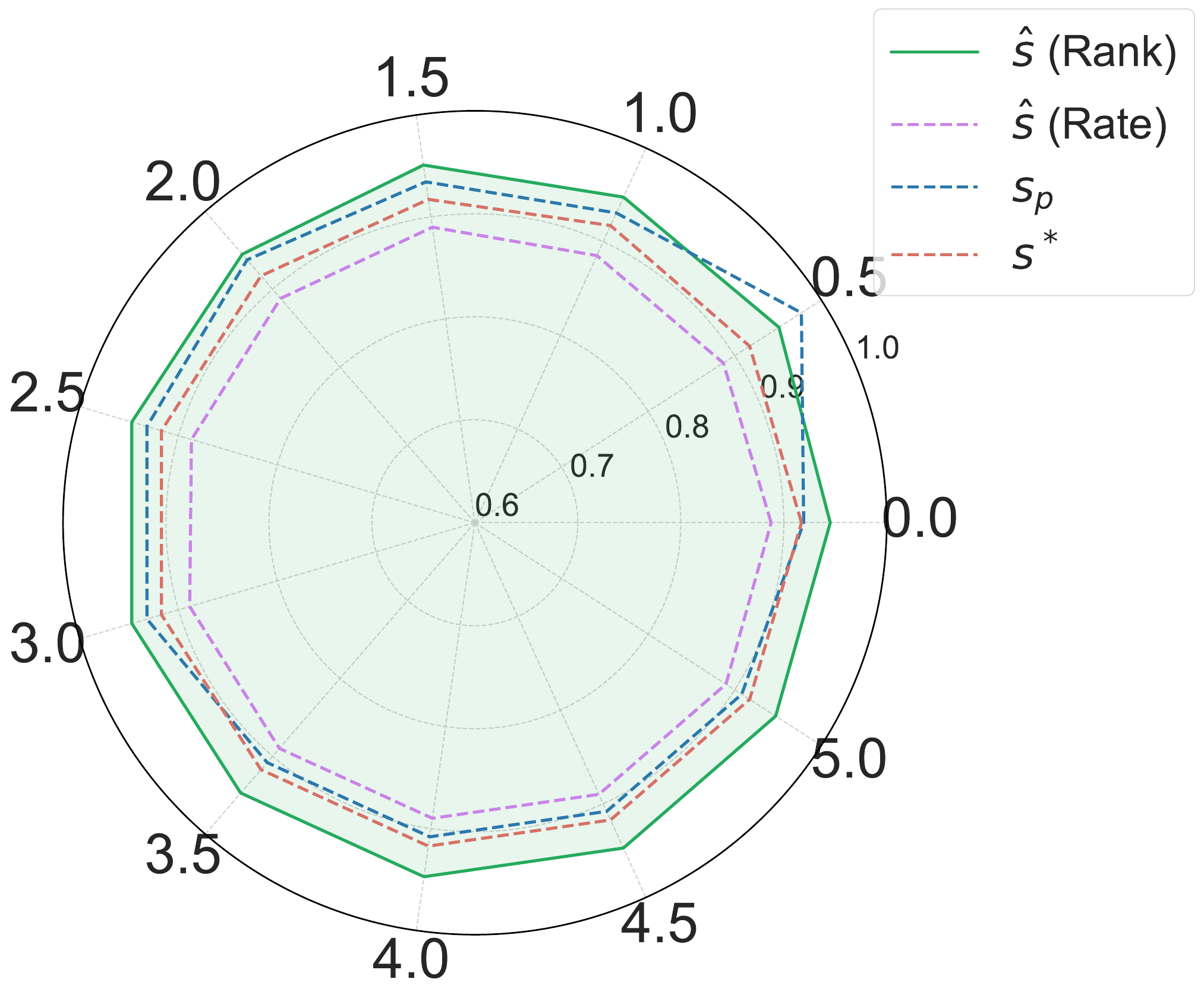}
    }
    \subfigure[LOCAL\_BLUR]{
        \includegraphics[width=0.23\textwidth]{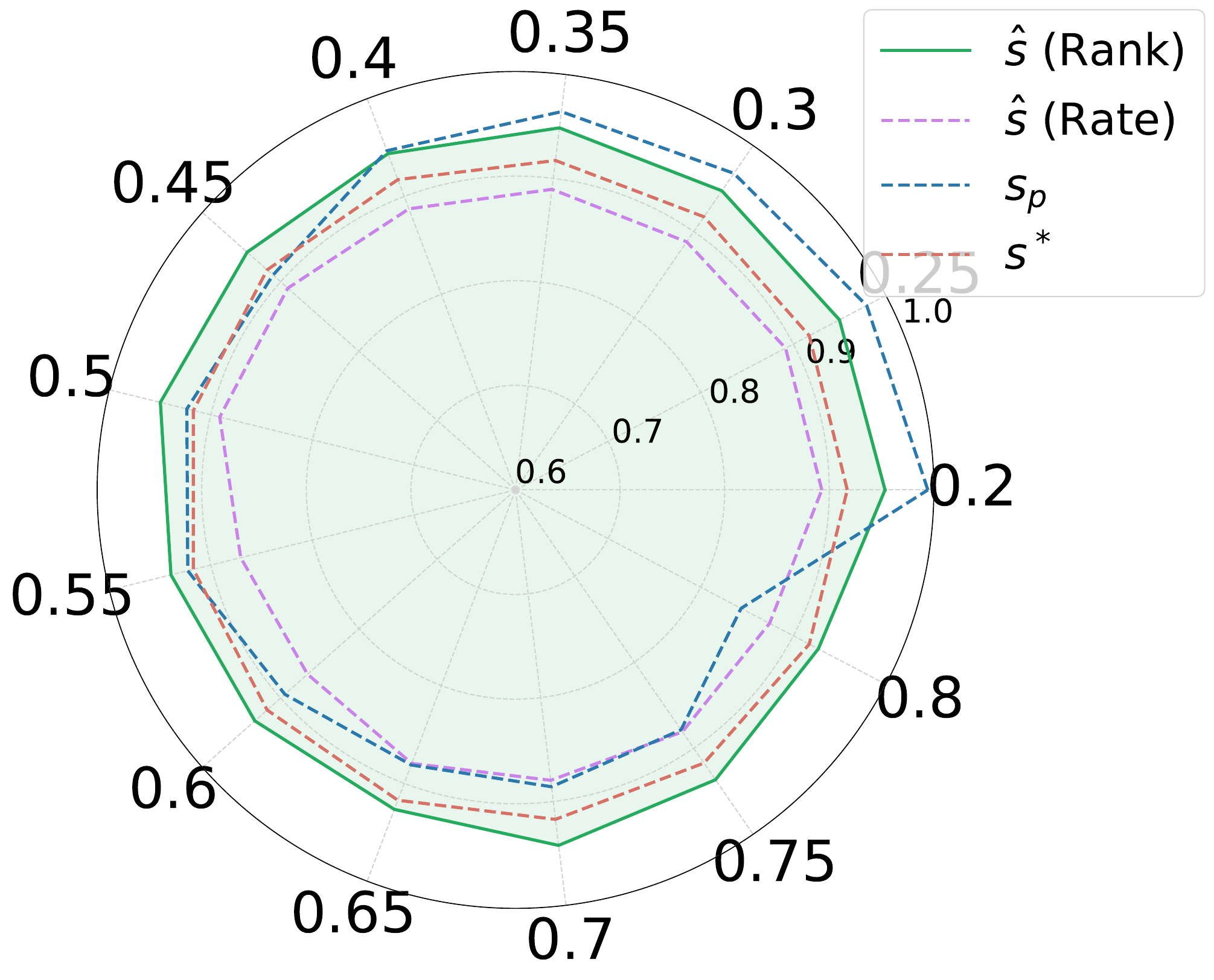}
    }
    \subfigure[LOCAL\_NOISE]{
        \includegraphics[width=0.23\textwidth]{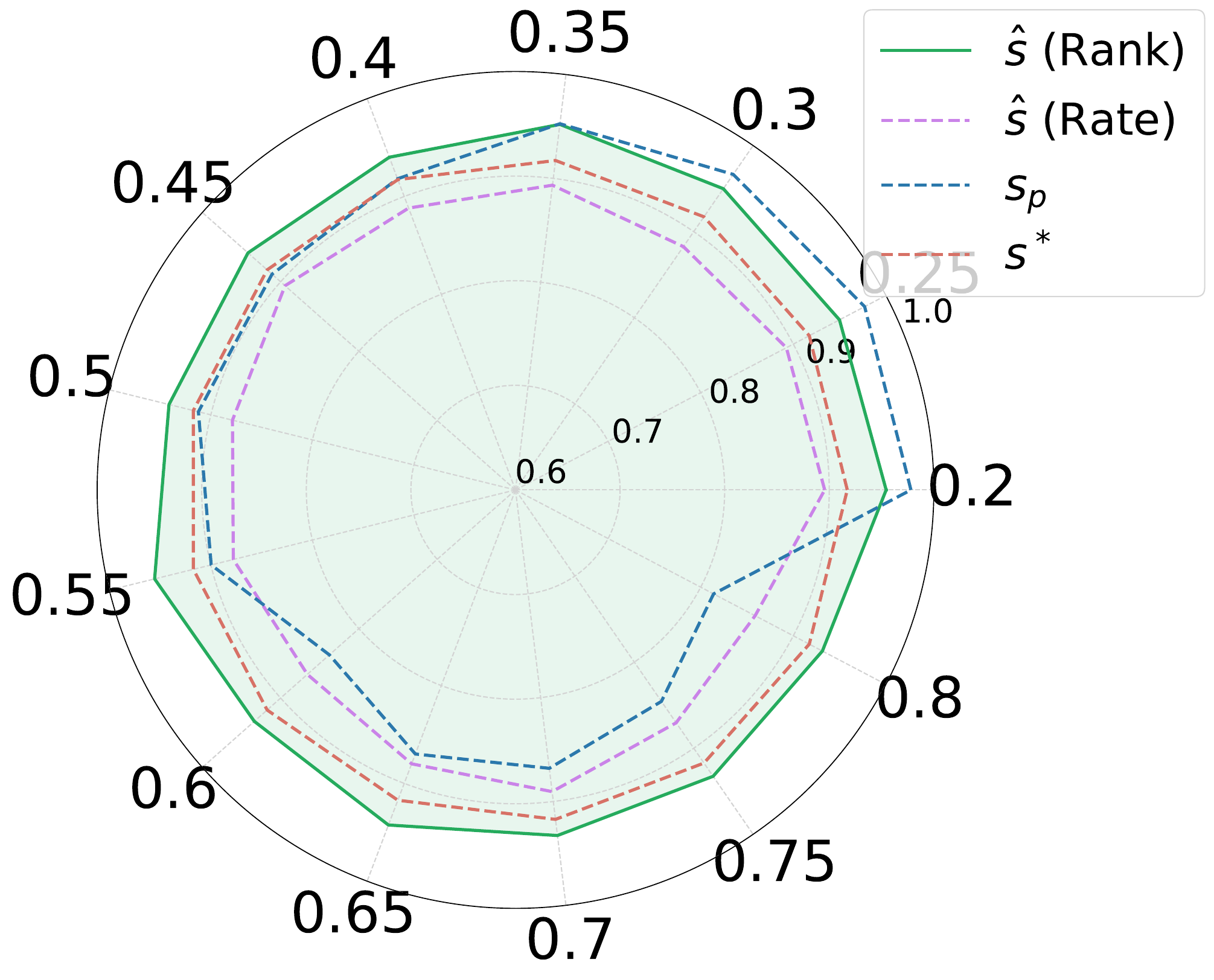}
    }
    \subfigure[WHITEOUT]{
        \includegraphics[width=0.23\textwidth]{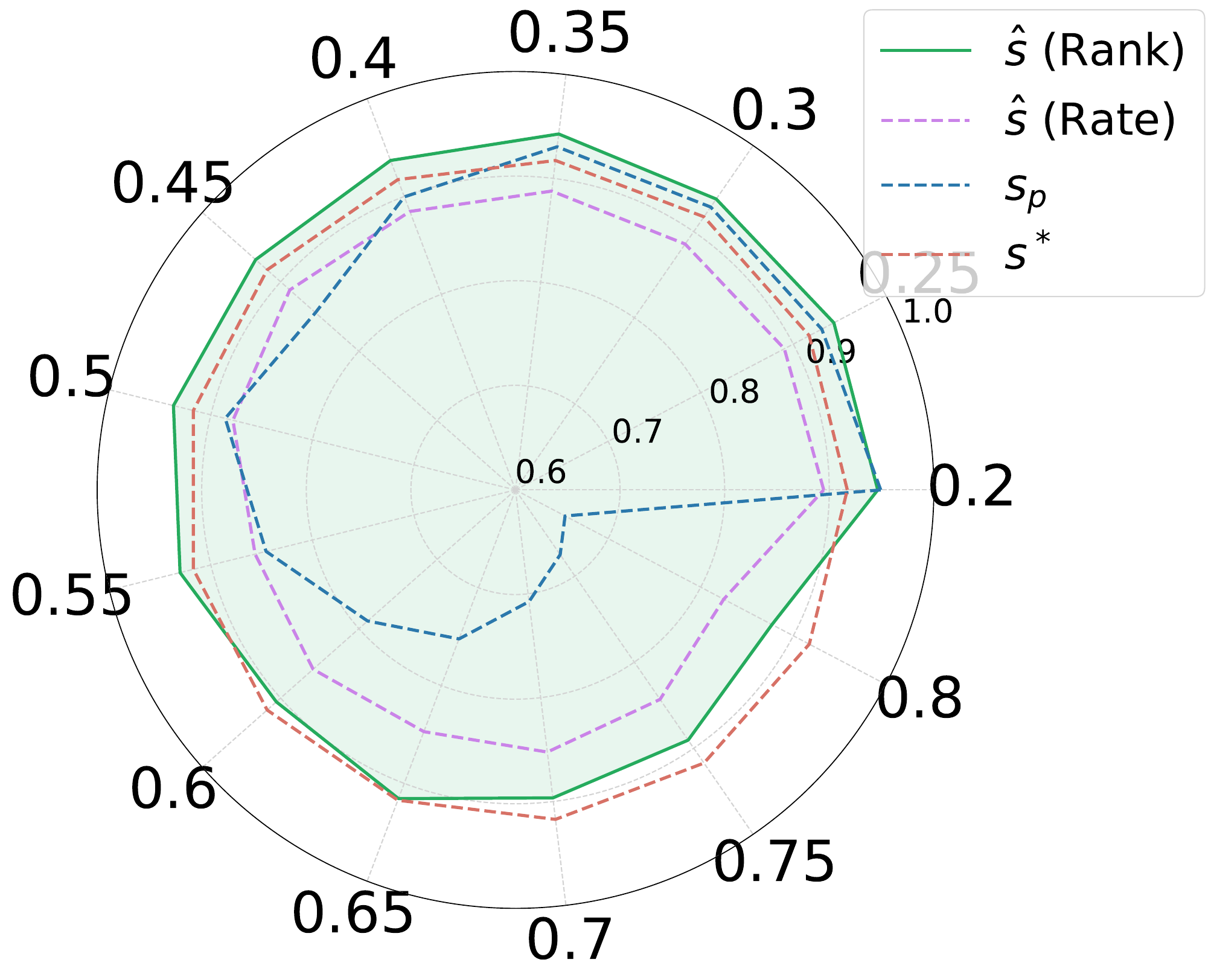}
    }
    \caption{Radar plots under different noise conditions (HRA-N).}
    \label{fig:radar_HRA-N}
\end{figure}

\section{\offer{Additional Related Work}}
\label{sec:related}
\offer{\textbf{Judgement Aggregation and Rank Aggregation.}
Inferring a consensus from multiple noisy annotators is a common problem in crowdsourcing and psychometrics~\citep{Murphy2003MixturesOD,Liu_Zhao_Liao_Lu_Xia_2019,hu2022explaining}. Classic models, like Dawid \& Skene’s ~\citep{bd9dfdb6-b296-3318-8bf2-0c827da00fd8} model estimate a latent true label or score while capturing annotator error rates, with extensions incorporating bias and expertise. 
Ranking models, such as Bradley–Terry or Thurstone–Mosteller variants, account for heterogeneous annotator reliability~\citep{JMLR:v23:20-315}. Our Stage-1 algorithm accounts for rater heterogeneity to estimate the consensus ranking $\mathbf s^{*}$, which is supported by Theorem~\ref{thm:cov-compare} and experiments, outperforms homogeneous annotator assumptions. Unlike prior methods focused solely on consensus, \AtC\ uses it as an ordinal scaffold for predictive model calibration, enhancing downstream efficiency with formal guarantees on predictive performance.}

\offer{\textbf{Model Calibration and Isotonic Regression.}
Calibrating model outputs to better align with empirical outcomes is a common practice in machine learning, especially for probabilistic predictions. Isotonic regression in particular is a non-parametric technique that enforces output monotonicity and has been used to calibrate scores to observed binary labels or probabilities~\citep{10.1214/11-AOAS504,baan-etal-2022-stop}. In \AtC, we novelly apply isotonic regression to calibrate a continuous model predictions to a rank-based target derived from human judgments. This usage differs from standard calibration, as we are mapping model scores to an ordering rather than absolute ground-truth values. Prior work on isotonic regression and calibration~\citep{NEURIPS2021_eaf76caa,chatterjee2015risk, bellec2018sharp} has not, to our knowledge, combined isotonic regression with a crowd-sourced ranking constraint. Theorem~\ref{thm:Robustness} provides calibration error analysis with imperfect rankings, demonstrating \AtC's ability to effectively calibrate using inferred consensus rather than fixed ground-truth labels. Unlike ~\citep{bellec2018sharp}, our innovation lies in modeling subjective optimum $\tilde{\mathbf s}$ as a random variable with allowable non-zero bias $\boldsymbol \nu$, enabling monotonic calibration when only relative judgments are available. We choose $L_2$ projection instead of $L_1$ or $L_\infty$ alternatives because it corresponds to maximum likelihood estimation under Gaussian noise and ensures unique solutions due to strict convexity~\citep{stout2023lpisotonicregressionalgorithms}.}

\textbf{Human-AI Complementarity.} Our work is also related to methods that integrate human and machine decision-making to improve prediction quality. Such approaches seek to optimally fuse human labels and algorithmic learning~\citep{JMLRraykar10a,yang2016hcomp,Li2018HybridMSTAH,bayesian-complementarity-2022}. Notably, calibration via human judgment has emerged as a practical entry point for real-world assessment, such as logistics area difficulty assessment~\citep{RAICA}. \AtC\ extends this line of work in two ways: it generalizes the framework beyond domain-specific settings, and provides an optimality theorem (Theorem~\ref{thm:optimality}) that explains when and why the two-stage pipeline outperforms either component alone, which gap is common across prior human-in-the-loop methods~\citep{NEURIPS2024_2375085c,alur2024human}. Unlike \citep{JMLR:v23:20-315,barcw}, which impose linear model constraints, \AtC\ places no restrictions on the aggregation method or the predictive model, making the design modular and broadly applicable.

\section{Conclusion}
\label{sec:conclusion}

We introduced \AtC, a two-stage framework that addresses the fundamental challenge of human-centered evaluation when ground truth is costly, unobservable, or future-dependent. Our theoretical analysis establishes three key results: (1) heterogeneous annotator modeling yields more efficient consensus estimation than homogeneous assumptions, (2) isotonic calibration provides risk bounds even under misspecified rankings, and (3) \AtC\ asymptotically outperforms model-only assessment. Empirically, \AtC\ demonstrates consistent improvements in accuracy and robustness across semi-synthetic and real-world datasets, particularly under data degradation conditions common in practice. 
The framework suits human-centered assessment tasks where model-only evaluation faces limitations due to unobservable or costly ground truth, such as when models receive degraded inputs while humans access complete information, or when collecting human comparisons is more cost-effective than measuring ground truth directly. 
Future work can extend \AtC\ in several directions: multi-dimensional assessments in which items are judged along multiple criteria, and domains with expert polarization that require multiple distinct consensus rankings.
Another direction is to extend our formulation to LLM-as-a-judge settings, where generalizing the heterogeneous annotator to language model evaluators would help build automated and adaptive assessment systems. We believe that addressing these challenges will broaden the impact of collaborative human-AI assessment systems.

\newpage
\section*{Statement}
This paper is the extended version of \citep{xie2026atc}, incorporating additional content that could not be included in the conference submission due to page limitations. We release code and datasets at \url{https://github.com/CAMELLIAxt/12500_AtC_supp}.
\normalem
\bibliography{iclr2026_conference}
\bibliographystyle{iclr2026_conference}

\newpage
\appendix
% \tableofcontents
% \input{appendix/A0}
\section{Asymptotic Efficiency of the Heterogeneous Thurstone Model Estimator}
\label{Appendix-of-cov}

This section details the asymptotic properties of the score estimators derived from the Heterogeneous Thurstone Model (HTM) and compares its efficiency against an estimator from a misspecified homogeneous model.

\subsection{Model Definitions and Problem Statement}

We consider the task of estimating item scores $\mathbf{s} \in \mathbb{R}^n$ from pairwise comparison data. The comparisons are provided by $m$ users, each user $u$ performs $k_u$ comparisons, leading to a total of $N = \sum_{u = 1}^m k_u$ observations $Y = \{Y_{ul}\}_{u=1..m, l=1..k_u}$. Each $Y_{ul} \in \{0,1\}$ represents the outcome of a comparison between items $i_l$ and $j_l$ by user $u$.

\subsubsection{The \textbf{Hetero}geneous Thurstone Model}
We assume the true data generating process follows a HTM. Under this model, the probability that user $u$ prefers item $i_l$ over item $j_l$ is given by:
\begin{align}
    P(Y_{ul}=1 | \mathbf{s}, \gamma_u) = F(\gamma_u a_{l,u}^\top \mathbf{s}), 
\end{align}
where $\mathbf{s} \in \mathbb{R}^n$ is the vector of true item scores, $\gamma_u \in \mathbb{R}^+$ is the accuracy parameter for user $u$, $a_{l,u} = e_{i_l} - e_{j_l}$ is a difference vector with $e_i$ being the $i$-th standard basis vector, and $F(\cdot)$ is a known link function, typically a Cumulative Distribution Function (CDF) such as the logistic function (leading to the Heterogeneous Bradley-Terry-Luce model, HBTL) or the Normal CDF (leading to the Heterogeneous Thurstone Case V model, HTCV).

The complete parameter vector for the HTM is $\theta = (\mathbf{s}, \gamma) \in \mathbb{R}^{n+m}$, where $\gamma = (\gamma_1, ..., \gamma_m)^\top$. The log-likelihood for a single observation $(Y_{ul}, a_{l,u})$ given parameters $(\mathbf{s}, \gamma_u)$ is denoted by $\ell(Y_{ul} | \mathbf{s}, \gamma_u) = \log P(Y_{ul} | \mathbf{s}, \gamma_u)$. The average log-likelihood for all $N$ observations is:
\begin{align}
    \mathcal{L}_{hete}(\theta; Y) = \frac{1}{N} \sum_{u=1}^m \sum_{l=1}^{k_u} \ell(Y_{ul} | \mathbf{s}, \gamma_u) .
\end{align}

The Maximum Likelihood Estimator (MLE) for the HTM is: 
\begin{align}
    \hat{\theta}_{hete} = (\hat{\mathbf{s}}_{hete}, \hat{\gamma}_{hete}) = \arg\max_{\theta} \mathcal{L}_{hete}(\theta; Y).
\end{align}
We are primarily interested in the score estimator $\hat{\mathbf{s}}_{hete}$.

\subsubsection{The \textbf{Homo}geneous Thurstone Model}
For comparison, we consider a misspecified HTM. This model incorrectly assumes that all users have the same accuracy parameter, $\gamma_u = \gamma_0$ for all $u$, where $\gamma_0$ is a fixed constant or a single parameter to be estimated. For simplicity and to highlight the effect of ignoring heterogeneity, we assume $\gamma_0$ is a known constant (e.g., $\gamma_0=1$).

 The probability of preference under this model is:
\begin{align}
    P(Y_{ul}=1 | \mathbf{s}, \gamma_0) = F(\gamma_0 a_{l,u}^\top \mathbf{s}) 
\end{align}
\begin{align}
     \mathcal{L}_{homo}(\mathbf{s}; Y, \gamma_0) = \frac{1}{N} \sum_{u=1}^m \sum_{l=1}^{k_u} \log P(Y_{ul} | \mathbf{s}, \gamma_0) 
\end{align}
The estimator obtained by maximizing $\hat{\mathbf{s}}_{homo} = \arg\max_{\mathbf{s}} \mathcal{L}_{homo}(\mathbf{s}; Y, \gamma_0)$, the misspecified likelihood, is a Quasi-Maximum Likelihood Estimator (QMLE)\citep{greene2003econometric}[Page 734].

% \subsubsection{Problem Statement}
The primary objective of Theorem~\ref{thm:cov-compare} is to demonstrate that, when the true data generating process is heterogeneous, the estimator $\hat{\mathbf s}_{hete}$ derived from the correctly specified HTM is asymptotically more efficient than the estimator $\hat{\mathbf{s}}_{homo}$ derived from the misspecified homogeneous model. This will be shown by comparing their asymptotic covariance matrices.

\subsection{Core Idea and Asymptotic Framework}
The efficiency comparison between $\hat{\mathbf{s}}_{hete}$ and $\hat{\mathbf{s}}_{homo}$ is grounded in the asymptotic theory of MLE and QMLE. The core idea is that an estimator derived from a correctly specified model should be asymptotically more efficient than an estimator derived from a misspecified model.

Under the stated assumptions, the MLE $\hat{\mathbf{s}}_{hete}$  from the correctly specified HTM is consistent for the true score vector $\mathbf{s}^*$, asymptotically normal, and asymptotically efficient. Its asymptotic covariance matrix, denoted $\Sigma_{\hat{\mathbf{s}}_{hete}}$, achieves the Cramér-Rao Lower Bound (CRLB) for estimating $\mathbf{s}^*$ under the HTM. This matrix is derived from the inverse (or pseudoinverse, to account for identifiability constraints) of the Fisher information matrix associated with the HTM, specifically from the Schur complement corresponding to the $\mathbf{s}$ parameters.

Conversely, the QMLE $\hat{\mathbf{s}}_{homo}$ from the misspecified homogeneous model will converge to a parameter vector $\mathbf{s}_*$ that minimizes the Kullback-Leibler divergence between the true data generating process and the misspecified model; generally, $\mathbf{s}_* \neq \mathbf{s}^*$ when true heterogeneity exists. The asymptotic covariance matrix of $\hat{\mathbf{s}}_{homo}$, denoted $\Sigma_{\hat{\mathbf{s}}_{homo}}$, is given by the ``sandwich" formula, which accounts for the misspecification. Due to the model misspecification, the standard information matrix equivalence does not hold for the homogeneous model, and its QMLE is typically not asymptotically efficient for $\mathbf{s}^*$. Moreover, the asymptotic covariance matrix of an estimator from a correctly specified model will be no larger than that of an estimator from a misspecified model in the Loewner sense. We aim to formalize this by showing that $\Sigma_{\hat{\mathbf{s}}_{homo}} - \Sigma_{\hat{\mathbf{s}}_{hete}}$ is a positive semi-definite matrix, and strictly positive definite in the identifiable subspace when true heterogeneity is present.

To formally compare the efficiency of the estimators $\hat{\mathbf{s}}_{hete}$ and $\hat{\mathbf{s}}_{homo}$, we first derive their respective asymptotic covariance matrices by deriving Lemma ~\ref{lem:hetero-cov} and Lemma ~\ref{lem:homo-cov}.

\subsection{Proof of Lemma ~\ref{lem:hetero-cov}}
\label{app:proof_lemma3.1}
The estimator $\hat{\mathbf{s}}_{hete}$ is derived from the HTM. The average negative log-likelihood function for the HTM, with parameters $\theta = (\mathbf{s}, \gamma)$, is given by
\begin{align}
    \mathcal{L}(\mathbf{s},\gamma) = -\frac{1}{N}\sum_{u=1}^{m} \sum_{l=1}^{k_u} \log P(Y_{ul} | \mathbf{s}, \gamma_u)
\end{align}
where $N = \sum_{u=1}^m k_u$ is the total number of observations, and $P(Y_{ul} | \mathbf{s}, \gamma_u) = F(\gamma_{u} a_{l,u}^{\top} \mathbf{s})$ if we denote $g(x; Y) = -\log P(Y|x)$ as the negative log-likelihood for a single observation. For simplicity, we adopt the notation from prior work where $g(x) = -\log F(x)$ is used when the outcome $Y_{ul}$ is $1$ or $F(x)$ represents the probability of a specific outcome that leads to the form in the loss function.

The score functions, which are the first-order partial derivatives of $\mathcal{L}(\mathbf{s},\gamma)$, are:
\begin{align}
    \nabla_{\mathbf{s}}\mathcal{L}(\mathbf{s},\gamma) &= \frac{1}{N}\sum_{u=1}^{m} \sum_{l=1}^{k_u} g'\left(\gamma_{u} a_{l,u}^{\top}\mathbf{s} \right)\gamma_{u}a_{l,u} \\
    \nabla_{\gamma}\mathcal{L}(\mathbf{s},\gamma) &= \frac{1}{N}\begin{bmatrix}
        \sum_{l=1}^{k_1}g '\left(\gamma_1 a_{l,1}^{\top} \mathbf{s} \right)a_{l,1}^{\top}\mathbf{s}\\
        \vdots\\
        \sum_{l=1}^{k_m}g'\left(\gamma_{m} a_{l,m}^{\top}\mathbf{s}\right)a_{l,m}^{\top}\mathbf{s}
        \end{bmatrix}
\end{align}

The Hessian matrix comprises the second-order partial derivatives:
\begin{align}
    \nabla_{\mathbf{s}}^2 \mathcal{L}(\mathbf{s},\gamma) &= \frac{1}{N}\sum_{u=1}^{m} \sum_{l=1}^{k_u} g''\left(\gamma_{u}a_{l,u}^{\top}\mathbf{s}\right)(\gamma_{u})^2a_{l,u}a_{l,u}^{\top} \\
    \nabla_{\gamma}^2 \mathcal{L}(\mathbf{s},\gamma) &= \frac{1}{N} \text{diag} \left( \left[ \sum_{l=1}^{k_u}g''\left(\gamma_{u} a_{l,u}^{\top}\mathbf{s}\right)(a_{l,u}^{\top}\mathbf{s})^2 \right]_{u=1}^m \right) \\
    (\nabla_{\mathbf{s}} \nabla_{\gamma} \mathcal{L}(\mathbf{s},\gamma))_{i,u} &= \frac{1}{N}\sum_{l=1}^{k_u}(a_{l,u})_i\left[g^{\prime\prime}\left(\gamma_ua_{l,u}^\top \mathbf{s}\right)(a_{l,u}^\top \mathbf{s})\gamma_u+g^{\prime}\left(\gamma_ua_{l,u}^\top \mathbf{s}\right)\right]
\end{align}
% where $(\nabla_{\mathbf{s}} \nabla_{\gamma} \mathcal{L}(\mathbf{s},\gamma))_{i,u}$ denotes the element in the $i$-th row and $u$-th column of the $n \times m$ matrix of mixed partial derivatives.

The Fisher information matrix for a single observation $I_{hete}(\theta_0)$, evaluated at the true parameters $\theta_0 = (\mathbf{s}^*, \gamma^*)$, is defined as the negative expectation of the Hessian of the log-likelihood for that observation. For the average log-likelihood $\mathcal{L}$, the corresponding expected Hessian (or information scaled by $1/N$) is $I(\mathbf{s}^*, \gamma^*) = -E_{Y|\mathbf{s}^*,\gamma^*}[\nabla^2 \mathcal{L}(\mathbf{s}^*,\gamma^*)]$. The blocks of this matrix are:
\begin{align}
    I_{ss} &= -E_{Y|\mathbf{s}^*,\gamma^*}[\nabla_{\mathbf{s}}^{2}\mathcal{L}(\mathbf{s}^*, \gamma^*)] \\
    (I_{\gamma\gamma})_{uu} &= -E_{Y|\mathbf{s}^*,\gamma^*}[(\nabla_{\gamma}^{2}\mathcal{L}(\mathbf{s}^*, \gamma^*))_{uu}] \\
    (I_{s\gamma})_{i,u} &= -E_{Y|\mathbf{s}^*,\gamma^*}[(\nabla_{\mathbf{s}}\nabla_{\gamma}\mathcal{L}(\mathbf{s}^*, \gamma^*))_{i,u}]
\end{align}
and $I_{\gamma s} = I_{s\gamma}^\top$.

The log-likelihood for a single comparison $Y_{ul}$ given $\mathbf{s}^*$ and $\gamma_u^*$ is:
\begin{align}
    \ell_{ul}(\mathbf{s}^*,\gamma_u^*) = Y_{ul} \log F(x^*_{ul}) + (1-Y_{ul}) \log(1-F(x^*_{ul})),
\end{align}
where $x^*_{ul} = \gamma_u^* a_{l,u}^\top \mathbf{s}^*$. Let $g(x; Y) = -[Y \log F(x) + (1-Y) \log(1-F(x))]$.
Under standard regularity conditions for MLE and assuming the model is correctly specified (Assumption A3):
\begin{align}
    E_{Y_{ul}|x^*_{ul}}[g'(x^*_{ul}; Y_{ul})] = 0,
\end{align} 
\begin{align}
    -E_{Y_{ul}|x^*_{ul}}[g''(x^*_{ul}; Y_{ul})] = W(x^*_{ul}), 
\end{align}
where $W(x) = \frac{(F'(x))^2}{F(x)(1-F(x))}$.

Substituting these expectations, the blocks of the Fisher information matrix (scaled by $1/N$) become:
\begin{align}
    I_{ss} &= \frac{1}{N}\sum_{u=1}^{m} \sum_{l=1}^{k_u} W(x^*_{ul}) (\gamma^*_u)^2 a_{l,u}a_{l,u}^{\top} \label{eq:Iss_final} \\
    (I_{\gamma\gamma})_{uu} &= \frac{1}{N}\sum_{l=1}^{k_u} W(x^*_{ul}) (a_{l,u}^{\top}\mathbf{s}^*)^2 \label{eq:Igg_final} \\
    (I_{s\gamma})_{i,u} &= \frac{1}{N}\sum_{l=1}^{k_u} W(x^*_{ul}) (a_{l,u}^{\top}\mathbf{s}^*) \gamma^*_u (a_{l,u})_i \label{eq:Isg_final}
\end{align}

The asymptotic covariance matrix of the MLE $\hat{\theta}_{hete} = (\hat{\mathbf{s}}_{hete}, \hat{\gamma}_{hete})$ is $N^{-1}[I_{hete}(\theta_0)]^{-1}$, where $I_{hete}(\theta_0)$ is the Fisher information for a single observation. The matrix $I(\mathbf{s}^*, \gamma^*)$ derived above corresponds to $N^{-1} I_{hete}(\theta_0)$ if we interpret the sums as averages per observation. We define:
\begin{align}
    S_{total} = N \cdot S = N \cdot (I_{ss} - I_{s\gamma}I_{\gamma\gamma}^{-1}I_{\gamma s}),
\end{align}
where $I_{ss}, I_{s\gamma}, I_{\gamma\gamma}$ are the blocks of $N \cdot I(\mathbf{s}^*, \gamma^*)$ (i.e., sums without the $1/N$ scaling).

The asymptotic covariance matrix for $\sqrt{N}(\hat{\mathbf{s}}_{hete} - \mathbf{s}^*)$ is then the top-left $(n \times n)$ block of $[I_{hete}(\theta_0)]^{-1}$. Using the formula for the inverse of a partitioned matrix, this block is:
\begin{align}
    (S_{total}/N)^{+} = S^{+}, 
\end{align}

where $S = I_{ss}' - I_{s\gamma}'(I_{\gamma\gamma}')^{-1}I_{\gamma s}'$, and $I_{ss}', I_{s\gamma}', I_{\gamma\gamma}'$ are the blocks of the Fisher information for a single observation (i.e., the expressions Eq.~\ref{eq:Iss_final}-Eq.~\ref{eq:Isg_final} without the $1/N$ factor and with sums replaced by expectations or a single representative term if all $k_u=k$ and $N=mk$).

Due to the identifiability constraint:
\begin{align}
    \mathbf{1}^\top \mathbf{s}^* = 0,
\end{align}
the score $\mathbf{s}$ is estimable only up to an additive constant. This implies that the Fisher information matrix $I_{hete}(\theta_0)$ is singular, and specifically, its block $S$ is singular with $S\mathbf{1} = \mathbf{0}$.
Let $S = U \Lambda U^\top$ be the spectral decomposition of $S$, where $U$ is an orthogonal matrix of eigenvectors and $\Lambda = \text{diag}(\lambda_1, ..., \lambda_{n-1}, 0)$ is the diagonal matrix of eigenvalues, with $\lambda_1, ..., \lambda_{n-1} > 0$. The Moore-Penrose pseudoinverse is:
\begin{align}
    S^{+} = U \Lambda^{+} U^\top = \sum_{i=1}^{n-1} \frac{1}{\lambda_i} u_i u_i^\top,
\end{align}
where $\Lambda^+ = \text{diag}(1/\lambda_1, 1/\lambda_2, ..., 1/\lambda_{n-1}, 0)$.
And this $S^{+}$ represents the inverse of $S$ restricted to the identifiable subspace:
\begin{align}
    V_{\perp\mathbf{1}} = \{ \mathbf{v} \in \mathbb{R}^n : \mathbf{1}^\top \mathbf{v} = 0 \}.
\end{align}
The asymptotic covariance matrix for $\sqrt{N}(\hat{\mathbf{s}}_{hete} - \mathbf{s}^*)$ is therefore $S^{+}$. This matrix is the Cramér-Rao Lower Bound for estimating $\mathbf{s}^*$ under the HTM, considering the estimation of $\gamma^*$ and the identifiability constraint on $\mathbf{s}^*$.
$\square$

\subsection{Proof of Lemma ~\ref{lem:homo-cov}}
\label{app:proof_lemma3.2}
The estimator $\hat{\mathbf{s}}_{homo}$ is obtained as the QMLE from the HTM, which assumes a common accuracy parameter $\gamma_0$ for all users. The average log-likelihood for this model is:
\begin{align}
    \mathcal{L}_{homo}(\mathbf{s}; Y, \gamma_0) = \frac{1}{N} \sum_{u=1}^m \sum_{l=1}^{k_u} \log P(Y_{ul} | \mathbf{s}, \gamma_0)
\end{align}
where $P(Y_{ul} | \mathbf{s}, \gamma_0) = F(\gamma_0 a_{l,u}^\top \mathbf{s})$.
Under Assumption (A2), the true data generating process is an HTM with heterogeneous accuracies $\gamma^* = (\gamma_1^*, ..., \gamma_m^*)^\top$, where not all $\gamma_u^*$ are equal to $\gamma_0$. Thus, the homogeneous model is misspecified.
According to \citep[Theorem 2.2]{white1982maximum}, the QMLE $\hat{\mathbf{s}}_{homo}$ converges in probability to a parameter vector $\mathbf{s}_*$, which is the unique minimizer of the Kullback-Leibler divergence between the true data generating process and the family of distributions defined by the misspecified homogeneous model. Equivalently, $\mathbf{s}_*$ maximizes the expected misspecified log-likelihood:
\begin{align}
    \mathbf{s}_* = \arg\max_{\mathbf{s}} E_{Y|\mathbf{s}^*,\gamma^*}[\log P(Y_{ul} | \mathbf{s}, \gamma_0)]
\end{align}
Due to the model misspecification , it is generally true that $\mathbf{s}_* \neq \mathbf{s}^*$.

Furthermore, according to \citep[Theorem 3.2]{white1982maximum}, the QMLE $\hat{\mathbf{s}}_{homo}$ is asymptotically normally distributed:
\begin{align}
    \sqrt{N}(\hat{\mathbf{s}}_{homo} - \mathbf{s}_*) \xrightarrow{d} N(\mathbf{0}, C_{homo}(\mathbf{s}_*))
\end{align}
where $C_{Homo}(\mathbf{s}_*)$ is the ``sandwich" covariance matrix given by
\begin{align}
    C_{homo}(\mathbf{s}_*) = [A_{homo}(\mathbf{s}_*)]^{-1} B_{homo}(\mathbf{s}_*) [A_{homo}(\mathbf{s}_*)]^{-1}
\end{align}
The matrices $A_{homo}(\mathbf{s}_*)$ and $B_{homo}(\mathbf{s}_*)$ are defined for a single observation under the true data generating process $(\mathbf{s}^*, \gamma^*)$, evaluated at $\mathbf{s}_*$:
\begin{align}
    A_{homo}(\mathbf{s}_*) &= E_{Y|\mathbf{s}^*,\gamma^*}[\nabla_{\mathbf{s}}^2 \log P(Y_{ul} | \mathbf{s}_*, \gamma_0)] \\
    B_{homo}(\mathbf{s}_*) &= E_{Y|\mathbf{s}^*,\gamma^*}[(\nabla_{\mathbf{s}} \log P(Y_{ul} | \mathbf{s}_*, \gamma_0))(\nabla_{\mathbf{s}} \log P(Y_{ul} | \mathbf{s}_*, \gamma_0))^\top]
\end{align}
The expectation $E_{Y|\mathbf{s}^*,\gamma^*}[\cdot]$ is taken with respect to the true probability distribution of $Y_{ul}$ determined by $\mathbf{s}^*$ and $\gamma^*$.
Due to the model misspecification, the information matrix equivalence \citep[Theorem 3.3]{white1982maximum} generally does not hold for the homogeneous model, meaning $A_{homo}(\mathbf{s}_*) \neq -B_{homo}(\mathbf{s}_*)$.

Moreover, to account for the identifiability constraint $\mathbf{1}^\top \mathbf{s} = 0$, the inverses in the sandwich formula are replaced by pseudoinverses. Thus, the asymptotic covariance matrix for $\sqrt{N}(\hat{\mathbf{s}}_{homo} - \mathbf{s}_*)$ is more precisely written as:
\begin{align}
    \text{AsyVar}(\sqrt{N}(\hat{\mathbf{s}}_{homo} - \mathbf{s}_*)) = [A_{homo}(\mathbf{s}_*)]^{+} B_{homo}(\mathbf{s}_*) [A_{homo}(\mathbf{s}_*)]^{+}
\end{align}

$\square$

\subsection{Proof of Theorem~\ref{thm:cov-compare}}
\label{app:proof_thm3.3}

We aim to compare the asymptotic efficiencies of $\hat{\mathbf{s}}_{hete}$ and $\hat{\mathbf{s}}_{homo}$ by examining their respective asymptotic covariance matrices.
From Lemma \ref{lem:hetero-cov}, the asymptotic covariance matrix for $\sqrt{N}(\hat{\mathbf{s}}_{hete} - \mathbf{s}^*)$ is $S^{+}$. This matrix represents the CRLB for estimating $\mathbf{s}^*$ under the correctly specified HTM, within the identifiable subspace defined by $\mathbf{1}^\top\mathbf{s}=0$. Thus,
\begin{align}
    \text{AsyVar}(\sqrt{N} \hat{\mathbf{s}}_{hete}) = S^{+} \label{eq:asyvar_hete}
\end{align}

From Lemma \ref{lem:homo-cov}, the asymptotic covariance matrix for $\sqrt{N}(\hat{\mathbf{s}}_{homo} - \mathbf{s}_*)$ is given by the sandwich formula:
\begin{align}
    \text{AsyVar}(\sqrt{N} \hat{\mathbf{s}}_{homo}) = [A_{homo}(\mathbf{s}_*)]^{+} B_{homo}(\mathbf{s}_*) [A_{homo}(\mathbf{s}_*)]^{+} \label{eq:asyvar_homo}
\end{align}
where $\mathbf{s}_*$ is the probability limit of $\hat{\mathbf{s}}_{homo}$ under the misspecified homogeneous model, and $A_{homo}(\mathbf{s}_*)$ and $B_{homo}(\mathbf{s}_*)$ are defined as in Lemma \ref{lem:homo-cov}.

To prove the first part of Theorem ~\ref{thm:cov-compare}, we compare the matrices from Eq.~\ref{eq:asyvar_homo} and Eq.~\ref{eq:asyvar_hete}.
According to the theory of estimation under misspecified models, the asymptotic covariance matrix of a QMLE from a misspecified model cannot be smaller (in the Loewner sense) than the CRLB achieved by an MLE from the correctly specified model, when both are estimating parameters related to the true data generating process. The matrix $S^{+}$ is the CRLB for $\mathbf{s}^*$ under the true HTM. Therefore, it must hold that:
\begin{align}
    [A_{homo}(\mathbf{s}_*)]^{+} B_{homo}(\mathbf{s}_*) [A_{homo}(\mathbf{s}_*)]^{+} \succeq S^{+} \label{eq:main_inequality}
\end{align}
This inequality establishes that:
\begin{align}
    N \cdot \Sigma_{\hat{\mathbf{s}}_{homo}} \succeq N \cdot \Sigma_{\hat{\mathbf{s}}_{hete}}.
\end{align}

For the second part of the theorem, concerning the strict inequality, Assumption (A2) states that true heterogeneity exists, meaning the true user accuracies $\gamma_u^*$ are not all equal to the fixed $\gamma_0$ used in the homogeneous model. This implies that the homogeneous model is genuinely misspecified.
Under such misspecification, the information matrix equivalence $A_{homo}(\mathbf{s}_*) = -B_{homo}(\mathbf{s}_*)$ does not hold. The QMLE $\hat{\mathbf{s}}_{homo}$ fails to incorporate user-specific accuracy information, leading to a loss of statistical efficiency compared to the MLE $\hat{\mathbf{s}}_{hete}$ from the correctly specified HTM. This efficiency loss manifests as a strict inequality in Eq.~\ref{eq:main_inequality} when considering the identifiable subspace $V_{\perp\mathbf{1}} = \{ \mathbf{v} \in \mathbb{R}^n : \mathbf{1}^\top \mathbf{v} = 0 \}$.
Specifically, for any non-zero vector $\mathbf{v} \in V_{\perp\mathbf{1}}$,
\begin{align}
    \mathbf{v}^\top \left( [A_{homo}(\mathbf{s}_*)]^{+} B_{homo}(\mathbf{s}_*) [A_{homo}(\mathbf{s}_*)]^{+} \right) \mathbf{v} > \mathbf{v}^\top S^{+} \mathbf{v}
\end{align}
This occurs because the misspecification prevents the sandwich covariance matrix from collapsing to the simpler inverse Fisher information form and results in a larger variance for any linear combination of parameters within the identifiable subspace, compared to the CRLB achieved by the correctly specified model.
Thus, $N \cdot \Sigma_{\hat{\mathbf{s}}_{homo}} - N \cdot \Sigma_{\hat{\mathbf{s}}_{hete}}$ is strictly positive definite on $V_{\perp\mathbf{1}}$.
$\square$

\section{Analysis of LSE under Model Misspecification and Non-zero Mean Noise}
\label{Appendix:th2}

\subsection{Preliminaries: Cones and Statistical Dimension}
\label{Appendix:th2_Preliminaries}
This section introduces fundamental concepts related to convex cones, tangent cones, and statistical dimension, which are essential for the subsequent analysis.

\begin{definition}[Convex Cone]
A set $K \subset \mathbb{R}^n$ is a \textbf{convex cone} if for any $\boldsymbol{v}_1, \boldsymbol{v}_2 \in K$ and any non-negative scalars $t_1, t_2 \ge 0$, the linear combination $t_1 \boldsymbol{v}_1 + t_2 \boldsymbol{v}_2$ is also in $K$. If, additionally, $K$ is a closed set, it is a closed convex cone.
\end{definition}

\begin{definition}[Tangent Cone]
Let $K \subset \mathbb{R}^n$ be a closed convex set and let $\boldsymbol{u} \in K$. The \textbf{tangent cone} to $K$ at $\boldsymbol{u}$, denoted by $\mathcal T_{K, \boldsymbol{u}}$, is defined as:
\begin{align}
     \mathcal T_{K, \boldsymbol{u}} = \overline{ \{ t (\boldsymbol{v} - \boldsymbol{u}) : t > 0, \boldsymbol{v} \in K \} }
\end{align}

where $\overline{\{\cdot\}}$ denotes the closure of the set.
Intuitively, the tangent cone $\mathcal {T}_{K, \boldsymbol{u}}$ comprises all feasible directions from $\boldsymbol{u}$ scaled by non-negative scalars, along which one can move infinitesimally while remaining within $K$.

If $K$ is a closed convex cone, the definition of the tangent cone at $\boldsymbol{u} \in K$ simplifies due to the cone property (i.e., if $\boldsymbol{x} \in K$, then $t\boldsymbol{x} \in K$ for $t \ge 0$). For a closed convex cone $K$ and $\boldsymbol{u} \in K$, the tangent cone is given by :
\begin{align}
     \mathcal T_{K, \boldsymbol{u}} = \overline{ K + \mathrm{span}(\{-\boldsymbol{u}\}) } = \overline{ \{ \boldsymbol{v} - t \boldsymbol{u} : \boldsymbol{v} \in K, t \ge 0 \} } 
\end{align}
where the sum is the Minkowski sum.
\end{definition}

\begin{definition}[Isotonic Cone]
The isotonic cone (or non-decreasing cone) in $\mathbb R^n$, denoted by $\mathcal M$, is defined as:
$$ \mathcal M := \{ \boldsymbol{x} = (x_1, \ldots, x_n)^T \in \mathbb R^n : x_1 \le x_2 \le \cdots \le x_n \} $$
The set $\mathcal M$ is a closed convex polyhedral cone. For any $\boldsymbol{u} \in \mathcal M$, the tangent cone $\mathcal T_{\mathcal M, \boldsymbol{u}}$ captures the local directional constraints at $\boldsymbol{u}$. If $u_i < u_{i+1}$ for some $i$, the tangent cone allows more freedom in the $i$-th and $(i+1)$-th coordinates compared to when $u_i = u_{i+1}$, where the constraint $x_i \le x_{i+1}$ becomes active for directions in the cone.
\end{definition}

\begin{definition}[Statistical Dimension]
Let $K \subset \mathbb R^n$ be a closed convex cone. Its statistical dimension, denoted by $\delta(K)$, is defined as:
\begin{align}
    \delta(K) := \mathbb{E} \left[ |\Pi_K(\mathbf{g})|_2^2 \right]
\end{align}
where $\mathbf{g} \sim \mathcal{N}(\mathbf{0}, I_{n})$ is a standard Gaussian vector in $\mathbb R^n$, and $ \Pi_K(\mathbf{g})$ is the Euclidean projection of $\boldsymbol g$ onto $K$.
Equivalently, the statistical dimension can also be expressed as:
\begin{align}
    \delta(K) = \mathbb{E} \left[ \mathbf{g}^T \Pi_K(\mathbf{g}) \right] = \mathbb{E} \left[ \left( \sup_{\boldsymbol{\theta} \in K : |\boldsymbol{\theta}|_2 \leq 1} \mathbf{g}^T \boldsymbol{\theta} \right)^2 \right]
\end{align}
The statistical dimension provides a measure of the ``complexity" of the cone from a statistical perspective, often related to the effective number of parameters in a model constrained to the cone.
\end{definition}

\subsection{Problem Setting and Relation to Prior Work}
Let $\mathbf{s} \in \mathbb{R}^n$ be an underlying ground truth score vector. The primary target of our estimation in the second stage is a related vector $\tilde{\mathbf{s}} \in \mathbb{R}^n$, which can be thought of as a ``subjective optimal point" or a noisy version of $\mathbf{s}$. Specifically, $\tilde{\mathbf{s}}$ is generated as:
\begin{align}
    \tilde{\mathbf{s}} = \mathbf{s} + \boldsymbol{\tilde{\epsilon}},
\end{align}
where $\boldsymbol{\tilde{\epsilon}}\sim \mathcal{N}(\mathbf{0}, \tilde{\sigma}^2 I_n)$. Thus, $\tilde{\mathbf{s}}$ is a random vector with $\mathbb{E}[\tilde{\mathbf{s}}] = \mathbf{s}$. We assume the true isotonic cone relevant to our target $\tilde{\mathbf{s}}$ is $\tilde{\mathcal{M}} := \mathcal{M}_{\pi(\tilde{\mathbf{s}})}$ (or $\mathcal{M}_{\pi(\mathbf{s})}$, if the true ordering is dictated by the mean $\mathbf{s}$; this distinction should be noted, but for now, we assume $\tilde{\mathbf{s}} \in \tilde{\mathcal{M}}$ in expectation or with high probability).

In the first stage of our overall procedure, a separate algorithm (e.g., \offer{Algorithm 1 from \citep{jin2020rank}}) processes auxiliary data to produce an estimate of latent scores, $\mathbf{s}^* \in \mathbb{R}^n$. The ranking induced by $\mathbf{s}^*$, denoted $\hat{\pi} := \pi(\mathbf{s}^*)$, defines a (random) isotonic cone $\hat{\mathcal{M}} := \mathcal{M}_{\hat{\pi}}$. This cone $\hat{\mathcal{M}}$ serves as the constraint set for the second stage.

The observations for the second stage, denoted by $\mathbf{s}_p \in \mathbb{R}^n$, are related to the ground truth $\mathbf{s}$ and a fixed, unknown systematic bias vector $\boldsymbol{\nu} \in \mathbb{R}^n$ as follows:
\begin{align}
    \mathbf{s}_p = \mathbf{s} + \boldsymbol{\nu} .
\end{align}
Using the relationship $\mathbf{s} = \tilde{\mathbf{s}} - \boldsymbol{\tilde{\epsilon}}$, we can express $\mathbf{s}_p$ in terms of our target $\tilde{\mathbf{s}}$:
$$ \mathbf{s}_p = (\tilde{\mathbf{s}} - \boldsymbol{\tilde{\epsilon}} ) + \boldsymbol{\nu} = \tilde{\mathbf{s}} + \boldsymbol{\nu} - \boldsymbol{\tilde{\epsilon}}$$
The final estimator for the target $\tilde{\mathbf{s}}$ is obtained by projecting $\mathbf{s}_p$ onto the estimated cone $\hat{\mathcal{M}}$:
$$ \hat{\mathbf{s}} := \Pi_{\hat{\mathcal{M}}}(\mathbf{s}_p) $$
Our objective is to analyze the risk of this estimator with respect to $\tilde{\mathbf{s}}$:
\begin{align}
    \mathbb{E}[\|\hat{\mathbf{s}} - \tilde{\mathbf{s}}\|_2^2].
\end{align}
The expectation is over all sources of randomness: $\boldsymbol{\tilde{\epsilon}} $ (which makes $\tilde{\mathbf{s}}$ random) and $\hat{\mathbf{s}}_{\text{alg1}}$ (which makes $\hat{\mathcal{M}}$ random).

The effective noise when estimating $\tilde{\mathbf{s}}$ from $\mathbf{s}_p$ is:
\begin{align}
    \boldsymbol{\xi} := \mathbf{s}_p - \tilde{\mathbf{s}} = \boldsymbol{\nu} - \boldsymbol{\tilde{\epsilon}} ,
\end{align}
in which $\boldsymbol{\xi} \sim \mathcal{N}(\boldsymbol{\nu}, \tilde{\sigma}^2 I_n)$, as $\mathbb{E}[\boldsymbol{\xi}] = \boldsymbol{\nu}$ and $\mathrm{Cov}(\boldsymbol{\xi}) = \mathrm{Cov}(-\boldsymbol{\tilde{\epsilon}} ) = \tilde{\sigma}^2 I_n$.

% \subsection{Oracle Inequality with Biased Effective Noise}
% \label{Appendix:th2-pre-proof}
\subsection{Proof of Theorem~\ref{thm:Robustness}}
\label{app:proof_thmRob}

% \begin{proof}
In stage-1, the target is $\tilde{\mathbf{s}}$, and our observation is $\mathbf{s}_p$. The effective noise is $\boldsymbol{\xi} = \mathbf{s}_p - \tilde{\mathbf{s}} \sim \mathcal{N}(\boldsymbol{\nu}, \tilde{\sigma}^2 I_n)$.

% \begin{lemma}[Oracle Inequality for LS Estimator with Biased Effective Noise] \citep[Proposition 2.1]{bellec2018sharp}
% \label{lem:adapted_bellec_prop2.1_final}
%  Let $\tilde{\mathbf{s}} \in \mathbb{R}^n$ be the target score vector (itself random, with $\mathbb{E}[\tilde{\mathbf{s}}] = \mathbf{s}$). Let $\hat{\mathcal{M}} \subset \mathbb{R}^n$ be a closed convex set (the estimated isotonic cone). The observations used for projection are $\mathbf{s}_p$, and the effective noise relating $\mathbf{s}_p$ to $\tilde{\mathbf{s}}$ is $\boldsymbol{\xi} = \mathbf{s}_p - \tilde{\mathbf{s}} \sim \mathcal{N}(\boldsymbol{\nu}, \tilde{\sigma}^2 I_n)$. The least squares estimator is $\hat{\mathbf{s}} = \Pi_{\hat{\mathcal{M}}}(\mathbf{s}_p)$.
% For any $\boldsymbol{u} \in \hat{\mathcal{M}}$, the conditional expected risk, given $\hat{\mathcal{M}}$ and $\tilde{\mathbf{s}}$, is bounded. We have:
% \begin{align}
%     \|\hat{\mathbf{s}} - \tilde{\mathbf{s}}\|_2^2 - \|\boldsymbol{u} - \tilde{\mathbf{s}}\|_2^2 \le \frac{1}{n} \left( \sup_{\boldsymbol{\theta} \in \mathcal{T}_{\hat{\mathcal{M}}, \boldsymbol{u}} : \|\boldsymbol{\theta}\|_2 \leq 1} \boldsymbol{\xi}^T \boldsymbol{\theta} \right)^2 
% \end{align}
% \end{lemma}

By Lemma ~\ref{lem:adapted_bellec_prop2.1_main}, we have:
\begin{align}
    \|\hat{\mathbf{s}} - \tilde{\mathbf{s}}\|_2^2 - \|\boldsymbol{u} - \tilde{\mathbf{s}}\|_2^2 \le \frac{1}{n} \left( \sup_{\boldsymbol{\theta} \in \mathcal{T}_{\hat{\mathcal{M}}, \boldsymbol{u}} : \|\boldsymbol{\theta}\|_2 \leq 1} \boldsymbol{\xi}^T \boldsymbol{\theta} \right)^2 
\end{align}

Taking expectation with respect to $\boldsymbol{\xi}$ (conditional on $\hat{\mathcal{M}}$ and $\tilde{\mathbf{s}}$), and choosing $\boldsymbol{u} = \Pi_{\hat{\mathcal{M}}}(\tilde{\mathbf{s}})$, we have:
\begin{align}
     \mathbb{E}_{\boldsymbol{\xi}} [\|\hat{\mathbf{s}} - \tilde{\mathbf{s}}\|_2^2 | \hat{\mathcal{M}}, \tilde{\mathbf{s}}] \le \|\Pi_{\hat{\mathcal{M}}}(\tilde{\mathbf{s}}) - \tilde{\mathbf{s}}\|_2^2 + \frac{1}{n} \mathbb{E}_{\boldsymbol{\xi}}\left[\left( \sup_{\boldsymbol{\theta} \in W'} \boldsymbol{\xi}^T \boldsymbol{\theta} \right)^2 \middle| \hat{\mathcal{M}}, \tilde{\mathbf{s}} \right] 
\end{align}
where $W' = \{ \boldsymbol{\theta} \in \mathcal{T}_{\hat{\mathcal{M}}, \Pi_{\hat{\mathcal{M}}}(\tilde{\mathbf{s}})} : \|\boldsymbol{\theta}\|_2 \leq 1 \}$.
Let $\boldsymbol{\epsilon} ' = \boldsymbol{\xi} - \boldsymbol{\nu} = - \boldsymbol{\tilde{\epsilon}}\sim \mathcal{N}(\mathbf{0}, \tilde{\sigma}^2 I_n)$.
The expectation term is bounded using $(a+b)^2 \le 2a^2 + 2b^2$:
\begin{align}
\mathbb{E}_{\boldsymbol{\xi}}\left[\left( \sup_{\boldsymbol{\theta} \in W'} ((\boldsymbol{\xi}-\boldsymbol{\nu})^T \boldsymbol{\theta} + \boldsymbol{\nu}^T \boldsymbol{\theta}) \right)^2 \middle| \hat{\mathcal{M}}, \tilde{\mathbf{s}} \right]
&\le 2 \mathbb{E}_{\boldsymbol{\epsilon} '}\left[\left( \sup_{\boldsymbol{\theta} \in W'} (\boldsymbol{\epsilon} ')^T \boldsymbol{\theta} \right)^2 \middle| \hat{\mathcal{M}}, \tilde{\mathbf{s}} \right] + 2 \left( \sup_{\boldsymbol{\theta} \in W'} \boldsymbol{\nu}^T \boldsymbol{\theta} \right)^2 \\
&= 2 \tilde{\sigma}^2 \delta(\mathcal{T}_{\hat{\mathcal{M}}, \Pi_{\hat{\mathcal{M}}}(\tilde{\mathbf{s}})}) + 2 \|\Pi_{\mathcal{T}_{\hat{\mathcal{M}}, \Pi_{\hat{\mathcal{M}}}(\tilde{\mathbf{s}})}}(\boldsymbol{\nu})\|_2^2
\end{align}
Thus, given $\hat{\mathcal{M}}$ and $\tilde{\mathbf{s}}$, the conditional risk is:
\begin{align}
    \mathbb{E}_{\boldsymbol{\xi}} [\|\hat{\mathbf{s}} - \tilde{\mathbf{s}}\|_2^2 | \hat{\mathcal{M}}, \tilde{\mathbf{s}}] \le \|\Pi_{\hat{\mathcal{M}}}(\tilde{\mathbf{s}}) - \tilde{\mathbf{s}}\|_2^2 + \frac{2\tilde{\sigma}^2}{n} \delta(\mathcal{T}_{\hat{\mathcal{M}}, \Pi_{\hat{\mathcal{M}}}(\tilde{\mathbf{s}})}) + \frac{2}{n} \|\Pi_{\mathcal{T}_{\hat{\mathcal{M}}, \Pi_{\hat{\mathcal{M}}}(\tilde{\mathbf{s}})}}(\boldsymbol{\nu})\|_2^2 
\end{align}
Taking the full expectation over $\hat{\mathcal{M}}$ and $\tilde{\mathbf{s}}$ :
\begin{align}
\mathbb{E}[\|\hat{\mathbf{s}} - \tilde{\mathbf{s}}\|_2^2] \le& \mathbb{E}_{\hat{\mathcal{M}}, \tilde{\mathbf{s}}} \left[ \|\Pi_{\hat{\mathcal{M}}}(\tilde{\mathbf{s}}) - \tilde{\mathbf{s}}\|_2^2 \right] \label{eq:term1_proj_error_final} \\
&+ \mathbb{E}_{\hat{\mathcal{M}}, \tilde{\mathbf{s}}} \left[ \frac{2\tilde{\sigma}^2}{n} \delta(\mathcal{T}_{\hat{\mathcal{M}}, \Pi_{\hat{\mathcal{M}}}(\tilde{\mathbf{s}})}) \right] \label{eq:term2_stat_dim_final} \\
&+ \mathbb{E}_{\hat{\mathcal{M}}, \tilde{\mathbf{s}}} \left[ \frac{2}{n} \|\Pi_{\mathcal{T}_{\hat{\mathcal{M}}, \Pi_{\hat{\mathcal{M}}}(\tilde{\mathbf{s}})}}(\boldsymbol{\nu})\|_2^2 \right] \label{eq:term3_bias_proj_final}
\end{align}
% This decomposition now reflects the risk of estimating the random target $\tilde{\mathbf{s}}$.
% \begin{enumerate}
%     \item Equation ~\ref{eq:term1_proj_error_final}: The error from projecting the target $\tilde{\mathbf{s}}$ onto the misspecified cone $\hat{\mathcal{M}}$. This now involves expectation over both $\hat{\mathcal{M}}$ and $\tilde{\mathbf{s}}$.
%     \item Equation ~\ref{eq:term2_stat_dim_final}: The statistical error from the zero-mean component $-\boldsymbol{\tilde{\epsilon}} $ of the effective noise.
%     \item Equation ~\ref{eq:term3_bias_proj_final}: The error from the fixed systematic bias $\boldsymbol{\nu}$ in the effective noise.
% \end{enumerate}
% Subsequent analysis would bound these terms. Term ~\ref{eq:term2_stat_dim_final} can be bounded by $2\tilde{\sigma}^2 \log(en)/n$. Term ~\ref{eq:term3_bias_proj_final} can be bounded by $2\|\boldsymbol{\nu}\|_2^2/n$. Term ~\ref{eq:term1_proj_error_final} is more complex due to the joint expectation.

We aim to bound the risk $\mathbb{E}[\|\hat{\mathbf{s}} - \tilde{\mathbf{s}}\|_2^2]$, where $\hat{\mathbf{s}} = \Pi_{\hat{\mathcal{M}}}(\mathbf{s}_p)$ is the estimator for the target $\tilde{\mathbf{s}}$. Recall the key relationships:
\begin{itemize}
    \item Target for estimation: $\tilde{\mathbf{s}} = \mathbf{s} + \boldsymbol{\tilde{\epsilon}} $, where $\boldsymbol{\tilde{\epsilon}}\sim \mathcal{N}(\mathbf{0}, \tilde{\sigma}^2 I_n)$.
    \item Observation: $\mathbf{s}_p = \mathbf{s} + \boldsymbol{\nu} = \tilde{\mathbf{s}} - \boldsymbol{\tilde{\epsilon}}+ \boldsymbol{\nu}$.
    \item Effective noise for estimating $\tilde{\mathbf{s}}$ from $\mathbf{s}_p$: $\boldsymbol{\xi} = \mathbf{s}_p - \tilde{\mathbf{s}} = \boldsymbol{\nu} - \boldsymbol{\tilde{\epsilon}}\sim \mathcal{N}(\boldsymbol{\nu}, \tilde{\sigma}^2 I_n)$.
    \item $\hat{\mathcal{M}} = \mathcal{M}_{\pi(\mathbf s^*)}$, where $\mathbf s^*$ is the output of a first-stage algorithm, estimating $\tilde {\mathbf s}$.
    \item We assume the expected ordering relevant for $\tilde{\mathbf{s}}$ is $\pi(\mathbb{E}[\tilde{\mathbf{s}}]) = \pi(\mathbf{s})$. Let $\tilde{\mathcal{M}} = \mathcal{M}_{\pi(\mathbf{\tilde s})}$. The error term arises if $\hat{\mathcal{M}} \ne \tilde{\mathcal{M}}$ and $\tilde{\mathbf{s}} \notin \hat{\mathcal{M}}$.
\end{itemize}
From Lemma~\ref{lem:adapted_bellec_prop2.1_main} (or Eq. ~\ref{eq:term1_proj_error_final}-~\ref{eq:term3_bias_proj_final} in previous discussions), the total risk is bounded by:
\begin{align}
\mathbb{E}[\|\hat{\mathbf{s}} - \tilde{\mathbf{s}}\|_2^2] \le& \underbrace{\mathbb{E}_{\hat{\mathcal{M}}, \tilde{\mathbf{s}}} \left[ \|\Pi_{\hat{\mathcal{M}}}(\tilde{\mathbf{s}}) - \tilde{\mathbf{s}}\|_2^2 \right]}_{\text{Term 1: Projection Error from Misspecified Cone}} \\
&+ \underbrace{\mathbb{E}_{\hat{\mathcal{M}}, \tilde{\mathbf{s}}} \left[ \frac{2\tilde{\sigma}^2}{n} \delta(\mathcal{T}_{\hat{\mathcal{M}}, \Pi_{\hat{\mathcal{M}}}(\tilde{\mathbf{s}})}) \right]}_{\text{Term 2: Statistical Error (Zero-Mean Noise Component)}} \\
&+ \underbrace{\mathbb{E}_{\hat{\mathcal{M}}, \tilde{\mathbf{s}}} \left[ \frac{2}{n} \|\Pi_{\mathcal{T}_{\hat{\mathcal{M}}, \Pi_{\hat{\mathcal{M}}}(\tilde{\mathbf{s}})}}(\boldsymbol{\nu})\|_2^2 \right]}_{\text{Term 3: Bias Error Component}}
\end{align}

We now bound each term.

\paragraph{Bounding Term 1: Projection Error from Misspecified Cone.}
\label{Appendix:th2-term1_bound}

Term 1 is $\mathbb{E}_{\hat{\mathcal{M}}, \tilde{\mathbf{s}}} \left[ \|\Pi_{\hat{\mathcal{M}}}(\tilde{\mathbf{s}}) - \tilde{\mathbf{s}}\|_2^2 \right]$. This term captures the error due to projecting the target $\tilde{\mathbf{s}}$ onto the cone $\hat{\mathcal{M}}$, which is estimated from $\mathbf s^*$ and may differ from the true cone associated with $\tilde{\mathbf{s}}$ (or its mean $\mathbf{s}$).

Let $\tilde \pi = \pi(\tilde {\mathbf s})$ be the true ranking corresponding to the target of the first-stage algorithm. Let $\mathcal{\tilde M} = \mathcal{M}_{\tilde \pi}$. The event of an incorrect cone $\hat{\mathcal{M}}$ (relative to the first stage's own target) is $A^c \equiv \{\hat{\pi} \ne \tilde \pi\}$, where $\hat{\pi} = \pi(\mathbf s^*)$.
\begin{align}
    \mathbb{E}_{\hat{\mathcal{M}}, \tilde{\mathbf{s}}} \left[ \|\Pi_{\hat{\mathcal{M}}}(\tilde{\mathbf{s}}) - \tilde{\mathbf{s}}\|_2^2 \right] = \mathbb{E}_{\tilde{\mathbf{s}}} \left[ \mathbb{E}_{\hat{\mathcal{M}} | \tilde{\mathbf{s}}} \left[ \|\Pi_{\hat{\mathcal{M}}}(\tilde{\mathbf{s}}) - \tilde{\mathbf{s}}\|_2^2 \right] \right] 
\end{align}
If $\tilde{\mathbf{s}} \in \hat{\mathcal{M}}$ (i.e., the random target happens to satisfy the random cone's constraints), then $\|\Pi_{\hat{\mathcal{M}}}(\tilde{\mathbf{s}}) - \tilde{\mathbf{s}}\|_2^2 = 0$. This occurs if $\pi(\mathbf s^*)$ is compatible with $\tilde{\mathbf{s}}$.
We follow the previous decomposition:
\begin{align}
    \mathbb{E}_{\hat{\mathcal{M}} | \tilde{\mathbf{s}}} \left[ \|\Pi_{\hat{\mathcal{M}}}(\tilde{\mathbf{s}}) - \tilde{\mathbf{s}}\|_2^2 \right] = \mathbb{E}_{\hat{\mathcal{M}} | \tilde{\mathbf{s}}} \left[ \|\Pi_{\hat{\mathcal{M}}}(\tilde{\mathbf{s}}) - \tilde{\mathbf{s}}\|_2^2 \middle| A^c \right] P(A^c) 
\end{align}
(assuming $\|\Pi_{\mathcal{\tilde M}}(\tilde{\mathbf{s}}) - \tilde{\mathbf{s}}\|_2^2$ is negligible or zero if $\pi(\mathbf{s}) = \tilde \pi$).
The term $\mathbb{E}_{\hat{\mathcal{M}} | \tilde{\mathbf{s}}} \left[ \|\Pi_{\hat{\mathcal{M}}}(\tilde{\mathbf{s}}) - \tilde{\mathbf{s}}\|_2^2 \middle| A^c \right]$ is bounded by 
\begin{align}
    \sup_{\hat{\mathcal{M}}} \|\Pi_{\hat{\mathcal{M}}}(\tilde{\mathbf{s}}) - \tilde{\mathbf{s}}\|_2^2 \le n \text{Var}(\tilde{\mathbf{s}}),
\end{align}
where $\text{Var}(\tilde{\mathbf{s}})$ is the sample variance of the components of the specific realization of $\tilde{\mathbf{s}}$.
Taking expectation over $\tilde{\mathbf{s}}$:
\begin{align}
     \mathbb{E}_{\hat{\mathcal{M}}, \tilde{\mathbf{s}}} \left[ \|\Pi_{\hat{\mathcal{M}}}(\tilde{\mathbf{s}}) - \tilde{\mathbf{s}}\|_2^2 \right] \le \mathbb{E}_{\tilde{\mathbf{s}}} [n \text{Var}(\tilde{\mathbf{s}})] P(A^c) 
\end{align}
The term $\mathbb{E}_{\tilde{\mathbf{s}}} [n \text{Var}(\tilde{\mathbf{s}})]$ is related to the expected dispersion of $\tilde{\mathbf{s}}$. Since $\tilde{\mathbf{s}} = \mathbf{s} + \boldsymbol{\tilde{\epsilon}} $, $\mathbb{E}[\tilde{\mathbf s}_i] = \mathbf s_i$.
\begin{align}
    \mathbb{E}_{\tilde{\mathbf{s}}} \left[ \sum_{i=1}^n (\tilde{\mathbf s}_i - \bar{\tilde{\mathbf s}})^2 \right] = \mathbb{E}_{\tilde{\mathbf{s}}} \left[ \sum_{i = 1}^n \tilde{\mathbf s}_i^2 - n \bar{\tilde{\mathbf s}}^2 \right]
\end{align}
This can be further analyzed, but for a simpler bound, we can use:
\begin{align}
    \mathbb{E}_{\tilde{\mathbf{s}}} [n \text{Var}(\tilde{\mathbf{s}})] \le \mathbb{E}_{\tilde{\mathbf{s}}} [\|\tilde{\mathbf{s}}\|_2^2] = \|\mathbf{s}\|_2^2 + n\tilde{\sigma}^2.
\end{align}
For a potentially looser but simpler worst-case, if the range of $\tilde{\mathbf s}_i$ is bounded, $n \text{Var}(\tilde{\mathbf{s}})$ can be bounded by a term related to the range, e.g., $n (R_{\tilde{\mathbf s}})^2/4$ if $\tilde{\mathbf s}_i \in [a,b]$ and $R_{\tilde{s}} = b-a$.
% Let's use the bound $M_{\tilde{\mathbf{s}}}^2 = \sup_{\tilde{\mathbf{s}}, \hat{\mathcal{M}}} \|\Pi_{\hat{\mathcal{M}}}(\tilde{\mathbf{s}}) - \tilde{\mathbf{s}}\|_2^2$.
A bound that depends on the properties of the ground truth $\mathbf{s}$ and noise $\tilde{\sigma}^2$ is needed here. For now, we denote it as $\mathbb{E}[n \text{Var}(\tilde{\mathbf{s}})]$.

The probability of a ranking error $P(A^c) = P(\pi(\mathbf s^*) \ne \pi(\tilde {\mathbf s}))$ is:
\begin{align}
    P(A^c) \le \sum_{j,k: \; \tilde{\mathbf s}_j < \tilde{\mathbf s}_k} P(\mathbf{s}^*_j > \mathbf{s}^*_k) 
\end{align}
Let $\eta = \mathbf s^* - \tilde {\mathbf s} \sim \mathcal{N}(0, \Sigma_{\mathbf s^*})$, where $\Sigma_{\mathbf s^*} = \frac{1}{mk} S^+$.
Let $\tilde{\Delta}_{kj} = \tilde{\mathbf s}_k - \tilde{\mathbf s}_j$,
and ${\sigma^*_{X_{jk}}}^2 = (\boldsymbol{e}_j - \boldsymbol{e}_k)^\top \Sigma_{\mathbf s^*} (\boldsymbol{e}_j - \boldsymbol{e}_k)$.
Using the Chernoff-Cramer bound:
\begin{align}
     P(\mathbf{s}^*_j > \mathbf{s}^*_k) \le \frac{\sigma^*_{X_{jk}}}{(\tilde{\Delta}_{kj})\sqrt{2\pi}} \exp\left( - \frac{(\tilde{\Delta}_{kj})^2}{2 {\sigma^*_{X_{jk}}}^2} \right) 
\end{align}
This completes the proof of Corollary ~\ref{cor:ranking_misspec_prob}. Thus,
\begin{align}
    \text{Term 1} \le \mathbb{E}_{\tilde{\mathbf{s}}} [n \text{Var}(\tilde{\mathbf{s}})] \sum_{j,k: \tilde{\mathbf s}_j < \tilde{\mathbf s}_k} \frac{\sigma^*_{X_{jk}}}{(\tilde{\Delta}_{kj})\sqrt{2\pi}} \exp\left( - \frac{(\tilde{\Delta}_{kj})^2}{2 {\sigma^*_{X_{jk}}}^2} \right) 
\end{align}
Alternatively, using the expected number of inversions $Inv(\pi^*, \tilde \pi)$ for the first stage:
\begin{align}
    \sum_{j,k: \tilde{\mathbf s}_j < \tilde{\mathbf s}_k} P(\mathbf{s}^*_j > \mathbf{s}^*_k) = \mathbb{E}[Inv(\pi^*, \tilde \pi)] 
\end{align}
Then,
\begin{align}
    \text{Term 1} \le \mathbb{E}_{\tilde{\mathbf{s}}} [n \text{Var}(\tilde{\mathbf{s}})] \cdot \mathbb{E}[Inv(\pi^*, \tilde \pi)] 
\end{align}

\paragraph{Bounding Term 2: Statistical Error.}
\label{subsec:term2_bound}

% Term 2 is $\mathbb{E}_{\hat{\mathcal{M}}, \tilde{\mathbf{s}}} \left[ \frac{2\tilde{\sigma}^2}{n} \delta(\mathcal{T}_{\hat{\mathcal{M}}, \Pi_{\hat{\mathcal{M}}}(\tilde{\mathbf{s}})}) \right]$.
% As established previously, for any isotonic cone $\hat{\mathcal{M}}$ and any point $\boldsymbol{u} \in \hat{\mathcal{M}}$, $\delta(\mathcal{T}_{\hat{\mathcal{M}}, \boldsymbol{u}}) \le \delta(\hat{\mathcal{M}}) = \delta(\mathcal{M}_n^\uparrow) \le \log(en)$.
% \begin{align}
%     \text{Term 2} \le \mathbb{E}_{\hat{\mathcal{M}}, \tilde{\mathbf{s}}} \left[ \frac{2\tilde{\sigma}^2}{n} \log(en) \right] = \frac{2\tilde{\sigma}^2 \log(en)}{n} 
% \end{align}

Term 2 from the main risk decomposition Eq. ~\ref{eq:term2_stat_dim_final} is given by:
\begin{align}
     \text{Term 2} = \mathbb{E}_{\hat{\mathcal{M}}, \tilde{\mathbf{s}}} \left[ \frac{2\tilde{\sigma}^2}{n} \delta(\mathcal{T}_{\hat{\mathcal{M}}, \Pi_{\hat{\mathcal{M}}}(\tilde{\mathbf{s}})}) \right] 
\end{align}
Here, $\hat{\mathcal{M}} = \mathcal{M}_{\pi(\mathbf s^*)}$ is the estimated isotonic cone, which is always a monotone cone . The point $\Pi_{\hat{\mathcal{M}}}(\tilde{\mathbf{s}})$ is an element of $\hat{\mathcal{M}}$.

A fundamental property of tangent cones is that for any closed convex set $K$ and any point $\boldsymbol{u} \in K$, the statistical dimension of the tangent cone $\mathcal{T}_{K, \boldsymbol{u}}$ is bounded by the statistical dimension of the set $K$ itself:
\begin{align}
    \delta(\mathcal{T}_{K, \boldsymbol{u}}) \le \delta(K) 
\end{align}
This result is standard in the analysis of such problems (see, e.g., discussions related to statistical dimension in \citep{amelunxen2014living}).
In our case, $K = \hat{\mathcal{M}}$. Since $\hat{\mathcal{M}}$ is always an isotonic cone, its statistical dimension is the same as that of the standard isotonic cone $\mathcal{M}_n^{\uparrow}$.
It is known (cf. \citep[Eq. (D.12)]{amelunxen2014living}) that the statistical dimension of the standard isotonic cone $\mathcal{M}_n^{\uparrow}$ is given by the $n$-th harmonic number $H_n$:
\begin{align}
     \delta(\mathcal{M}_n^{\uparrow}) = \sum_{k=1}^n \frac{1}{k} = H_n 
\end{align}
Therefore, for any realization of $\hat{\mathcal{M}}$ and $\tilde{\mathbf{s}}$:
\begin{align}
    \delta(\mathcal{T}_{\hat{\mathcal{M}}, \Pi_{\hat{\mathcal{M}}}(\tilde{\mathbf{s}})}) \le \delta(\hat{\mathcal{M}}) = \delta(\mathcal{M}_n^{\uparrow}) = H_n 
\end{align}
The harmonic number $H_n$ has well-known asymptotic expansions. For $n \ge 1$:
\begin{align}
     H_n = \ln(n) + \gamma_E + \frac{1}{2n} - \frac{1}{12n^2} + O\left(\frac{1}{n^4}\right) 
\end{align}
where $\gamma_E \approx 0.57721$ is the Euler-Mascheroni constant.
Thus, we can use the upper bound:
\begin{align}
    H_n \le \ln(n) + \gamma_E + \frac{1}{2n} 
\end{align}
For a simpler, more common bound in such analyses, $H_n \approx \ln(n) + \gamma_E$, or $H_n \le \ln(n) + 1$ for $n \ge 1$. The bound $\log(en) = \ln(n) + 1$ used in  \citep{bellec2018sharp} (Eq. 1.29) is a convenient and slightly looser upper bound for $H_n$.
Using $H_n \le \ln(n) + \gamma_E + O(1/n)$:
\begin{align}
     \text{Term 2} \le \mathbb{E}_{\hat{\mathcal{M}}, \tilde{\mathbf{s}}} \left[ \frac{2\tilde{\sigma}^2}{n} H_n \right] = \frac{2\tilde{\sigma}^2 H_n}{n} = \frac{2\tilde{\sigma}^2}{n} \left(\ln(n) + \gamma_E + O\left(\frac{1}{n}\right)\right) 
\end{align}
Alternatively, using the $\log(en)$ bound for consistency with some literature:
\begin{align}
    \text{Term 2} \le \frac{2\tilde{\sigma}^2 \log(en)}{n}
\end{align}
This provides a bound for the statistical error component that depends on the variance $\tilde{\sigma}^2$ of the zero-mean noise component $\boldsymbol{\tilde{\epsilon}} $ and logarithmically on $n$.

\paragraph{Bounding Term 3: Bias Error Component.}
\label{subsec:term3_bound}

% Term 3 is $\mathbb{E}_{\hat{\mathcal{M}}, \tilde{\mathbf{s}}} \left[ \frac{2}{n} \|\Pi_{\mathcal{T}_{\hat{\mathcal{M}}, \Pi_{\hat{\mathcal{M}}}(\tilde{\mathbf{s}})}}(\boldsymbol{\nu})\|_2^2 \right]$.
% Since projection does not increase the norm, $\|\Pi_{\mathcal{T}_{\hat{\mathcal{M}}, \Pi_{\hat{\mathcal{M}}}(\tilde{\mathbf{s}})}}(\boldsymbol{\nu})\|_2^2 \le \|\boldsymbol{\nu}\|_2^2$.
% \begin{align}
%     \text{Term 3} \le \mathbb{E}_{\hat{\mathcal{M}}, \tilde{\mathbf{s}}} \left[ \frac{2}{n} \|\boldsymbol{\nu}\|_2^2 \right] = \frac{2\|\boldsymbol{\nu}\|_2^2}{n} 
% \end{align}

Term 3 from the main risk decomposition Eq.~\ref{eq:term3_bias_proj_final} is given by:
\begin{align}
    \text{Term 3} = \mathbb{E}_{\hat{\mathcal{M}}, \tilde{\mathbf{s}}} \left[ \frac{2}{n} \|\Pi_{\mathcal{T}_{\hat{\mathcal{M}}, \Pi_{\hat{\mathcal{M}}}(\tilde{\mathbf{s}})}}(\boldsymbol{\nu})\|_2^2 \right] 
\end{align}
This term arises from the fixed systematic bias $\boldsymbol{\nu}$ in the effective noise $\boldsymbol{\xi} = \boldsymbol{\nu} - \boldsymbol{\tilde{\epsilon}} $. The expectation $\mathbb{E}_{\hat{\mathcal{M}}, \tilde{\mathbf{s}}}$ averages over the randomness of the estimated cone $\hat{\mathcal{M}}$ and the randomness of the target $\tilde{\mathbf{s}}$ , which influences the point $\Pi_{\hat{\mathcal{M}}}(\tilde{\mathbf{s}})$ at which the tangent cone is evaluated.

Let $\boldsymbol{u}_{\text{proj}} = \Pi_{\hat{\mathcal{M}}}(\tilde{\mathbf{s}})$. We need to bound $\|\Pi_{\mathcal{T}_{\hat{\mathcal{M}}, \boldsymbol{u}_{\text{proj}}}}(\boldsymbol{\nu})\|_2^2$.
A fundamental property of projection onto a closed convex set is that it does not increase the $L_2$ norm. Therefore,
\begin{align}
     \|\Pi_{\mathcal{T}_{\hat{\mathcal{M}}, \boldsymbol{u}_{\text{proj}}}}(\boldsymbol{\nu})\|_2^2 \le \|\boldsymbol{\nu}\|_2^2
\end{align}
This inequality holds for any realization of $\hat{\mathcal{M}}$ and $\tilde{\mathbf{s}}$.
Substituting this into the expression for Term 3:
\begin{align}
     \text{Term 3} \le \mathbb{E}_{\hat{\mathcal{M}}, \tilde{\mathbf{s}}} \left[ \frac{2}{n} \|\boldsymbol{\nu}\|_2^2 \right] 
\end{align}
As $\boldsymbol{\nu}$ is a fixed vector, $\|\boldsymbol{\nu}\|_2^2$ is a constant. Then the expectation operator does not change it:
\begin{align}
    \text{Term 3} \le \frac{2\|\boldsymbol{\nu}\|_2^2}{n}
\end{align}
This provides an upper bound for the bias error component, indicating that it scales with the squared norm of the fixed bias vector $\boldsymbol{\nu}$ and decreases with $1/n$. The factor of $2$ arises from the use of the inequality $(a+b)^2 \le 2a^2+2b^2$ when separating the zero-mean noise and bias components in the derivation of Lemma~\ref{lem:adapted_bellec_prop2.1_main}.

\paragraph{Combined Risk Bound.}
Combining the bounds for the three terms:
\begin{align*}
\mathbb{E}[\|\hat{\mathbf{s}} - \tilde{\mathbf{s}}\|_2^2] \le&  \mathbb{E}_{\tilde{\mathbf{s}}} [n \text{Var}(\tilde{\mathbf{s}})]\cdot \mathbb{E}[Inv(\pi^*, \tilde \pi)] \\
&+ \frac{2\tilde{\sigma}^2 \log(en)}{n} \\
&+ \frac{2\|\boldsymbol{\nu}\|_2^2}{n}
\end{align*}
where $\mathbb{E}[Inv(\pi^*, \tilde \pi)]$ is the expected number of inversions produced by the first-stage algorithm $\mathbf s^*$ with respect to its true target $\tilde {\mathbf s}$, and can be expressed as:
$$ \mathbb{E}[Inv(\pi^*, \tilde \pi)] = \sum_{j,k: \tilde{\mathbf s}_j < \tilde{\mathbf s}_k} P(\mathbf{s}^*_j > \mathbf{s}^*_k) $$
with each $P(\mathbf{s}^*_j > \mathbf{s}^*_k)$ bounded using the Chernoff-Cramer bound:
$$ P(\mathbf{s}^*_j > \mathbf{s}^*_k) \le \frac{\sqrt{(\boldsymbol{e}_j - \boldsymbol{e}_k)^\top \Sigma_{\mathbf s^*} (\boldsymbol{e}_j - \boldsymbol{e}_k)}}{(\tilde{\mathbf s}_k - \tilde{\mathbf s}_j)\sqrt{2\pi}} \exp\left( - \frac{(\tilde{\mathbf s}_k - \tilde{\mathbf s}_j)^2}{2 (\boldsymbol{e}_j - \boldsymbol{e}_k)^\top \Sigma_{\mathbf s^*} (\boldsymbol{e}_j - \boldsymbol{e}_k)} \right) $$
and $\Sigma_{\mathbf s^*}$ is the asymptotic covariance matrix of $\mathbf s^*$ in Lemma ~\ref{lem:hetero-cov}.
% \end{proof}

\textbf{Additional discussion.} The core distinctions remain similar to previous discussions, with the key aspect being that our target for estimation $\tilde{\mathbf{s}}$ is a noisy version of the underlying ground truth $\mathbf{s}$, our observations $\mathbf{s}_p$ are related to $\tilde{\mathbf{s}}$ via an effective noise term $\mathcal{N}(\boldsymbol{\nu}, \tilde{\sigma}^2 I_n)$, and the projection is onto a random cone $\hat{\mathcal{M}}$.
(1) Differs from traditional isotonic regression which assumes a fixed cone and zero-mean noise relative to its direct target.
(2) Differs from fixed model misspecification as our cone is random and our effective observation noise has a fixed bias $\boldsymbol{\nu}$ and variance $\tilde{\sigma}^2 I_n$ relative to the target $\tilde{\mathbf{s}}$.

\section{Probability of Estimation Improvement via Projection}
\label{sec:improvement_probability}

In this section, we establish a probabilistic bound demonstrating that, under certain conditions related to the accuracy of the estimated cone $\hat{\mathcal{M}}$, the projection-based estimator $\hat{\mathbf{s}}$ provides a strictly better estimate of the ground truth $\mathbf{s}$ than the direct observation $\mathbf{s}_p$.

\subsection{Proof of Corollary \ref{cor:ranking_misspec_prob}}

\begin{lemma}[Restatement of Corollary \ref{cor:ranking_misspec_prob}]
\label{lem:delta1_appendix}
Let $\mathbf s^*$ be the estimate from the first-stage algorithm for its true target $\tilde {\mathbf s}$. For any two items $j,k$ with $\tilde{\mathbf s}_j < \tilde{\mathbf s}_k$, let $\tilde{\Delta}_{kj}=\tilde{\mathbf s}_k - \tilde{\mathbf s}_j$ be the true score gap for the first-stage target, and let ${\sigma^*_{X_{jk}}}^2 = (\boldsymbol{e}_j - \boldsymbol{e}_k)^T \Sigma_{\mathbf s^*} (\boldsymbol{e}_j - \boldsymbol{e}_k)$ be the variance of the estimated score difference from the first stage, where $\Sigma_{\mathbf s^*}$ is the asymptotic covariance of $\mathbf s^* - \tilde {\mathbf s}$. Then the probability that the estimated cone $\hat{\mathcal{M}} = \mathcal{M}_{\pi(\mathbf s^*)}$ differs from the true cone for the first-stage target $\mathcal{\tilde M} = \mathcal{M}_{\pi(\tilde {\mathbf s})}$ is bounded by
\begin{align}
\delta_1 := P(\pi(\mathbf s^*) \neq \pi(\tilde {\mathbf s})) \;\le\; \sum_{j,k:\,\tilde{\mathbf s}_j < \tilde{\mathbf s}_k} \frac{\sigma^*_{X_{jk}}}{\tilde{\Delta}_{kj}\sqrt{2\pi}}\; \exp\!\left(-\,\frac{(\tilde{\Delta}_{kj})^2}{2\,{\sigma^*_{X_{jk}}}^2}\right)\,.
\end{align}
\end{lemma}
\begin{proof}
The event $\pi(\mathbf s^*) \neq \pi(\tilde {\mathbf s})$ occurs if and only if there exists at least one pair of items $(j,k)$ such that their true order according to $\tilde {\mathbf s}$ is $\tilde{\mathbf s}_j < \tilde{\mathbf s}_k$, but their estimated order is $\mathbf{s}^*_j > \mathbf{s}^*_k$.
Using the union bound over all such pairs:
$$ P(\pi(\mathbf s^*) \neq \pi(\tilde {\mathbf s})) \le \sum_{j,k:\,\tilde{\mathbf s}_j < \tilde{\mathbf s}_k} P(\mathbf{s}^*_j > \mathbf{s}^*_k) $$
Let $\eta = \mathbf s^* - \tilde {\mathbf s}$. We assume $\eta \sim \mathcal{N}(\mathbf{0}, \Sigma_{\mathbf s^*})$.
The probability of a single incorrect pairwise ranking is $P(\mathbf{s}^*_j > \mathbf{s}^*_k) = P(\eta_{j} - \eta_{k} > \tilde{\mathbf s}_k - \tilde{\mathbf s}_j)$.
Let $X_{jk} = \eta_{j} - \eta_{k}$. Then $X_{jk} \sim \mathcal{N}(0, {\sigma^*_{X_{jk}}}^2)$.
The probability becomes $P(X_{jk} > \tilde{\Delta}_{kj})$.
Applying the Chernoff-Cramer bound (see Appendix~\ref{Appendix:th2-term1_bound} or a similar result), $P(Z > t) \le \frac{1}{t\sqrt{2\pi}}e^{-t^2/2}$ for $Z \sim \mathcal{N}(0,1)$ and $t>0$.
Let $t_{jk} = \tilde{\Delta}_{kj} / \sigma^*_{X_{jk}}$.
$$ P(X_{jk} > \tilde{\Delta}_{kj}) \le \frac{\sigma^*_{X_{jk}}}{\tilde{\Delta}_{kj}\sqrt{2\pi}}\; \exp\!\left(-\,\frac{(\tilde{\Delta}_{kj})^2}{2\,{\sigma^*_{X_{jk}}}^2}\right) $$
Summing over all relevant pairs yields the stated bound for $\delta_1$.
\end{proof}

\subsection{Proof of Lemma ~\ref{lem:delta2}}
\label{subsec:proof_lemma_delta2}

\begin{lemma}[Restatement of Lemma \ref{lem:delta2}]
\label{lem:delta2_appendix}
Let $\mathbf{s} \in \mathbb{R}^n$ be the ground truth score and $\tilde{\mathbf{s}} = \mathbf{s} + \boldsymbol{\tilde{\epsilon}}$, where $\boldsymbol{\tilde{\epsilon}} \sim \mathcal{N}(\mathbf{0}, \tilde{\sigma}^2 I_n)$. Let $B^c$ be the event $B^c := \{\pi(\tilde{\mathbf{s}}) \neq \pi(\mathbf{s})\}$. Then,
\begin{equation}
\delta_2 := P(B^{\mathsf{c}}) \;\le\; \sum_{j,k:\,s_k>s_j} \frac{\tilde{\sigma}}{(\Delta_{kj}(\mathbf{s}))\sqrt{\pi}} \exp\!\left(-\tfrac{(\Delta_{kj}(\mathbf{s}))^{2}}{4\tilde{\sigma}^{2}}\right)
\end{equation}
where $\Delta_{kj}(\mathbf{s}) = \mathbf s_k - \mathbf s_j$.
\end{lemma}

\begin{proof}
The event $B^c \equiv \{\pi(\tilde{\mathbf{s}}) \neq \pi(\mathbf{s})\}$ occurs if there exists at least one pair of items $(j,k)$ such that their true order according to $\mathbf{s}$ is (without loss of generality) $\mathbf s_j < \mathbf s_k$, but their order according to $\tilde{\mathbf{s}}$ is $\tilde{\mathbf s}_j > \tilde{\mathbf s}_k$.
Using the union bound over all such pairs $(j,k)$ where $\mathbf s_j < \mathbf s_k$:
$$ P(B^c) \le \sum_{j,k:\,\mathbf s_j < \mathbf s_k} P(\tilde{\mathbf s}_j > \tilde{\mathbf s}_k) $$
Consider a single pair $(j,k)$ such that $\mathbf s_j < \mathbf s_k$.
The event $\tilde{\mathbf s}_j > \tilde{\mathbf s}_k$ is equivalent to $(\mathbf s_j + \tilde{\epsilon}_{j}) > (\mathbf s_k + \tilde{\epsilon}_{k})$, which simplifies to:
$$ \tilde{\epsilon}_{j} - \tilde{\epsilon}_{k} > \mathbf s_k - \mathbf s_j $$
Let $Y_{jk} = \tilde{\epsilon}_{j} - \tilde{\epsilon}_{k}$. Since $\tilde{\epsilon}_{j}$ and $\tilde{\epsilon}_{k}$ are independent and identically distributed as $\mathcal{N}(0, \tilde{\sigma}^2)$, their difference $Y_{jk}$ follows a normal distribution:
$$ Y_{jk} \sim \mathcal{N}(0, \mathrm{Var}(\tilde{\epsilon}_{j}) + \mathrm{Var}(\tilde{\epsilon}_{k})) = \mathcal{N}(0, 2\tilde{\sigma}^2) $$
Let $\Delta_{kj}(\mathbf{s}) = \mathbf s_k - \mathbf s_j$. Since $\mathbf s_j < \mathbf s_k$, we have $\Delta_{kj}(\mathbf{s}) > 0$.
The probability of interest for a single pair is $P(Y_{jk} > \Delta_{kj}(\mathbf{s}))$.
Let $Z = Y_{jk} / \sqrt{2\tilde{\sigma}^2}$. Then $Z \sim \mathcal{N}(0,1)$.
$$ P(Y_{jk} > \Delta_{kj}(\mathbf{s})) = P\left(Z > \frac{\Delta_{kj}(\mathbf{s})}{\sqrt{2\tilde{\sigma}^2}}\right) $$
Applying the Chernoff-Cramer bound $P(Z > t) \le \frac{1}{t\sqrt{2\pi}}e^{-t^2/2}$ with $t = \frac{\Delta_{kj}(\mathbf{s})}{\sqrt{2\tilde{\sigma}^2}}$:
\begin{align}
P(Y_{jk} > \Delta_{kj}(\mathbf{s})) &\le \frac{\sqrt{2\tilde{\sigma}^2}}{\Delta_{kj}(\mathbf{s})\sqrt{2\pi}} \exp\left( - \frac{(\Delta_{kj}(\mathbf{s}))^2}{2 \cdot (2\tilde{\sigma}^2)} \right) \\
&= \frac{\tilde{\sigma}}{\Delta_{kj}(\mathbf{s})\sqrt{\pi}} \exp\left( - \frac{(\Delta_{kj}(\mathbf{s}))^2}{4\tilde{\sigma}^2} \right)
\end{align}
Summing this probability over all pairs $(j,k)$ such that $\mathbf s_j < \mathbf s_k$ (which is equivalent to summing over pairs $j,k$ such that $\mathbf s_k > \mathbf s_j$ by swapping indices if necessary, ensuring $\Delta_{kj}(\mathbf{s})$ remains positive) gives the stated bound for $\delta_2$.
\end{proof}

\subsection{Proof of Theorem ~\ref{thm:optimality}}
\label{sec:proof_thm_optimality}

\begin{theorem}[Restatement of Theorem~\ref{thm:optimality}]
\label{thm:optimality_appendix}
Let $\delta_1 = P(\pi(\mathbf s^*) \neq \pi(\tilde {\mathbf s}))$ be the probability of a ranking error in the first-stage estimation relative to its true target $\tilde {\mathbf s}$ (as defined in Corollary~\ref{cor:ranking_misspec_prob}).
Let $\delta_2 = P(\pi(\tilde{\mathbf{s}}) \neq \pi(\mathbf{s}))$ be the probability that the ranking of the subjective optimal point $\tilde{\mathbf{s}}$ differs from the ranking of the ground truth $\mathbf{s}$ (as defined in Lemma~\ref{lem:delta2_appendix}).
Assume that the true ranking for the first-stage target is consistent with the ground truth ranking, i.e., $\pi(\tilde {\mathbf s}) = \pi(\mathbf{s})$.
Then, with probability at least $1-\delta_1-\delta_2$, 
\begin{align}
\label{eq:optimality_ineq}
\|\hat{\mathbf{s}} - \mathbf{s}\|_2^{2} < \|\mathbf{s}_p - \mathbf{s}\|_2^{2},
\end{align}
provided that $\mathbf{s}_p \notin \mathcal{M}_{\pi(\mathbf{s})}$.
\end{theorem}

\begin{proof}

Consider the following events:
\begin{itemize}
    \item Let $A$ be the event that the estimated cone $\hat{\mathcal{M}}$ correctly identifies the cone $\mathcal{M}_{\tilde{\mathbf{s}}}$ derived from the subjective optimal point $\tilde{\mathbf{s}}$, i.e., $A := \{\hat{\mathcal{M}} = \mathcal{M}_{\tilde{\mathbf{s}}}\}$.
    \item Let $B$ be the event that the cone $\mathcal{M}_{\tilde{\mathbf{s}}}$ correctly identifies the true cone $\mathcal{M}_{\mathbf{s}}$ derived from the ground truth $\mathbf{s}$, i.e., $B := \{\mathcal{M}_{\tilde{\mathbf{s}}} = \mathcal{M}_{\mathbf{s}}\}$.
\end{itemize}
The probability $P(A^c) = P(\hat{\mathcal{M}} \neq \mathcal{M}_{\tilde{\mathbf{s}}}) = P(\pi(\mathbf s^*) \neq \pi(\tilde{\mathbf{s}})) = 1 - \delta_1$ Thus, $P(A) = 1-\delta_1$.
And $P(B^c) = P(\mathcal{M}_{\pi(\tilde{\mathbf{s}})} \neq \mathcal{M}_{\pi(\mathbf{s})}) = P(\pi(\tilde{\mathbf{s}}) \neq \pi(\mathbf{s})) = \delta_2$, as defined in Lemma~\ref{lem:delta2_appendix}. Thus,$P(B) = 1-\delta_2.$
Now we let $C$ be the event that $\hat{\mathcal{M}} = \mathcal{M}_{\mathbf{s}}$. Event $C$ occurs with probability at least $1 - \delta_1 - \delta_2$.

When event $C$ occurs, since $\mathbf{s}$ defines the cone $\mathcal{M}_{\mathbf{s}}$, it follows that $\mathbf{s} \in \mathcal{M}_{\mathbf{s}}$.Thus, 
\begin{align}
    \mathbf{s} \in \hat{\mathcal{M}}.
\end{align}
By the fundamental property of projection onto a closed convex set $\hat{\mathcal{M}}$, for any vector $\boldsymbol{x} \in \hat{\mathcal{M}}$, we have:
\begin{align}
    \|\mathbf{s}_p - \boldsymbol{x}\|_2^2 \ge \|\mathbf{s}_p - \hat{\mathbf{s}}\|_2^2 + \|\boldsymbol{x} - \hat{\mathbf{s}}\|_2^2 
\end{align}
Then we can set $\boldsymbol{x} = \mathbf{s}$:
\begin{align}
     \|\mathbf{s}_p - \mathbf{s}\|_2^2 \ge \|\mathbf{s}_p - \hat{\mathbf{s}}\|_2^2 + \|\mathbf{s} - \hat{\mathbf{s}}\|_2^2
\end{align}
This inequality holds when event $C$ occurs.
For the strict inequality $\|\mathbf{s} - \hat{\mathbf{s}}\|_2^2 < \|\mathbf{s}_p - \mathbf{s}\|_2^2$ to hold, we require that the term $\|\mathbf{s}_p - \hat{\mathbf{s}}\|_2^2$ is strictly positive.
This occurs if and only if $\mathbf{s}_p \neq \hat{\mathbf{s}}$, which is true if $\mathbf{s}_p \notin \hat{\mathcal{M}}$.
Thus, if event $C$ occurs and it is also true that $\mathbf{s}_p \notin \mathcal{M}_{\mathbf{s}}$, then $\|\mathbf{s}_p - \hat{\mathbf{s}}\|_2^2 > 0$.
In this situation, we have:
\begin{align}
     \|\mathbf{s} - \hat{\mathbf{s}}\|_2^2 < \|\mathbf{s}_p - \mathbf{s}\|_2^2
\end{align}
This strict inequality holds when event $C$ occurs and the condition $\mathbf{s}_p \notin \hat{\mathcal{M}}$ is met. The condition $\mathbf{s}_p \notin \mathcal{M}_{\mathbf{s}}$ means that the biased observation $\mathbf{s} + \boldsymbol{\nu}$ violates the true ordering of $\mathbf{s}$, which is generally true if $\boldsymbol{\nu}$ is non-zero and not perfectly aligned with the cone $\mathcal{M}_{\mathbf{s}}$ in a way that preserves membership.
\end{proof}

\begin{remark}[Conditions for Small Error Probabilities $\delta_1$ and $\delta_2$]
The overall risk bound depends critically on the magnitudes of $\delta_1$ and $\delta_2$. These probabilities become small under favorable conditions:
\begin{itemize}
    \item \textbf{For $\delta_1$ (First-stage ranking accuracy):}
    The probability $P(\mathbf{s}^*_j > \mathbf{s}^*_k)$ for $\tilde{\mathbf s}_j < \tilde{\mathbf s}_k$ is small if the ratio $\frac{(\tilde{\Delta}_{kj})^2}{{\sigma^*_{X_{jk}}}^2}$ is large. This occurs when:
        \begin{enumerate}
            \item The true gaps $\tilde{\Delta}_{kj} = \tilde{\mathbf s}_k - \tilde{\mathbf s}_j$ between scores (for the first-stage target $\tilde {\mathbf s}$) are large, making items more distinguishable.
            \item The variance of the score difference estimates ${\sigma^*_{X_{jk}}}^2 = (\boldsymbol{e}_j - \boldsymbol{e}_k)^T \Sigma_{\mathbf s^*} (\boldsymbol{e}_j - \boldsymbol{e}_k)$ is small. This implies that the first-stage algorithm provides precise estimates of score differences. The matrix $\Sigma_{\mathbf s^*} = \frac{1}{mk} S^+$ indicates that precision increases with more users ($m$) or more comparisons per user ($k$).
        \end{enumerate}
    Thus, $\delta_1$ is small if the first-stage estimation is based on sufficient data and the items it aims to rank are well-separated.

    \item \textbf{For $\delta_2$ (Ranking consistency between $\tilde{\mathbf{s}}$ and its mean $\mathbf{s}$):}
    The probability $P(\tilde{\mathbf s}_j > \tilde{\mathbf s}_k)$ for $\mathbf s_j < \mathbf s_k$ is small if the ratio $\frac{(\Delta_{kj}(\mathbf{s}))^2}{4\tilde{\sigma}^2}$ is large. This occurs when:
        \begin{enumerate}
            \item The true gaps $\Delta_{kj}(\mathbf{s}) = \mathbf s_k - \mathbf s_j$ in the ground truth scores $\mathbf{s}$ are large.
            \item The variance $\tilde{\sigma}^2$ of the subjective noise $\boldsymbol{\tilde{\epsilon}}$ (which makes $\tilde{\mathbf{s}}$ deviate from $\mathbf{s}$) is small.
        \end{enumerate}
    Thus, $\delta_2$ is small if the ground truth scores $\mathbf{s}$ are well-separated and the process generating the subjective optimal point $\tilde{\mathbf{s}}$ from $\mathbf{s}$ has low noise.
\end{itemize}
\end{remark}

\section{Additional Information of Experiments}\label{sec:appendix-exp}

\subsection{Datasets and Settings}

\noindent \textbf{Reading Level:} This dataset contains pairwise comparisons of text documents based on reading difficulty. The dataset includes 490 documents with known reading levels, serving as ground truth difficulty scores. A total of 624 annotators provided 12728 pairwise judgments, with each comparison indicating which document is easier to read. The dataset structure includes judge information, judgment outcomes (where ``A" indicates document A is easier than document B), document identification numbers, and the corresponding reading levels for each document pair.
We construct a semi-synthetic setting by generating pairwise comparisons from an underlying Bradley-Terry/Thurstone model using the true item scores, while adding annotator-specific noise to simulate heterogeneity. This provides us with realistic but controlled pairwise preference data. We then obtain an ``oracle” objective score $\mathbf s_p$ for each item by perturbing its true score with additional noise or bias, mimicking the output of an imperfect predictive model. This noisy $\mathbf s_p$ serves as the initial model assessment for Stage~2 of \AtC. The availability of ground truth in these tasks allows us to quantitatively evaluate $\mathbf s^*$, $\mathbf s_p$, and $\hat{\mathbf s}$ under various metrics.

\noindent \offer{\textbf{Dots-activity:} This dataset~\citep{kemmer2020enhancing} contains judgments from 300 participants on 30 distinct images, yielding 8700 pairwise comparisons for estimating dot counts. We transformed this into a human-centered assessment scenario as follows: humans observe \textbf{original complete images} to provide comparative judgments, while the objective model (OpenCV's contour detection) processes \textbf{corrupted images} with varying degradation levels (global blur, localized blur/noise/whiteout). The task is to estimate dot counts in the \textbf{original images}, so ground truth remains unchanged regardless of corruption applied to model inputs. This setup simulates realistic scenarios where human judgment helps calibrate models that receive degraded inputs (e.g., low-quality sensors, occluded views), while humans access complete information through experience or alternative channels. To our knowledge, this represents the first open-source study using real human data in this field; given the scarcity of suitable datasets, this transformation provides our best validation approach on authentic data.}

\subsection{Rank Aggregation Baselines}

We compare several methods for aggregating the crowdsourced pairwise judgments (Stage-1), including both homogeneous models and heterogeneity-aware models:

\begin{itemize}[leftmargin=*]
\item \textbf{BTL (Bradley-Terry-Luce):} A classic model for pairwise comparisons that assumes each item $i$ has a latent score $\mathbf s_i$ such that the probability of $i$ beating $j$ is $\Pr(i \succ j) = \frac{\exp(\mathbf s_i)}{\exp(\mathbf s_i)+\exp(\mathbf s_j)}$. We fit the scores $s$ by maximum likelihood, given all pairwise preferences. BTL assumes all annotators are identical and noise is homogeneous (every comparison is an independent sample from the same underlying distribution).

\item \textbf{Thurstone Case V (TCV):} An alternative pairwise comparison model proposed by Thurstone, which assumes differences in item scores are normally distributed. In Case V, $\Pr(i \succ j) = \Phi(\mathbf s_i - \mathbf s_j)$ where $\Phi$ is the standard normal CDF. Like BTL, the homogeneous TCV model assumes a single common noise level for all annotators. We fit item scores by MLE under this probit-based model. BTL and TCV typically produce similar rankings; TCV can be slightly more robust to outliers in some cases (due to the normal vs. logistic noise assumption).

\item \textbf{CrowdBT:} A variant of the Bradley-Terry model that incorporates annotator-specific reliability parameters. In CrowdBT~\citep{10.1145/2433396.2433420}, each annotator $u$ has an associated consistency weight or bias parameter that influences the comparisons they provide. Intuitively, this model down-weights votes from inconsistent annotators and up-weights those from reliable annotators, yielding a better aggregate score $\mathbf s^*$. We implement CrowdBT following the approach of~\citep{10.1145/2433396.2433420}, including a ``virtual annotator"regularization strategy to stabilize the reliability estimates.

\item \textbf{CrowdTCV:} A heterogeneity-aware extension of Thurstone Case V that incorporates annotator-specific precision parameters. Each annotator $u$ is assumed to have their own precision $\gamma_u$ in the Thurstone model, such that $\Pr(u: i \succ j) = \Phi(\gamma_u (\mathbf s_i - \mathbf s_j))$. \offer{This model was introduced by ~\citep{jin2020rank}.} We refer to our implementation simply as CrowdTCV. Like CrowdBT, it learns which annotators are more consistent (higher $\gamma_u$) and which are noisier (lower $\gamma_u$), improving the quality of the aggregated $\mathbf s^*$.

\item \textbf{Heterogeneous Rank Aggregation (HRA) variants:} We adopt the heterogeneous rank aggregation framework of~\citet{jin2020rank}, which explicitly models annotator-specific noise in the Bradley-Terry and Thurstone settings. Specifically, in the heterogeneous Bradley-Terry-Luce (HBTL) variant, $\Pr(u: i \succ j) = \sigma(\gamma_u (\mathbf s_i - \mathbf s_j))$ where $\sigma$ is the logistic CDF, whereas in the heterogeneous Thurstone Case V (HTCV) variant, the logistic is replaced by the normal CDF $\Phi$. Both item scores $\mathbf s$ and annotator precisions $\{\gamma_u\}$ are jointly estimated from data via iterative optimization. These methods achieve strong aggregation performance by accounting for individual differences: an annotator with a very small $\gamma_u$ (unreliable) contributes nearly random comparisons, which the model down-weights, whereas an annotator with large $\gamma_u$ is given more influence. We use the term ``HRA'' to refer to this class of heterogeneity-aware aggregators, with HBTL and HTCV as representative implementations (corresponding to HRA-E and HRA-N in Table~\ref{tab:semi-synthetic}).
\end{itemize}

Each aggregation method produces an output consensus score vector $\mathbf s^{*}$ (determined up to an arbitrary scale, which we normalize for evaluation). The $\mathbf s^{*}$'s rank $\hat{\pi}$ serve as inputs to Stage-2 of \AtC.

\subsection{Evaluation Metrics}
\label{app:metrics}
We evaluate the quality of the predicted scores and rankings using the following metrics. Let $\hat{\mathbf s} = (\hat{\mathbf s}_1,\dots,\hat{\mathbf s}_n)$ be the scores produced by a given method and $\mathbf s = (\mathbf s_1,\dots,\mathbf s_n)$ be the ground-truth scores (when available). For distribution-based metrics, let $p$ denote the true distribution of scores (or true outcome values) and $q$ denote the distribution derived from $\hat{\mathbf s}$.

\begin{itemize}[leftmargin=*]
\item \textbf{Kendall’s Tau ($\tau$):} A rank correlation coefficient between the ordering induced by $\hat{\mathbf s}$ and the true ordering. It is defined as $\displaystyle \tau = \frac{C - D}{\binom{n}{2}},$ where $C$ is the number of concordant item pairs and $D$ is the number of discordant pairs when comparing the rankings of $\hat{\mathbf s}$ and $\mathbf s$. We have $\tau=1$ if the two rankings are identical, $\tau=0$ if the rankings are uncorrelated, and $\tau=-1$ if one ranking is the exact reverse of the other. Higher $\tau$ indicates better agreement with the ground-truth ordering.

\item \textbf{MSE:} The root-mean-squared error of the scores. We compute $\displaystyle \text{MSE} = \sqrt{\frac{1}{n}\sum_{i=1}^n (\hat{\mathbf s}_i - \mathbf s_i)^2}\,. $ This is the square root of the mean squared error (MSE), expressed in the same units as the scores. A lower L2 error indicates that the predicted scores $\hat{\mathbf s}$ are numerically closer to the true scores $\mathbf s$.

\item \textbf{Wasserstein Distance:} Also known as the Earth Mover’s Distance, this measures the distance between two distributions. Treating the set of predicted scores $\{\hat{\mathbf s}_i\}$ and true scores $\{\mathbf s_i\}$ as empirical distributions, the Wasserstein-1 distance is defined as $\displaystyle W_1(\hat{\mathbf s},\,\mathbf s) = \int_{-\infty}^{\infty} \big| F_{\hat{\mathbf s}}(x) - F_\mathbf s(x) \big|\,dx,$ where $F_{\hat{\mathbf s}}$ and $F_\mathbf s$ are the cumulative distribution functions (CDFs) of the predicted and true score distributions. In practice, we compute $W_1$ by sorting the scores and finding the area between the two empirical CDF curves. Smaller $W_1$ implies that the distribution of predicted scores is closer to that of the true scores.

\item \textbf{Kolmogorov–Smirnov (KS) Statistic:} Another measure of distributional discrepancy between $\hat{\mathbf s}$ and $\mathbf s$. The KS statistic is $\displaystyle D_{\text{KS}} = \sup_x \big| F_{\hat{\mathbf s}}(x) - F_\mathbf s(x) \big|\,,$ the maximum absolute difference between the CDF of $\hat{\mathbf s}$ and the CDF of $\mathbf s$. This represents the largest gap between the two cumulative distributions. Lower KS values (closer to 0) indicate a better alignment of the predicted score distribution with the true distribution.

\item \textbf{Kullback–Leibler (KL) Divergence:} For tasks where we compare probability distributions, we use KL divergence to measure how one distribution diverges from another. Let $p = \{p_i\}$ be the true distribution and $q = \{q_i\}$ be the predicted distribution derived from $\hat{\mathbf s}$ (e.g. by normalizing or exponentiating the $\hat{\mathbf s}$ values). The KL divergence of $q$ from $p$ is defined as $\displaystyle D_{\text{KL}}(p \parallel q) = \sum_{i=1}^n p_i \log \frac{p_i}{q_i}\,. $ We treat lower KL as better (with $D_{\text{KL}}=0$ indicating $q$ exactly equals $p$). Note that KL is an asymmetric measure and is undefined if there exists some $i$ with $p_i>0$ but $q_i=0$; in our implementation we add a small $\epsilon$ to predicted probabilities to avoid zeros.

\end{itemize}
\offer{\paragraph{Remark.} An important phenomenon in \AtC\ is that $\hat{\mathbf{s}}$ can achieve higher Kendall's $\tau$ than $\mathbf{s}^{*}$ despite both inducing the same ranking $\hat{\pi}$. The key is that $\hat{\mathbf{s}}$ often contains \textbf{ties} introduced by the Pool-Adjacent-Violators (PAVA) algorithm, which are absent in $\mathbf{s}^{*}$. When $\mathbf{s}_p$ violates the consensus ranking $\hat{\pi}$, PAVA resolves conflicts by averaging violating scores, creating ties. This improves $\tau$ by converting discordant pairs into tied pairs, which Kendall's $\tau$ penalizes less severely than discordant pairs~\citep{raties}.}

\offer{Consider an illustrative example where ground truth is $\mathbf{s} = [10, 20, 30, 40]$ with order A $<$ B $<$ C $<$ D. Stage-1 produces $\mathbf{s}^{*} = [1, 2, 5, 4]$, inducing the ranking A $<$ B $<$ D $<$ C where items C and D are inverted. The model generates $\mathbf{s}_p = [12, 22, 29, 31]$ with the correct order. PAVA detects the violation at positions C and D, then averages them to produce $\hat{\mathbf{s}} = [12, 22, 30, 30]$. Computing Kendall's $\tau$ against ground truth yields $\tau(\mathbf{s}^{*}, \mathbf{s}) = (5-1)/6 \approx 0.67$ but $\tau(\hat{\mathbf{s}}, \mathbf{s}) = (5-0)/6 \approx 0.83$. The tie at positions C and D eliminates the discordant pair, improving correlation despite identical ordinal constraints from Stage-1.}

\subsection{Estimated Score Distributions on Reading Level Dataset}
\label{sec:score_visualizations}

This section provides visualizations of the estimated item scores obtained by different methods on the Reading Level dataset. We present violin plots to illustrate the distribution of scores for each item.

\begin{figure}[H]
    \centering

    \subfigure[BTL-MLE]{
        \includegraphics[width=0.3\textwidth]{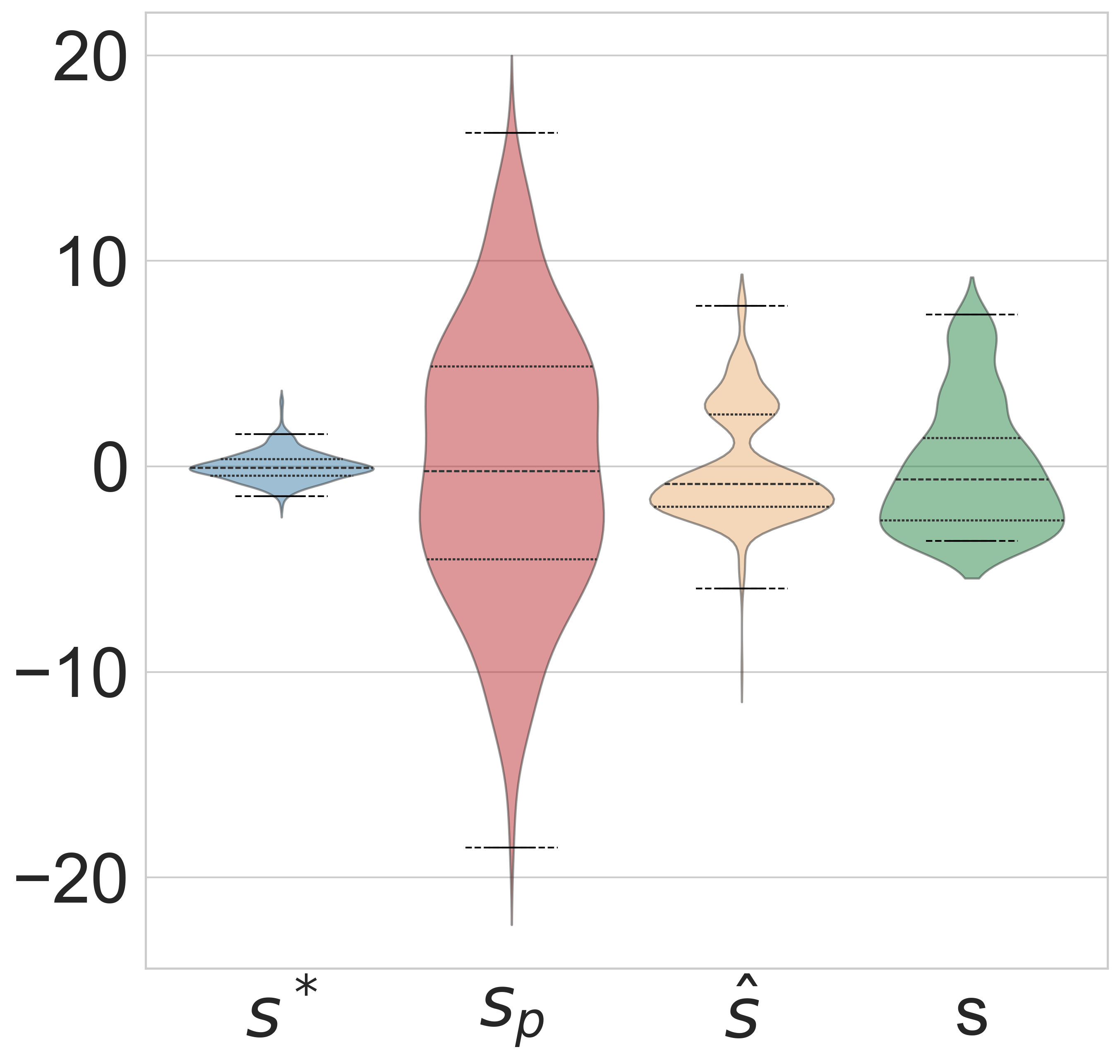}
        \label{fig:btl_mle_scores_single}
    }
    \subfigure[CrowdBT]{
        \includegraphics[width=0.3\textwidth]{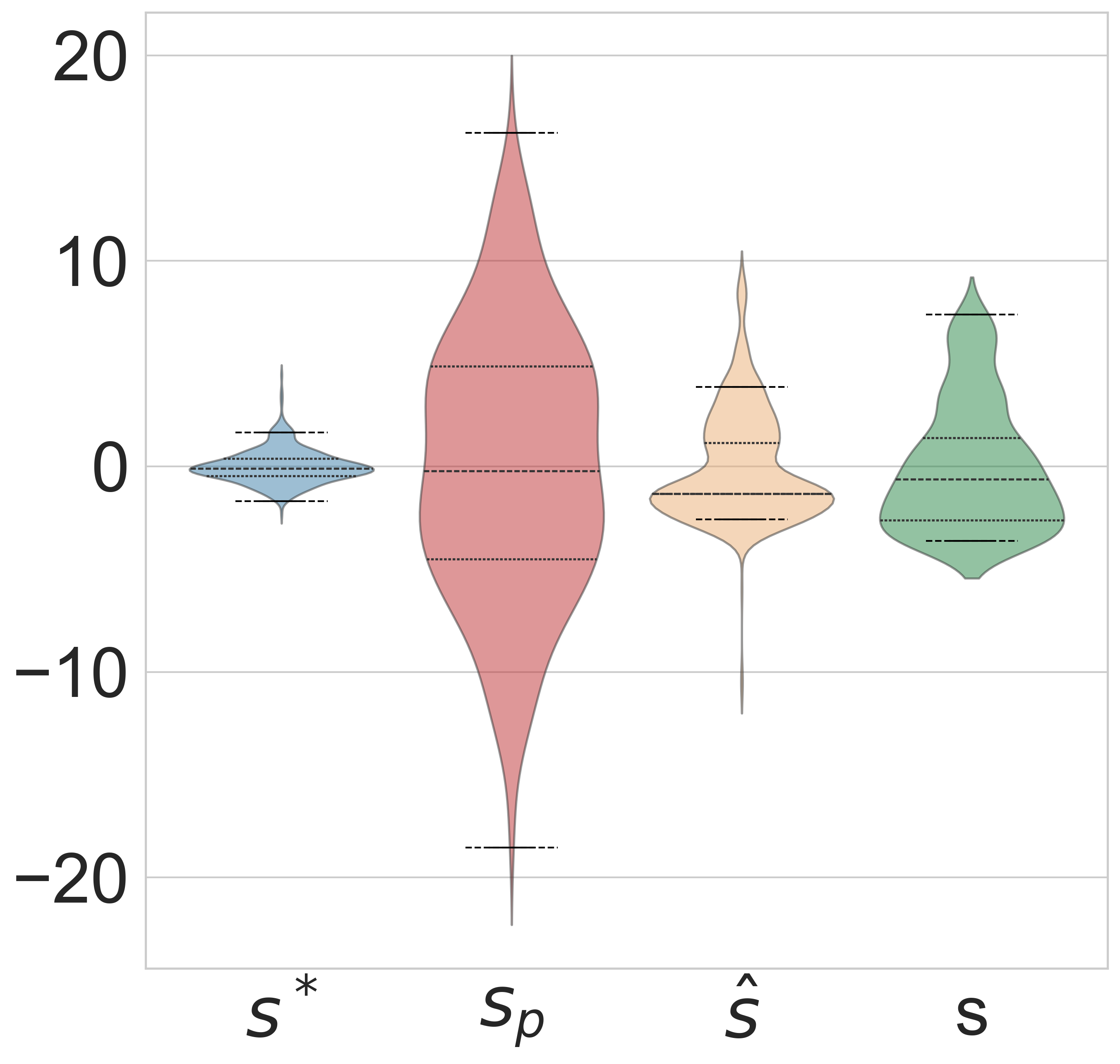}
        \label{fig:crowdbt_scores_single}
    }
    \subfigure[TCV-MLE]{
        \includegraphics[width=0.3\textwidth]{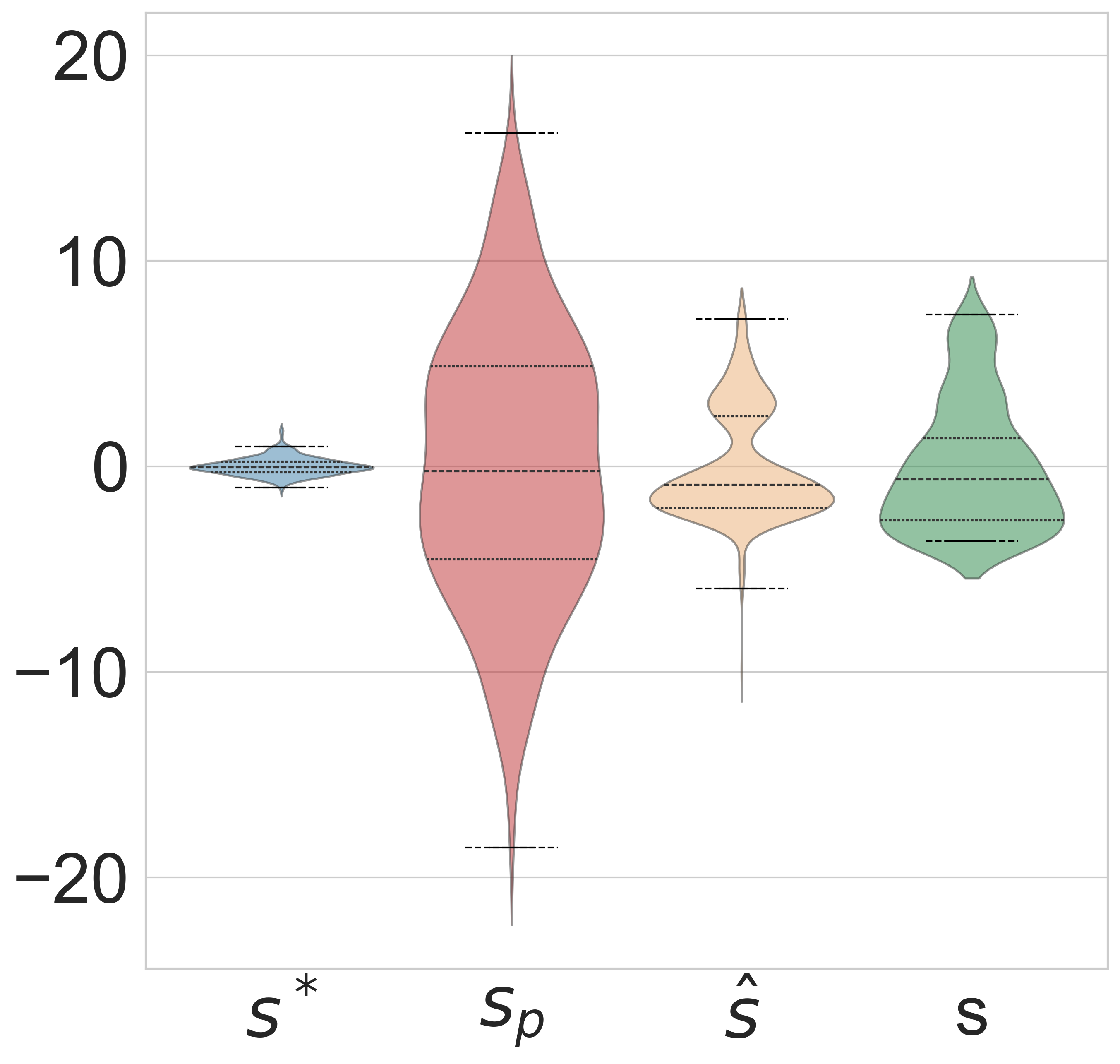}
        \label{fig:tcv_mle_scores_single}
    }

    \subfigure[CrowdTCV]{
        \includegraphics[width=0.3\textwidth]{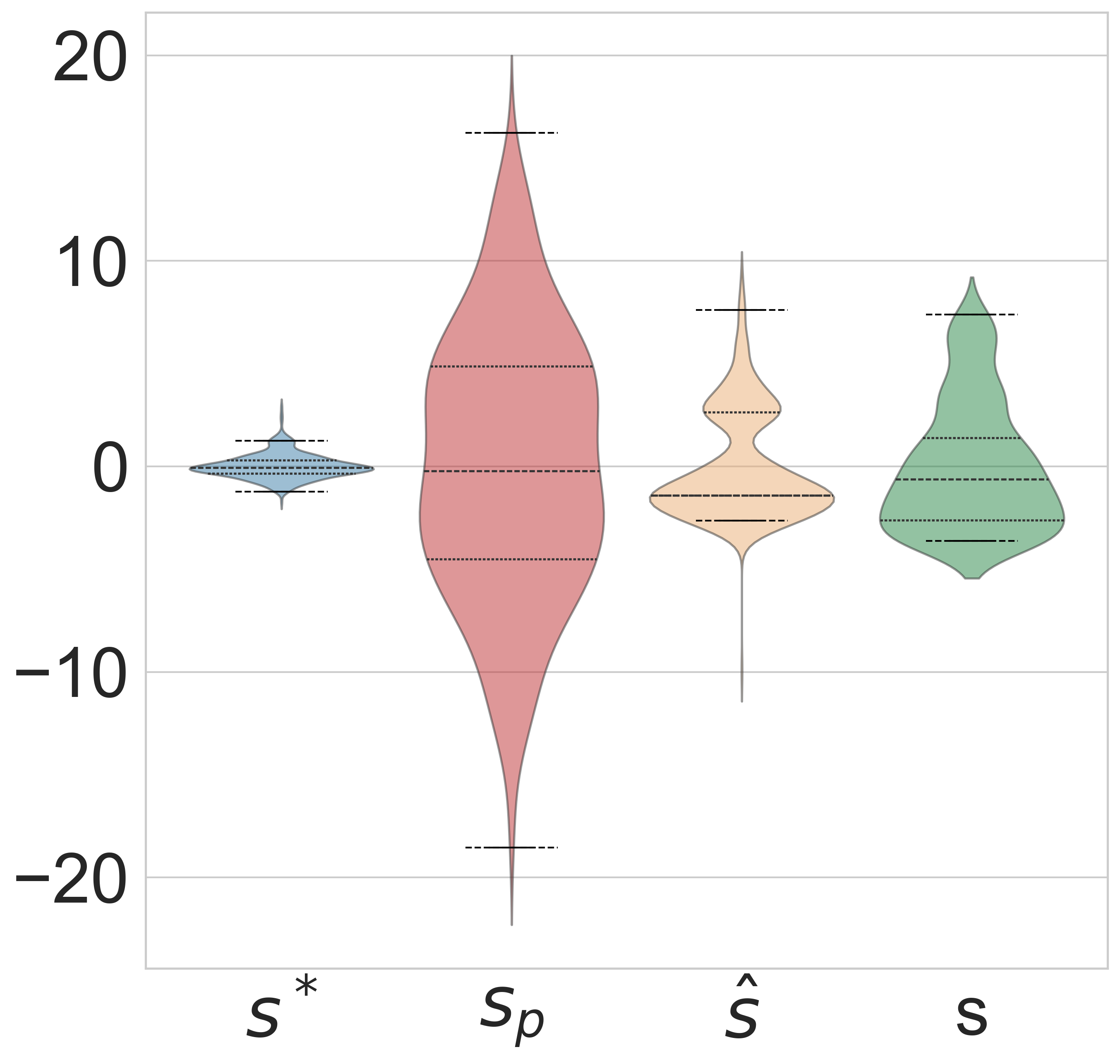}
        \label{fig:crowdtcv_scores_single}
    }
    \subfigure[HRA-N]{
        \includegraphics[width=0.3\textwidth]{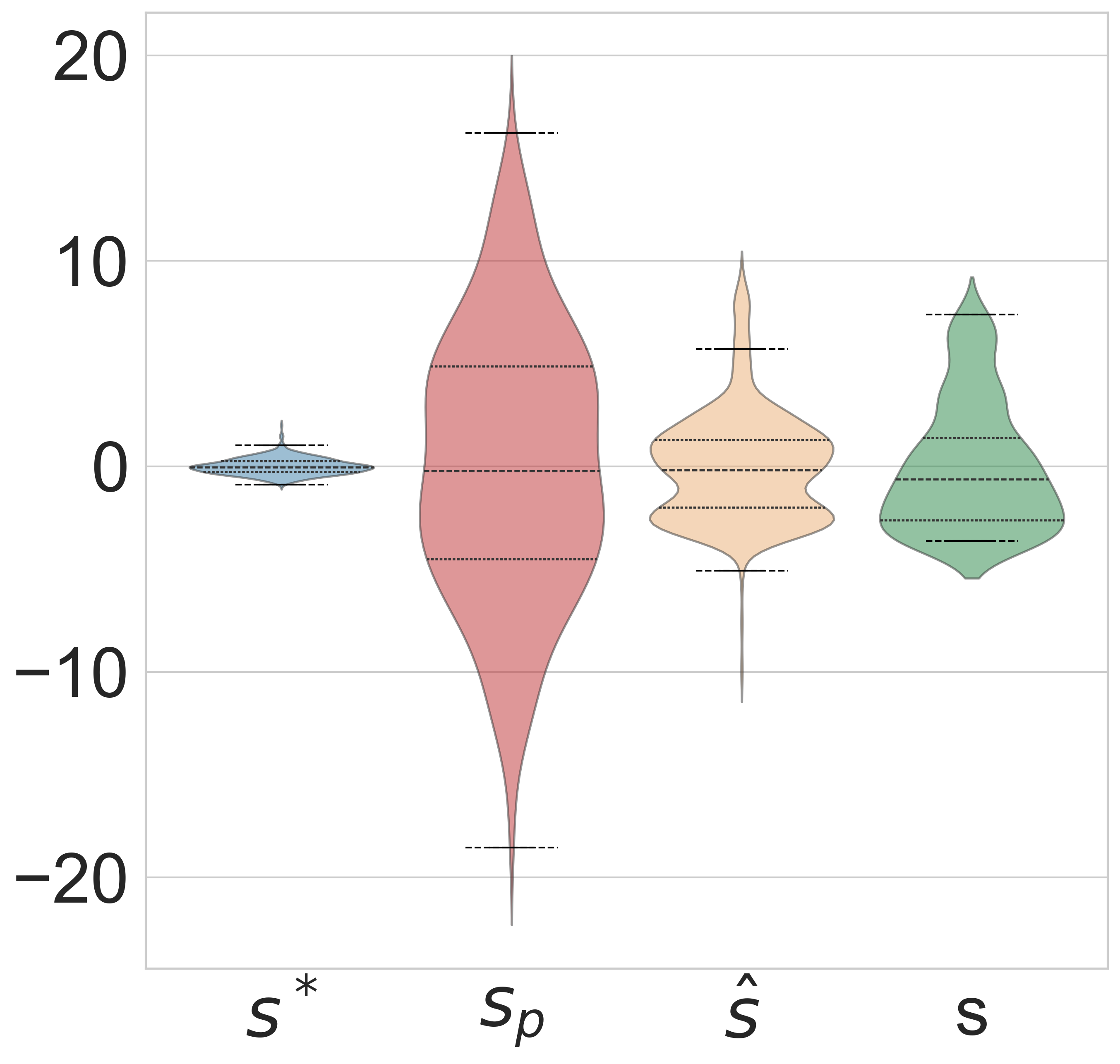}
        \label{fig:hra_e_scores_single}
    }
    \subfigure[HRA-G]{
        \includegraphics[width=0.3\textwidth]{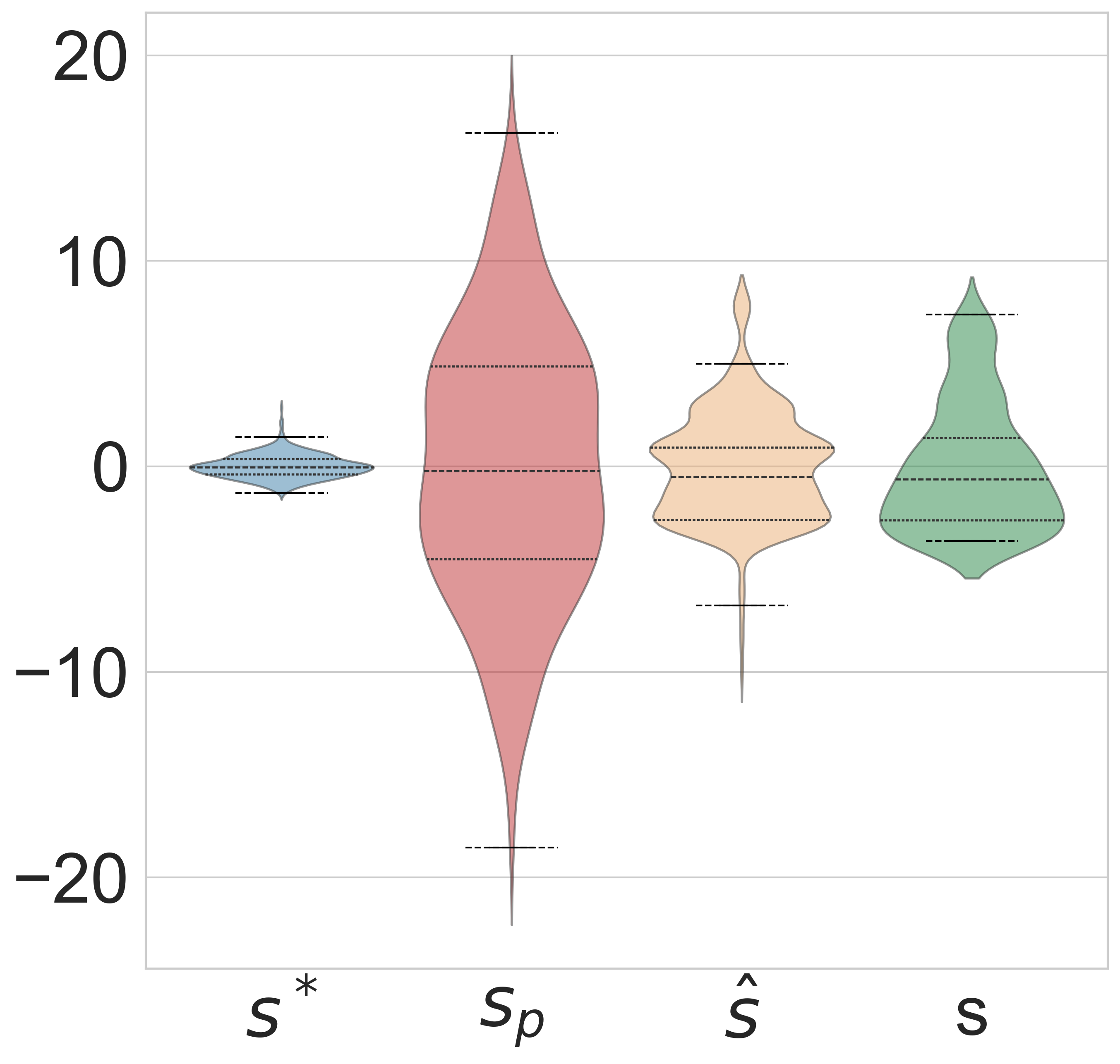}
        \label{fig:hra_g_scores_single}
    }
    
    \caption{Violin plots of estimated scores for different baseline algorithms.}
    \label{fig:violin_all_baselines}
\end{figure}

\subsection{\offer{Impact of Scale Inconsistency on Reading Level Dataset}}
\label{app:scaling-analysis}

\offer{A common concern in human judgment aggregation is whether scale inconsistency between consensus scores and ground truth inflates distributional divergence metrics. To address this, we conducted experiments scaling $\mathbf{s}^{*}$ to match the range of ground truth $\mathbf{s}$ using linear transformation: $$\mathbf{s}^{*}_{scaled} = \frac{\mathbf{s}^{*} - \min(\mathbf{s}^{*})}{\max(\mathbf{s}^{*}) - \min(\mathbf{s}^{*})} \cdot (\max(\mathbf{s}) - \min(\mathbf{s})) + \min(\mathbf{s}).$$}

\offer{Table~\ref{tab:scaling-analysis} presents results on the Reading Level dataset comparing four variants: unscaled consensus ($\mathbf{s}^{*}$), scaled consensus ($\mathbf{s}^{*}_{scaled}$), model predictions ($\mathbf{s}_p$), and \AtC\ output ($\hat{\mathbf{s}}$). The results show that $\mathbf{s}^{*}_{scaled}$ substantially improves performance over $\mathbf{s}^{*}$ (e.g., Wasserstein distance drops from 2.25 to 1.44 for HRA-G), confirming that scale inconsistency is a major issue. However, $\hat{\mathbf{s}}$ still outperforms $\mathbf{s}^{*}_{scaled}$ across all metrics and methods, demonstrating that isotonic calibration provides value beyond simple rescaling. This validates that projection onto $\widehat{\mathcal{M}}$'s surface yields superior calibration compared to linear transformations that only address scale but not ordinal alignment with model predictions.}

\begin{table}[!bht]\small
\centering
\caption{\offer{Scale Inconsistency Analysis on Reading Level Dataset.} \\
\small We compare unscaled consensus ($\mathbf{s}^{*}$), scaled consensus ($\mathbf{s}^{*}_{scaled}$), model predictions ($\mathbf{s}_p$), and \AtC\ output ($\hat{\mathbf{s}}$). \textbf{Bold} indicates the best performance among $\mathbf{s}^{*}_{scaled}$, $\mathbf{s}^{*}$, $\mathbf{s}_p$, and $\hat{\mathbf{s}}$. The \textcolor{red}{\textbf{red}} denotes optimal performance across all methods.}
\label{tab:scaling-analysis}
\begin{tabular}{p{1.5cm}cc}
\toprule
\multirow{2}{*}{\textbf{Stage-1}} 
& \textbf{Wasserstein}$\downarrow$ 
& \textbf{KS}$\downarrow$ \\
\cmidrule(lr){2-2}\cmidrule(lr){3-3}
\textbf{Method} 
& $\mathbf{s}^{*}_{scaled}/\ \mathbf{s}^{*}/\ \mathbf{s}_p/\ \hat{\mathbf{s}}$ 
& $\mathbf{s}^{*}_{scaled}/\ \mathbf{s}^{*}/\ \mathbf{s}_p/\ \hat{\mathbf{s}}$ \\
\midrule
HRA-G   & 1.44 / 2.25 / 2.83 / \textcolor{black}{\textbf{0.84}} & 0.31 / 0.50 / 0.30 / \textcolor{red}{\textbf{0.16}} \\
HRA-E   & 1.46 / 2.24 / 2.83 / \textcolor{black}{\textbf{0.83}} & 0.34 / 0.50 / 0.30 / \textcolor{red}{\textbf{0.16}} \\
HRA-N   & 1.45 / 2.35 / 2.83 / \textcolor{red}{\textbf{0.74}}   & 0.31 / 0.56 / 0.30 / \textcolor{black}{\textbf{0.19}} \\
\midrule
BTL-MLE & 1.78 / 2.19 / 2.83 / \textcolor{black}{\textbf{0.89}} & 0.42 / 0.46 / \textbf{0.30} / \textcolor{black}{\textbf{0.30}} \\
CrowdBT & 1.72 / 2.15 / 2.83 / \textcolor{black}{\textbf{0.84}} & 0.40 / 0.46 / \textbf{0.30} / \textcolor{black}{\textbf{0.27}} \\
CrowdTCV& 1.76 / 2.27 / 2.83 / \textcolor{black}{\textbf{1.02}} & 0.42 / 0.51 / 0.30 / \textcolor{black}{\textbf{0.31}} \\
TCV-MLE & 1.78 / 2.34 / 2.83 / \textcolor{black}{\textbf{0.91}} & 0.43 / 0.55 / \textbf{0.30} / \textcolor{black}{\textbf{0.30}} \\
\bottomrule
\end{tabular}
% \vspace*{-.3in}
\end{table}
\subsection{\offer{Comprehensive Baseline Comparison on Real-World Dataset}}
\label{app:full-results}

\offer{We present two evaluations on the Dots-activity dataset: (1) score-level performance analysis comparing human-only ($\mathbf{s}^{*}$), model-only ($\mathbf{s}_p$), and \AtC\ ($\hat{\mathbf{s}}$) assessments across all Stage-1 aggregation methods, and (2) end-to-end comparison against three established human-centered calibration baselines:}

\begin{itemize}[leftmargin=*]
    \item \offer{\textbf{GPPL}~\citep{10.1145/1102351.1102369}: Gaussian Process Preference Learning models pairwise preferences using GP regression with specialized covariance functions, learning latent utility functions from comparative judgments.}
    \item \offer{\textbf{Rank-SVM}~\citep{ranksvm}: Support Vector Machine adaptation for ranking problems that learns from pairwise preferences by optimizing margin-based objectives, commonly used in information retrieval and clickthrough data analysis.}
    \item \offer{\textbf{BARCW}~\citep{barcw}: Bayesian Approach for Ranking with Comparison Weights, a recent method designed to handle inconsistent preferences by modeling annotator-specific biases and reliabilities using Gaussian processes.}
\end{itemize}

\offer{Table~\ref{tab:real-world-dots-full} demonstrates that \AtC\ consistently outperforms both $\mathbf{s}_p$ and $\mathbf{s}^{*}$ across all seven aggregation methods. While the three baseline methods achieve competitive ranking accuracy (Kendall's $\tau \approx 0.92\text{--}0.94$), they exhibit catastrophic failures in cardinal calibration: MSE deteriorates by 400--450$\times$ and Wasserstein distance by 24--26$\times$ compared to \AtC. This stark contrast validates \AtC's unique strength—simultaneously achieving accurate absolute score calibration and maintaining ordinal consistency with human consensus, a capability not demonstrated by existing human-centered assessment approaches.}

\begin{table}[hbt]\small
\centering
\caption{\offer{Real-world results on Dots dataset.}\\
\small \textbf{Bold} indicates the best performance among human-only assessment ($\mathbf{s}^{*}$), model-only assessment ($\mathbf{s}_p$), and \AtC\ assessment ($\hat{\mathbf{s}}$), respectively. The \textcolor{red}{\textbf{red}} denotes the optimal performance across all methods for the given metric.
}
\label{tab:real-world-dots-full}
\begin{tabular}{p{0.3cm}p{1.15cm}p{2.55cm}p{2.1cm}p{2.7cm}p{2.55cm}}
\toprule
& \multirow{2}{*}{\textbf{Stage-1}} & \textbf{Kendall $\tau$}$\uparrow$ & \textbf{Wasserstein}$\downarrow$ & \textbf{KL}$\downarrow$ & \textbf{MSE}$\downarrow$ \\
\cmidrule(lr){3-3}\cmidrule(lr){4-4}\cmidrule(lr){5-5}\cmidrule(lr){6-6}
& \textbf{Method} & $\mathbf{s}^{*}/\ \mathbf{s}_p/\ \hat{\mathbf{s}}$ & $\mathbf{s}^{*}/\ \mathbf{s}_p/\ \hat{\mathbf{s}}$ & $\mathbf{s}^{*}/\ \mathbf{s}_p/\ \hat{\mathbf{s}}$ & $\mathbf{s}^{*}/\ \mathbf{s}_p/\ \hat{\mathbf{s}}$ \\
\midrule
\multirow{7}{*}{\textbf{AtC}}
& HRA-G & 0.917 / 0.923 / \textbf{0.940} & 6.97 / 2.53 / \textcolor{red}{\textbf{2.53}} & 123.13 / 0.881 / \textbf{0.860} & 65.08 / 11.75 / \textbf{9.61} \\
& HRA-E & 0.922 / 0.922 / \textcolor{red}{\textbf{0.943}} & 6.98 / 2.53 / \textcolor{red}{\textbf{2.53}} & 123.16 / 0.881 / \textbf{0.861} & 65.16 / 11.75 / \textcolor{red}{\textbf{9.59}} \\
& HRA-N & 0.917 / 0.923 / \textbf{0.934} & 7.05 / 2.53 / \textcolor{red}{\textbf{2.53}} & 125.50 / 0.881 / \textcolor{red}{\textbf{0.859}} & 66.44 / 11.75 / \textbf{9.75} \\
& CrowdBT & 0.894 / 0.923 / \textbf{0.927} & 5.32 / \textcolor{red}{\textbf{2.53}} / 2.55 & 72.84 / 0.881 / \textbf{0.865} & 39.28 / 11.75 / \textbf{10.24} \\
& CrowdTCV & 0.899 / 0.923 / \textbf{0.926} & 5.73 / \textcolor{red}{\textbf{2.53}} / 2.55 & 84.74 / 0.881 / \textbf{0.865} & 44.97 / 11.75 / \textbf{10.24} \\
& BTL & 0.917 / 0.923 / \textbf{0.937} & 6.72 / 2.53 / \textcolor{red}{\textbf{2.53}} & 115.37 / 0.881 / \textbf{0.860} & 60.51 / 11.75 / \textbf{9.73} \\
& TCV & 0.917 / 0.923 / \textbf{0.933} & 6.81 / \textcolor{red}{\textbf{2.53}} / 2.54 & 117.81 / 0.881 / \textbf{0.878} & 62.09 / 11.75 / \textbf{10.18} \\
\midrule
\multicolumn{2}{l}{GPPL} & -- / -- / 0.931 & -- / -- / 64.50 & -- / -- / 126.62 & -- / -- / 4220.36 \\
\multicolumn{2}{l}{Rank-SVM} & -- / -- / 0.923 & -- / -- / 61.20 & -- / -- / 133.70 & -- / -- / 3814.49 \\
\multicolumn{2}{l}{BARCW} & -- / -- / 0.940 & -- / -- / 64.62 & -- / -- / 24.09 & -- / -- / 4236.16 \\
\bottomrule
\end{tabular}
% \vspace*{-.2in}
\end{table}

\paragraph{Remark.} Our work is related to methods that integrate human and machine decision-making to improve prediction quality~\citep{zhao2025redone,Li_Li_Zhou_2025,xie-etal-2025-coalign, li2025spatiotemporalhierarchicalcausalmodels}. Beyond the assessment scenarios examined in this paper, the \texttt{AtC} framework connects to a range of applied domains where reliable evaluation requires reconciling heterogeneous human expertise with automated predictions. In logistics and operations research, platforms must continuously assess latent service quality metrics such as route difficulty, courier workload, and delivery reliability~\citep{A01lyu2025inco, A02lyu2024towards, A03lyu2023rede, A2510.1145/3746252.3761550, A1410.1145/3690624.3709425, A2110.1145/3746252.3761560, A32zhou2024cross, A33zhou2024multi}, where ground-truth measurements are expensive to obtain at scale, making principled fusion of operational data with frontline worker judgments a natural fit for the aggregate-then-calibrate paradigm. In urban and spatial computing, evaluating the effectiveness of intelligent transportation strategies~\citep{A1310.1145/3627673.3679605, A08yang2024behavior, A09yang2023carpg}, infrastructure robustness~\citep{A07fang2025cellular, A2210437846, A31ZhangCWLHD26}, and public health outcomes~\citep{A27LI2025104278, A26hong2026wed, LI2025104278} similarly involves latent targets that are only partially observable through both sensor-driven models and domain expert assessments. The growing adoption of foundation models in sensing and perception tasks~\citep{A0410.1145/3557915.3560944, A0610.1145/3659597, A20liu2025towards, A119209704, A30lu2024scaneru, A34zhou2026fullpromptingdepthdepth,A35yu2025finesat,A36wang2025wicg, li2025singlemodelsmitigatingmultimodal, 10.1145/3774904.3793059} and in large-scale graph mining systems~\citep{A15Outside-in, A16paths2pair, A17hang2024complex, A18feng2025neighsqueeze, A19feng2025hierarchical, A10zhang2020revisitinggraphconvolutionalnetwork} further amplifies the demand for systematic human-AI evaluation frameworks, as these powerful models require calibrated human oversight to ensure trustworthy deployment. More broadly, scientific disciplines such as computational biology~\citep{A23genes13040568, A28Liu2025.11.25.690439, A29unknown, 10.1145/3803851} and statistical learning under privacy constraints~\citep{A12wang2023finite} increasingly rely on combining expert curation with computational evidence, calling for principled statistical and computational methods that can aggregate diverse signals under uncertainty. These settings share the core structure that \texttt{AtC} addresses: ground truth is costly, delayed, or partially observable, and human judgments and model outputs carry complementary but imperfect information.

\subsection{Robustness Results for All Baseline Methods on Real-world Dataset.}
\label{app:appendix-radars}

This section provides supplementary results to complement the main analysis presented in Section~\ref{sec:experiments}. We present the complete set of radar plots for the six additional baseline aggregation methods: HRA-G, HRA-N, BTL, TCV, CrowdBT, and CrowdTCV. The experimental methodology, including the four types of image corruption and the evaluation metric (Kendall's Tau), is identical to the one used for the HRA-E analysis in the main body of the paper.

As illustrated in Figures~\ref{fig:radar_HRA-G} ~\ref{fig:radar_HRA-N} ~\ref{fig:radar_CrowdTCV}  ~\ref{fig:radar_BTL} ~\ref{fig:radar_TCV} ~\ref{fig:radar_CrowdBT}, a consistent trend emerges across all baselines. The primary finding from the main text is strongly reinforced: the \AtC\ framework, when calibrated using aggregated rankings ($\hat{\mathbf s}$ (Rank)), demonstrates superior robustness compared to all alternatives. In most cases, the performance of $\hat{\mathbf s}$ (Rank) remains high and degrades gracefully, while the raw objective model scores ($\mathbf s_p$) are significantly more vulnerable to input corruptions. This consistent outperformance across a diverse set of aggregation algorithms underscores the general applicability and reliability of our proposed ranking-based calibration strategy.

% \begin{figure}[htbp]
%     \centering
%     \subfigure[BLUR]{
%         \includegraphics[width=0.23\textwidth]{figure/radar_plots/radar_plot_HRA-G_GAUSSIAN_BLUR.pdf}
%     }
%     \subfigure[LOCAL\_BLUR]{
%         \includegraphics[width=0.23\textwidth]{figure/radar_plots/radar_plot_HRA-G_LOCAL_BLUR.pdf}
%     }
%     \subfigure[LOCAL\_NOISE]{
%         \includegraphics[width=0.23\textwidth]{figure/radar_plots/radar_plot_HRA-G_LOCAL_NOISE.pdf}
%     }
%     \subfigure[WHITEOUT]{
%         \includegraphics[width=0.23\textwidth]{figure/radar_plots/radar_plot_HRA-G_WHITEOUT.pdf}
%     }
%     \caption{Radar plots under different noise conditions (HRA-G).}
%     \label{fig:radar_HRA-G}
% \end{figure}

% \begin{figure}[htbp]
%     \centering
%     \subfigure[BLUR]{
%         \includegraphics[width=0.23\textwidth]{figure/radar_plots/radar_plot_HRA-N_GAUSSIAN_BLUR.pdf}
%     }
%     \subfigure[LOCAL\_BLUR]{
%         \includegraphics[width=0.23\textwidth]{figure/radar_plots/radar_plot_HRA-N_LOCAL_BLUR.pdf}
%     }
%     \subfigure[LOCAL\_NOISE]{
%         \includegraphics[width=0.23\textwidth]{figure/radar_plots/radar_plot_HRA-N_LOCAL_NOISE.pdf}
%     }
%     \subfigure[WHITEOUT]{
%         \includegraphics[width=0.23\textwidth]{figure/radar_plots/radar_plot_HRA-N_WHITEOUT.pdf}
%     }
%     \caption{Radar plots under different noise conditions (HRA-N).}
%     \label{fig:radar_HRA-N}
% \end{figure}

\begin{figure}[htbp]
    \centering
    \subfigure[BLUR]{
        \includegraphics[width=0.23\textwidth]{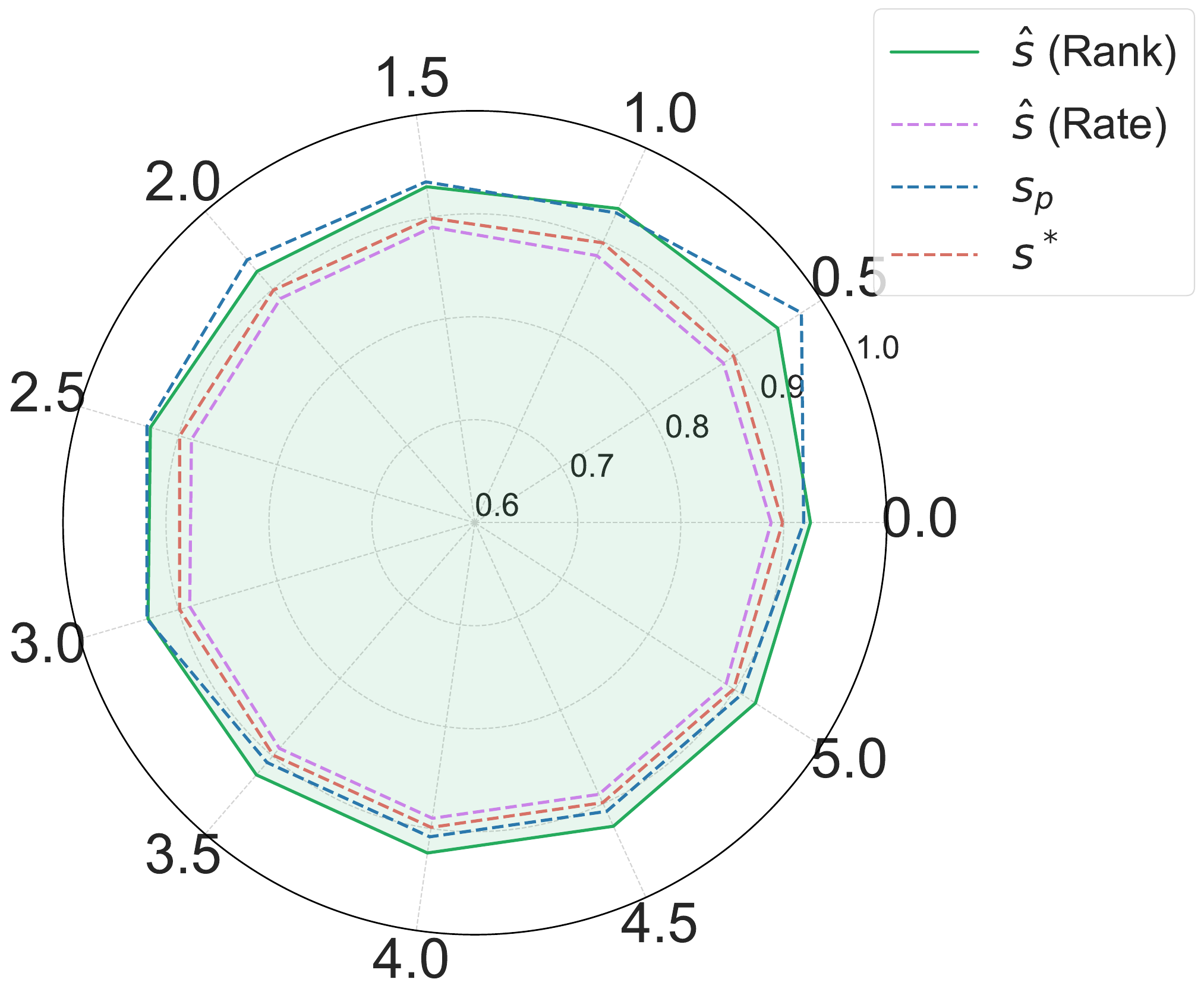}
    }
    \subfigure[LOCAL\_BLUR]{
        \includegraphics[width=0.23\textwidth]{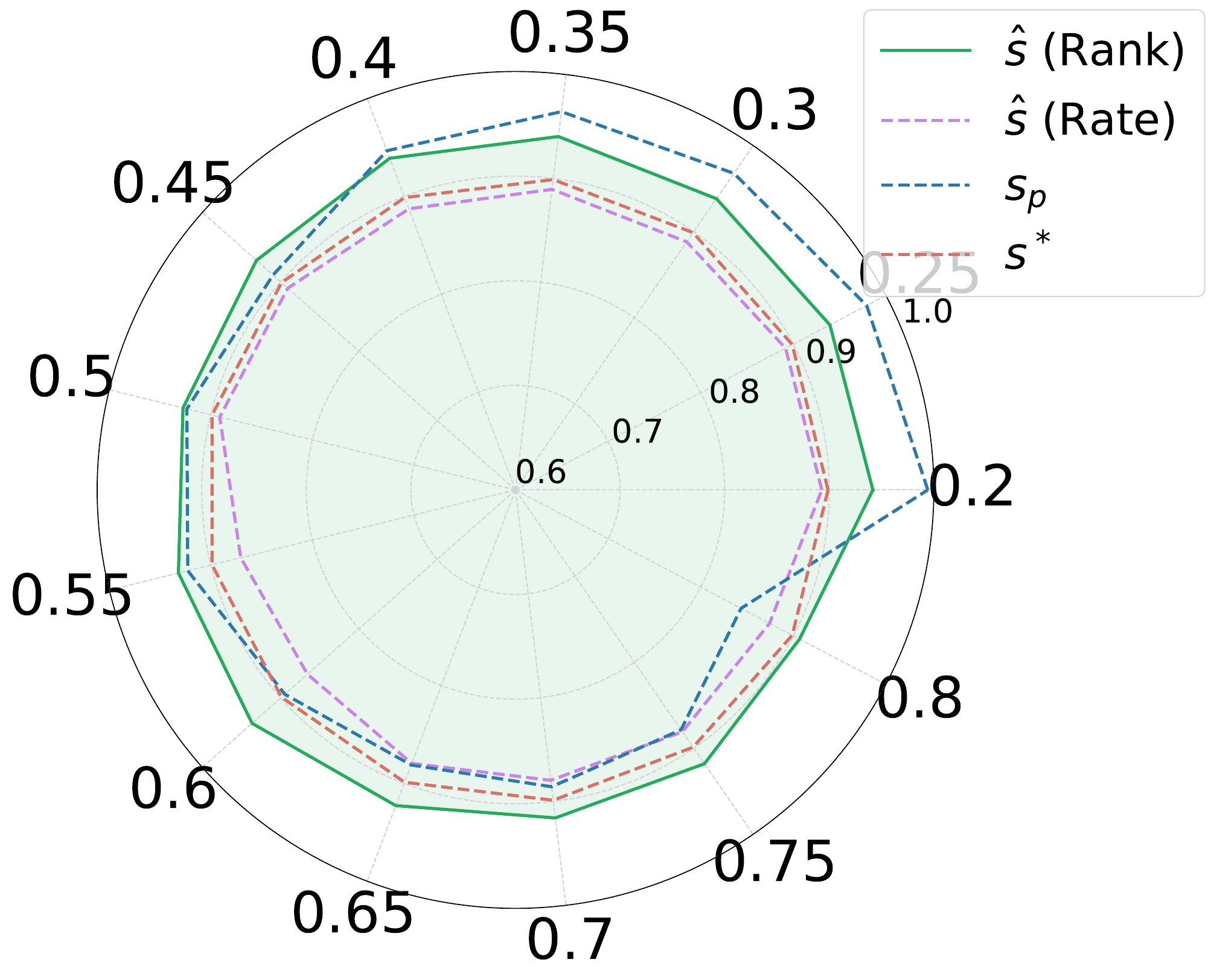}
    }
    \subfigure[LOCAL\_NOISE]{
        \includegraphics[width=0.23\textwidth]{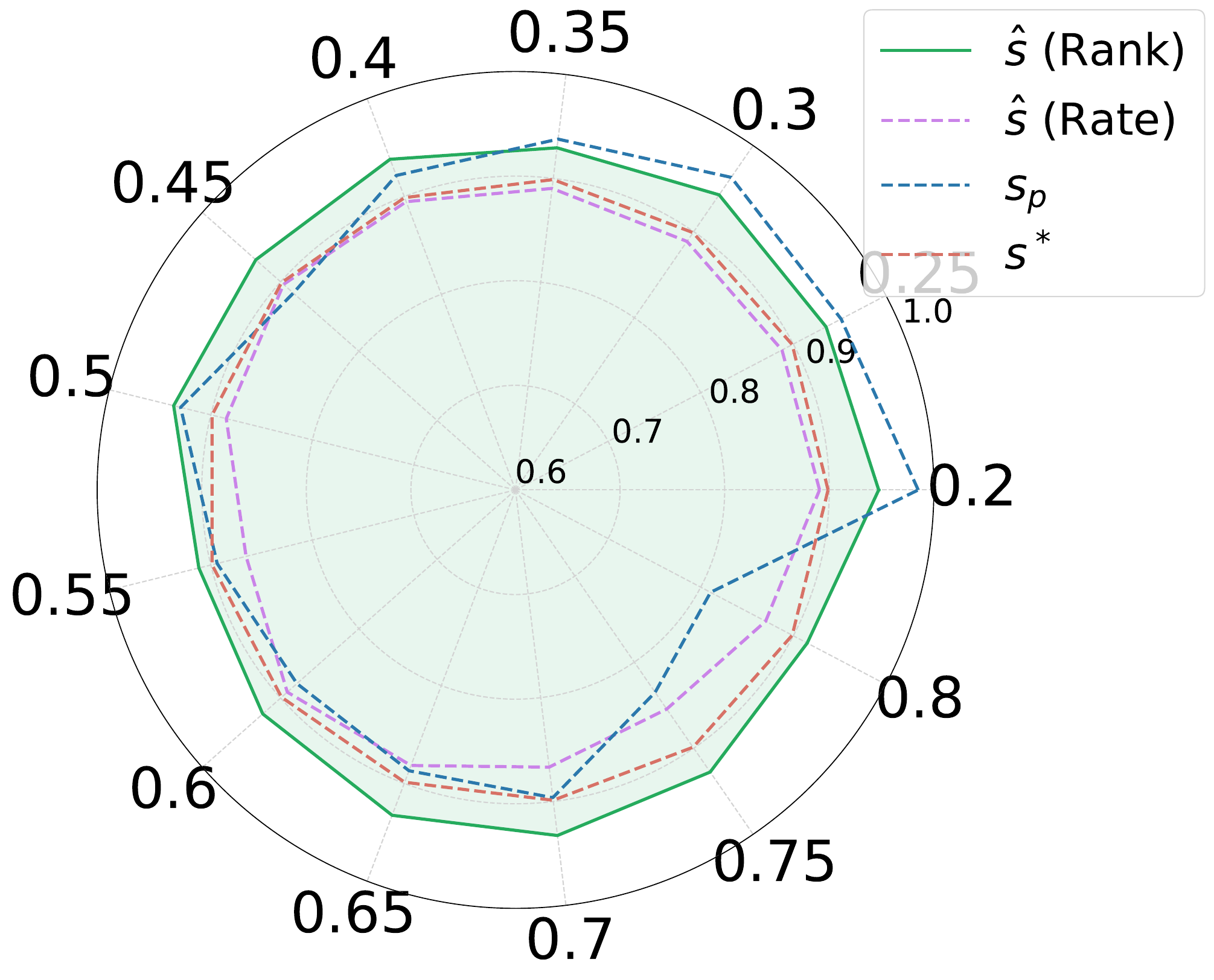}
    }
    \subfigure[WHITEOUT]{
        \includegraphics[width=0.23\textwidth]{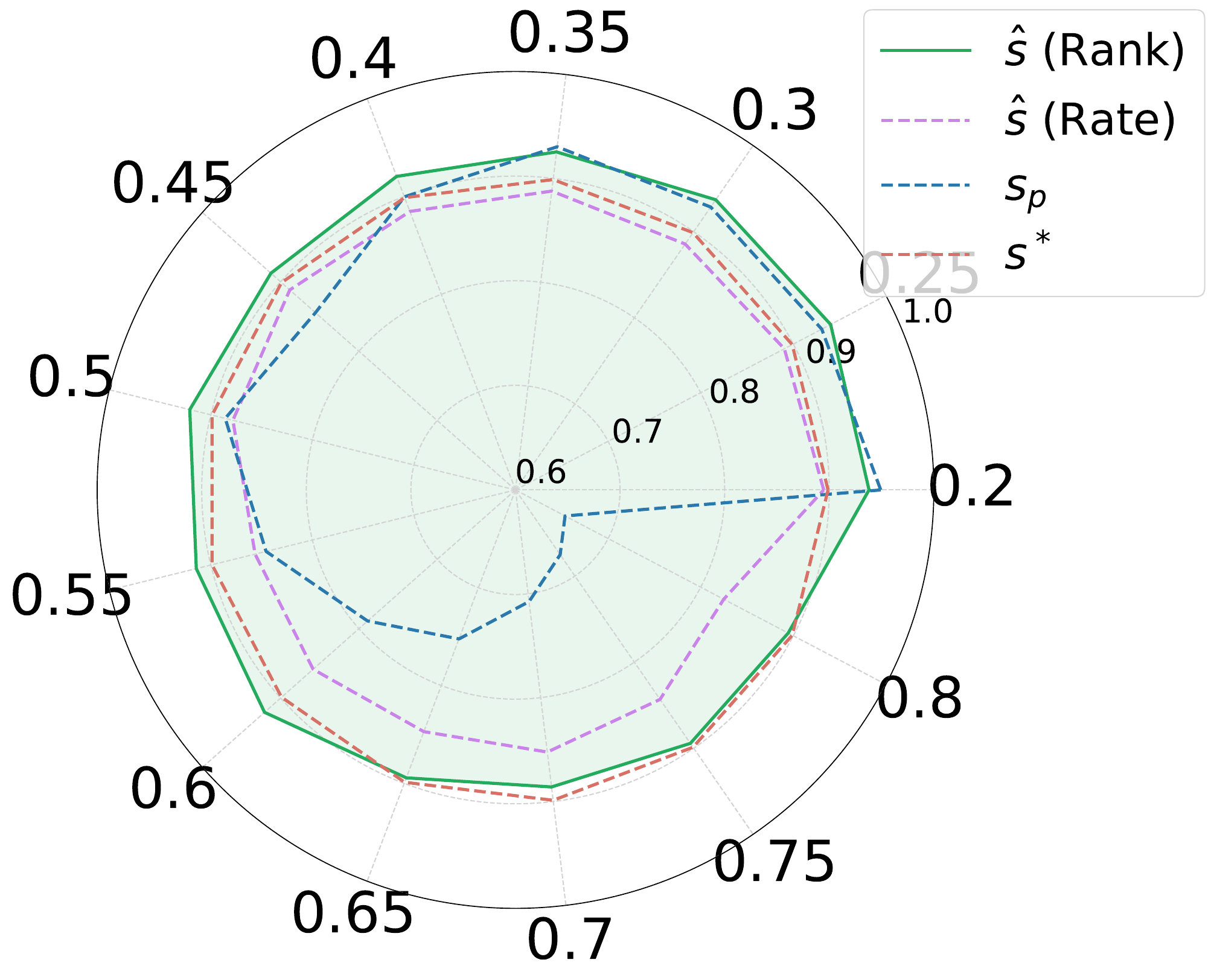}
    }
    \caption{Radar plots under different noise conditions (CrowdTCV).}
    \label{fig:radar_CrowdTCV}
\end{figure}

\begin{figure}[htbp]
    \centering
    \subfigure[BLUR]{
        \includegraphics[width=0.23\textwidth]{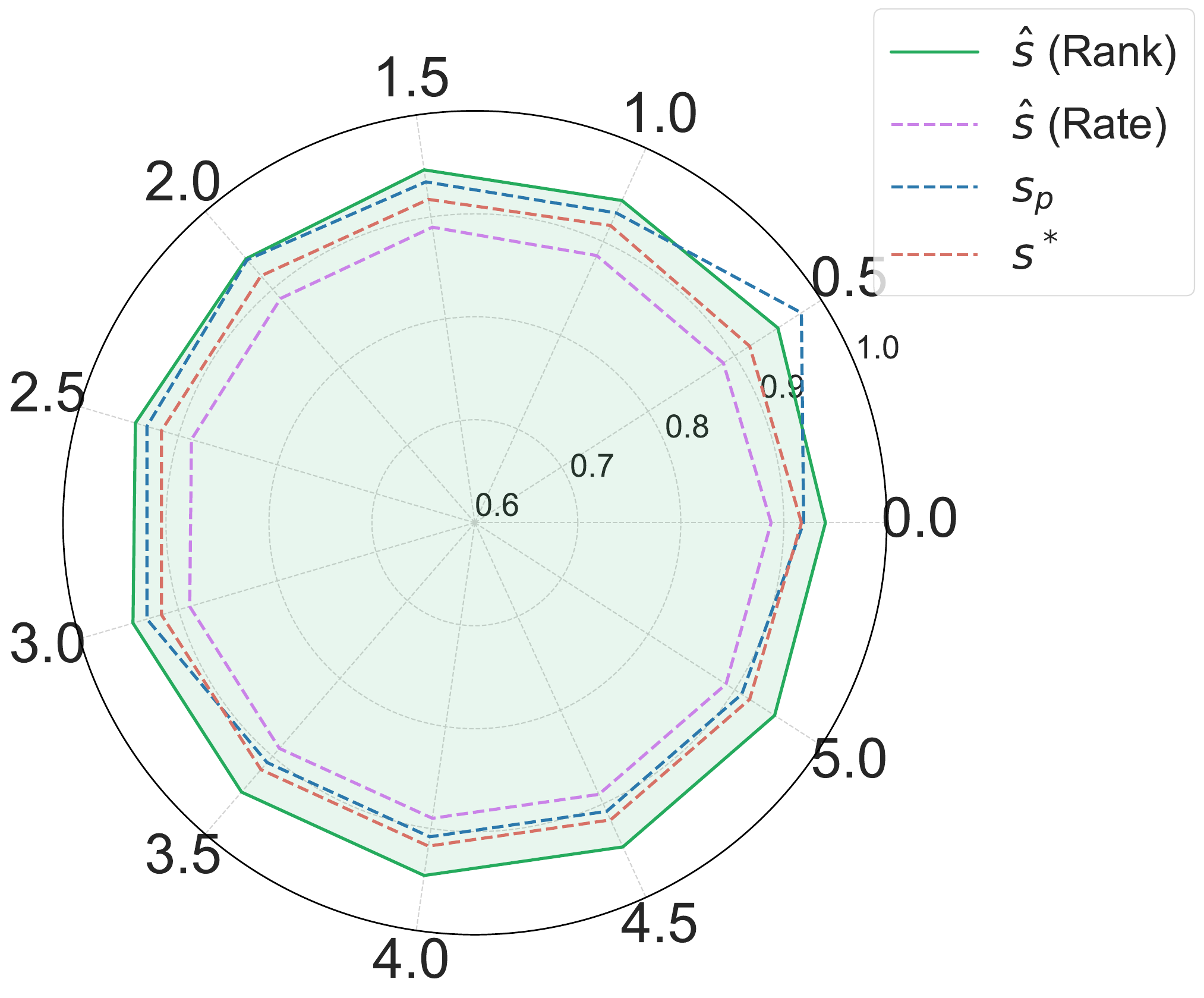}
    }
    \subfigure[LOCAL\_BLUR]{
        \includegraphics[width=0.23\textwidth]{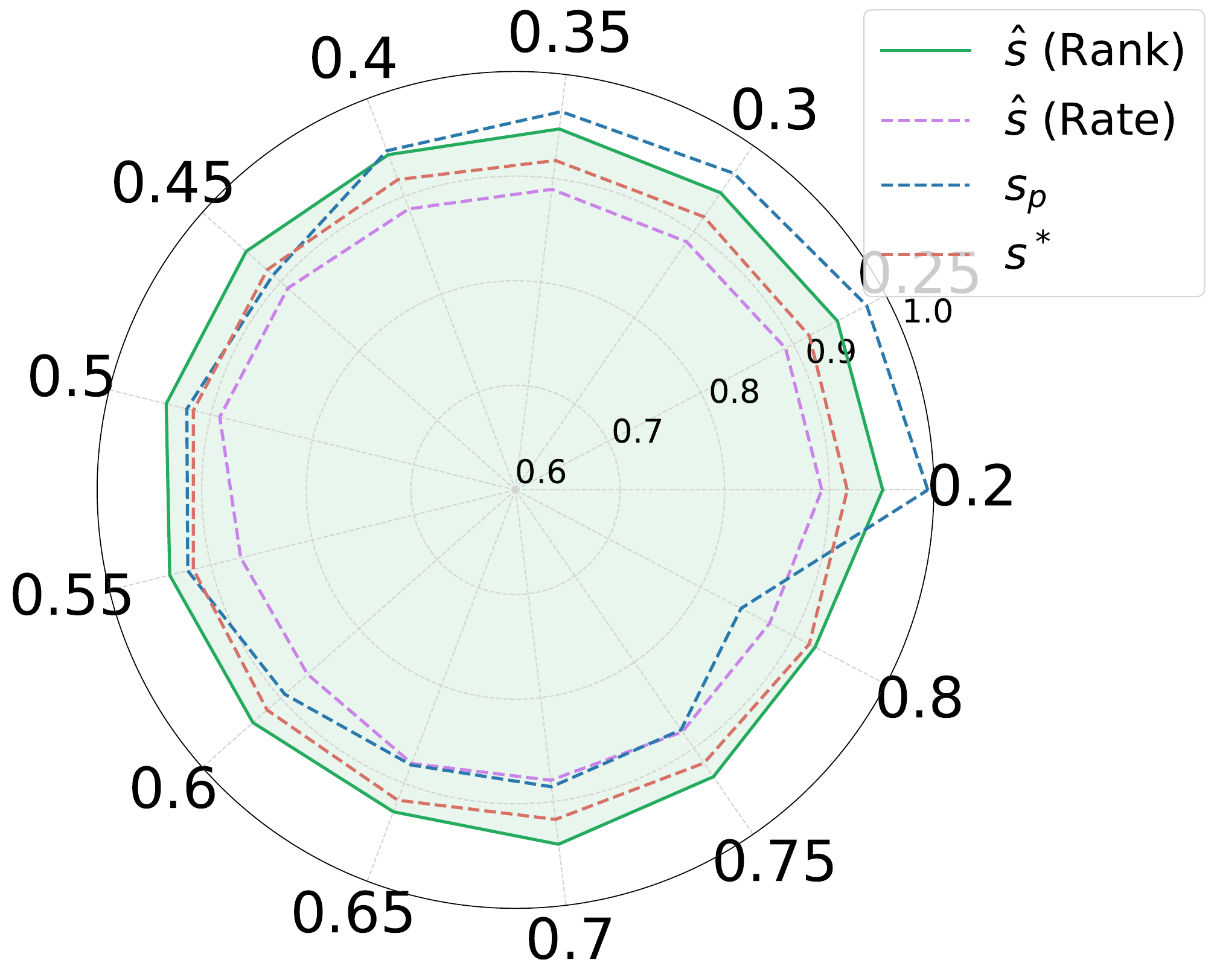}
    }
    \subfigure[LOCAL\_NOISE]{
        \includegraphics[width=0.23\textwidth]{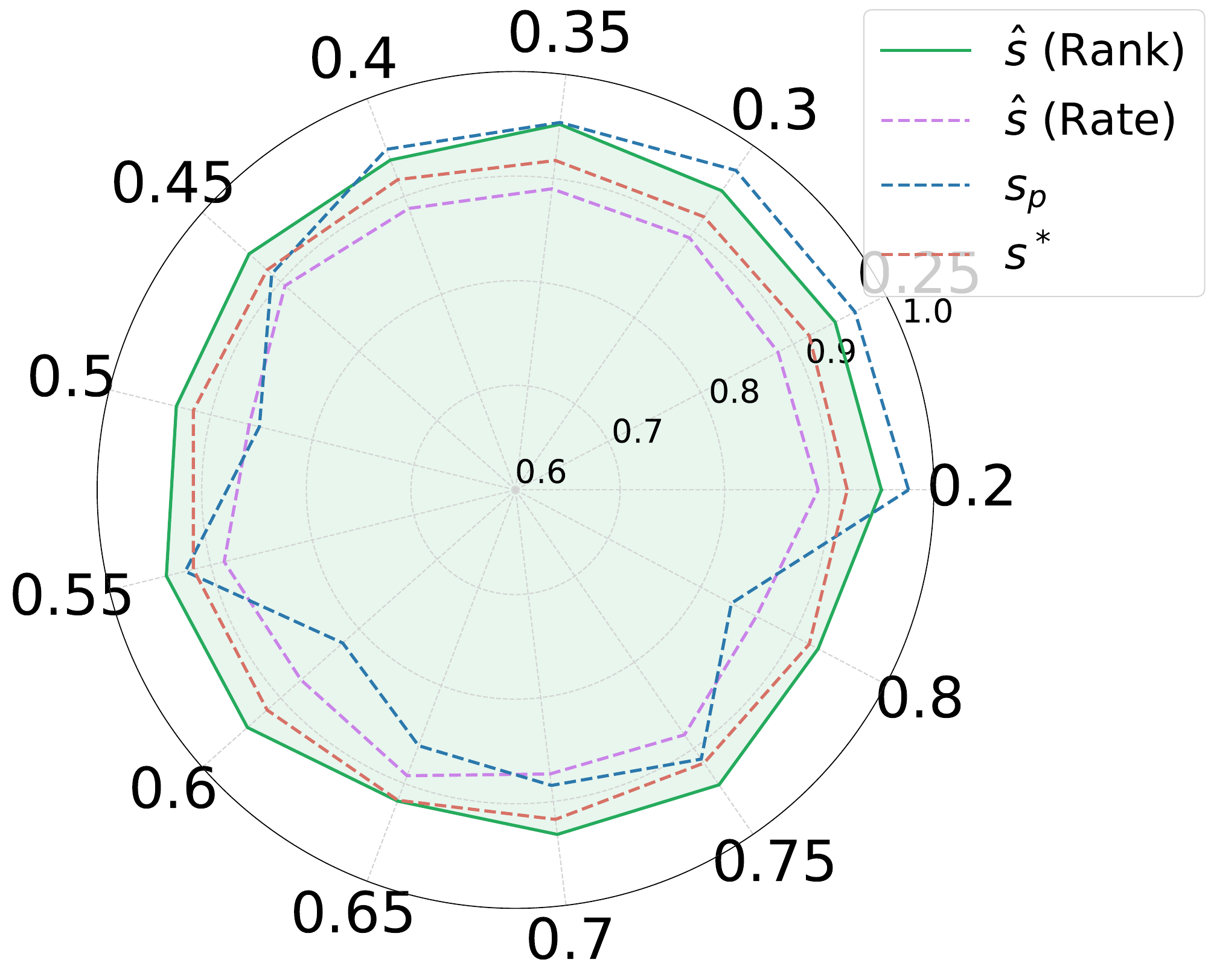}
    }
    \subfigure[WHITEOUT]{
        \includegraphics[width=0.23\textwidth]{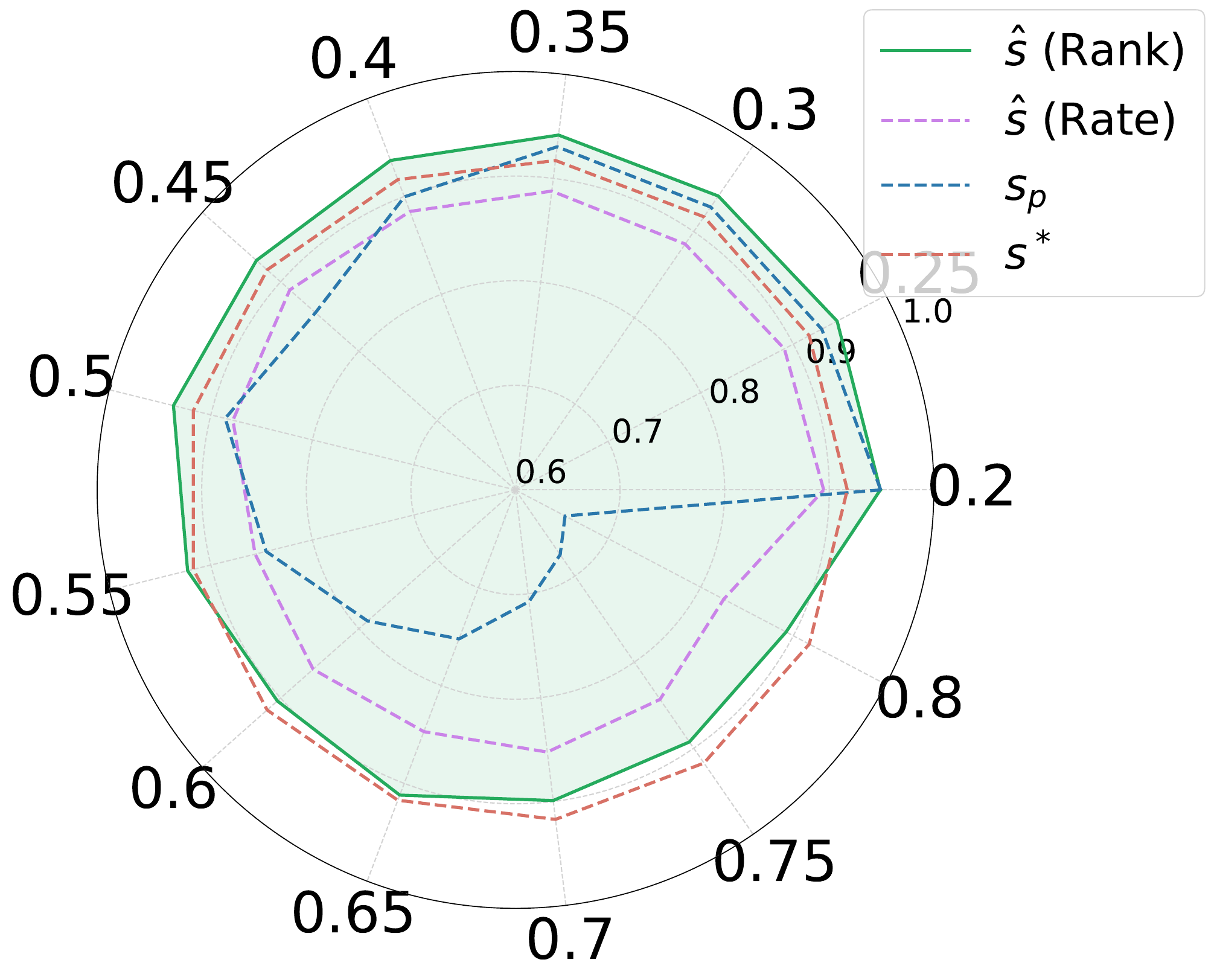}
    }
    \caption{Radar plots under different noise conditions (BTL).}
    \label{fig:radar_BTL}
\end{figure}

\begin{figure}[htbp]
    \centering
    \subfigure[BLUR]{
        \includegraphics[width=0.23\textwidth]{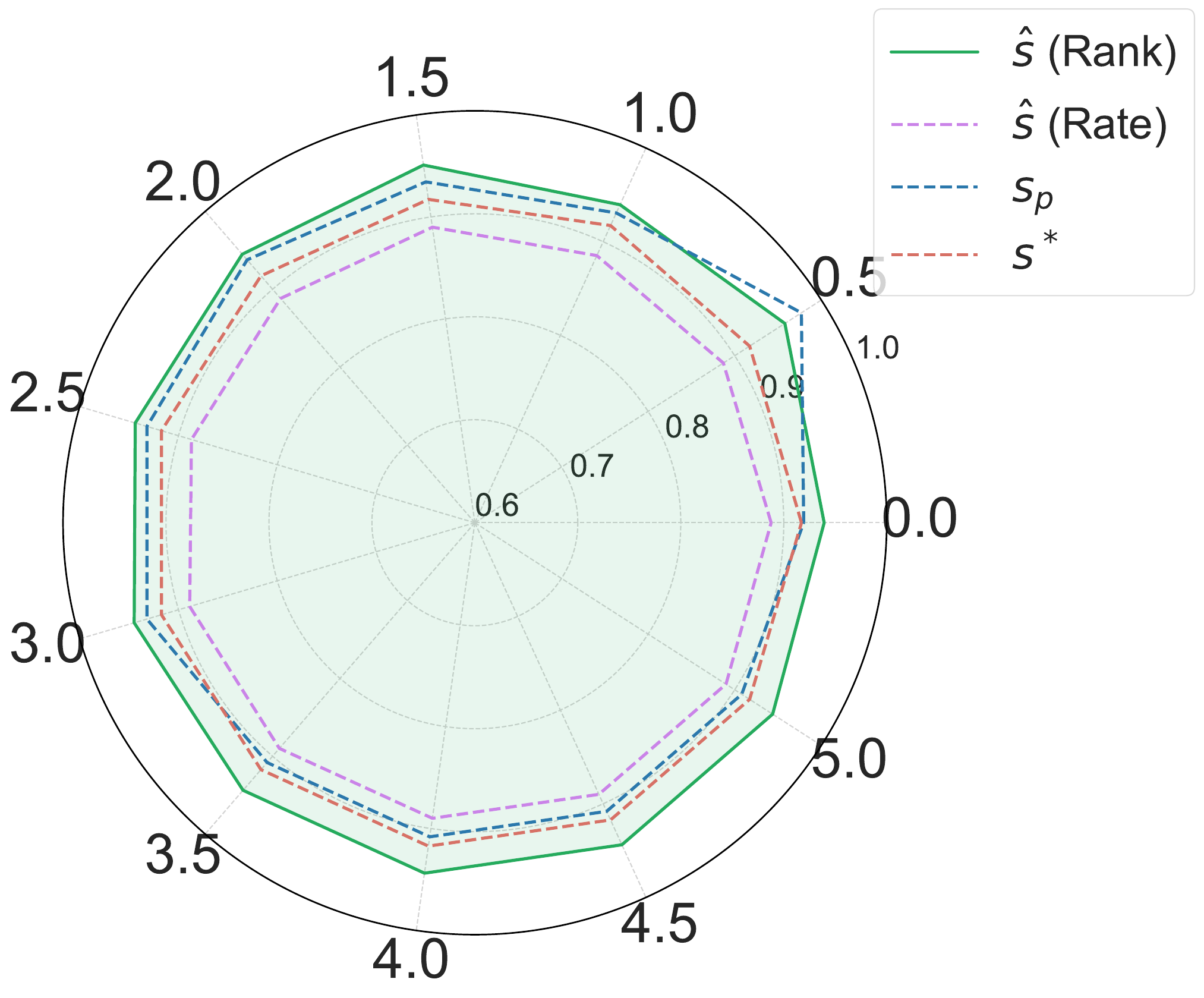}
    }
    \subfigure[LOCAL\_BLUR]{
        \includegraphics[width=0.23\textwidth]{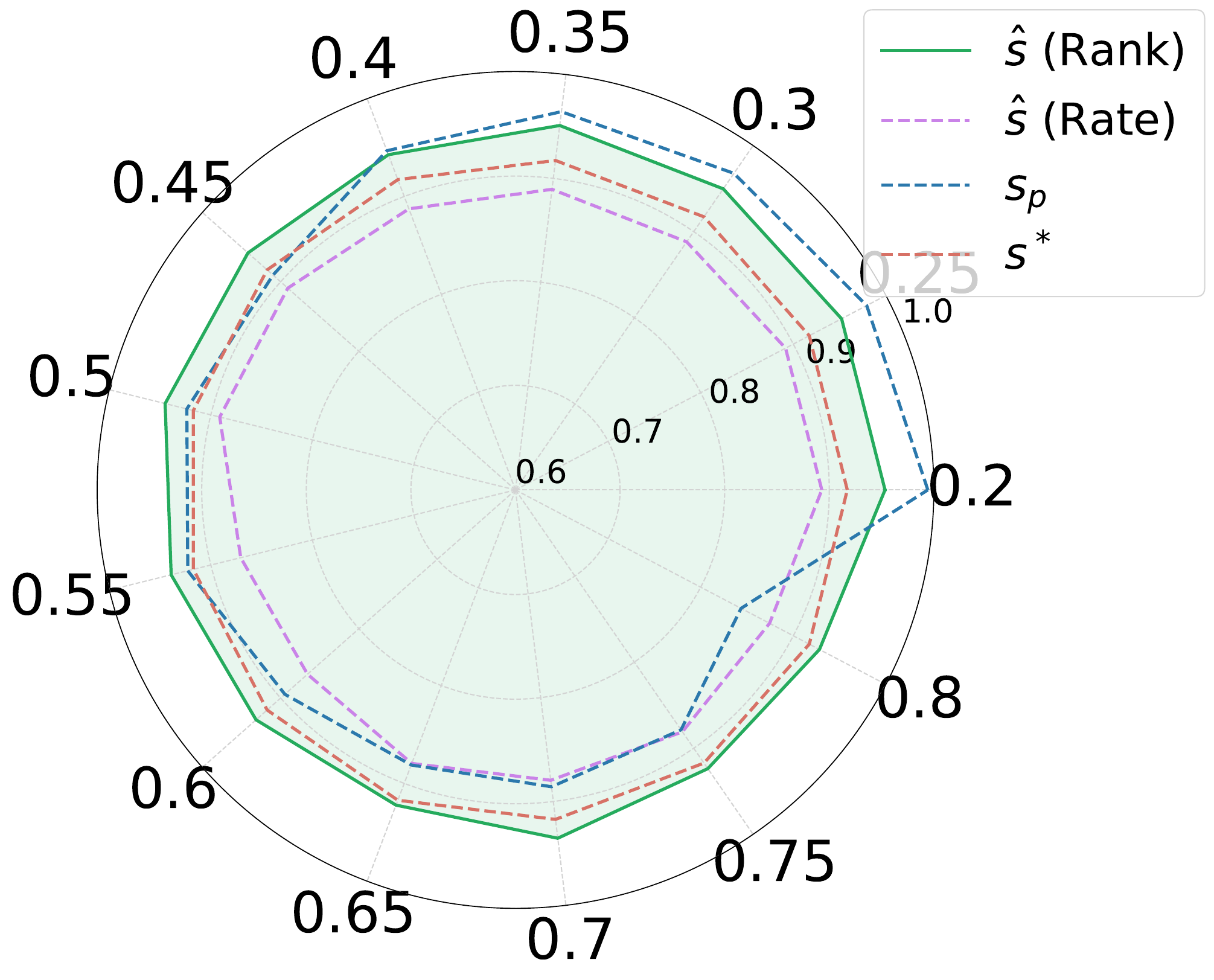}
    }
    \subfigure[LOCAL\_NOISE]{
        \includegraphics[width=0.23\textwidth]{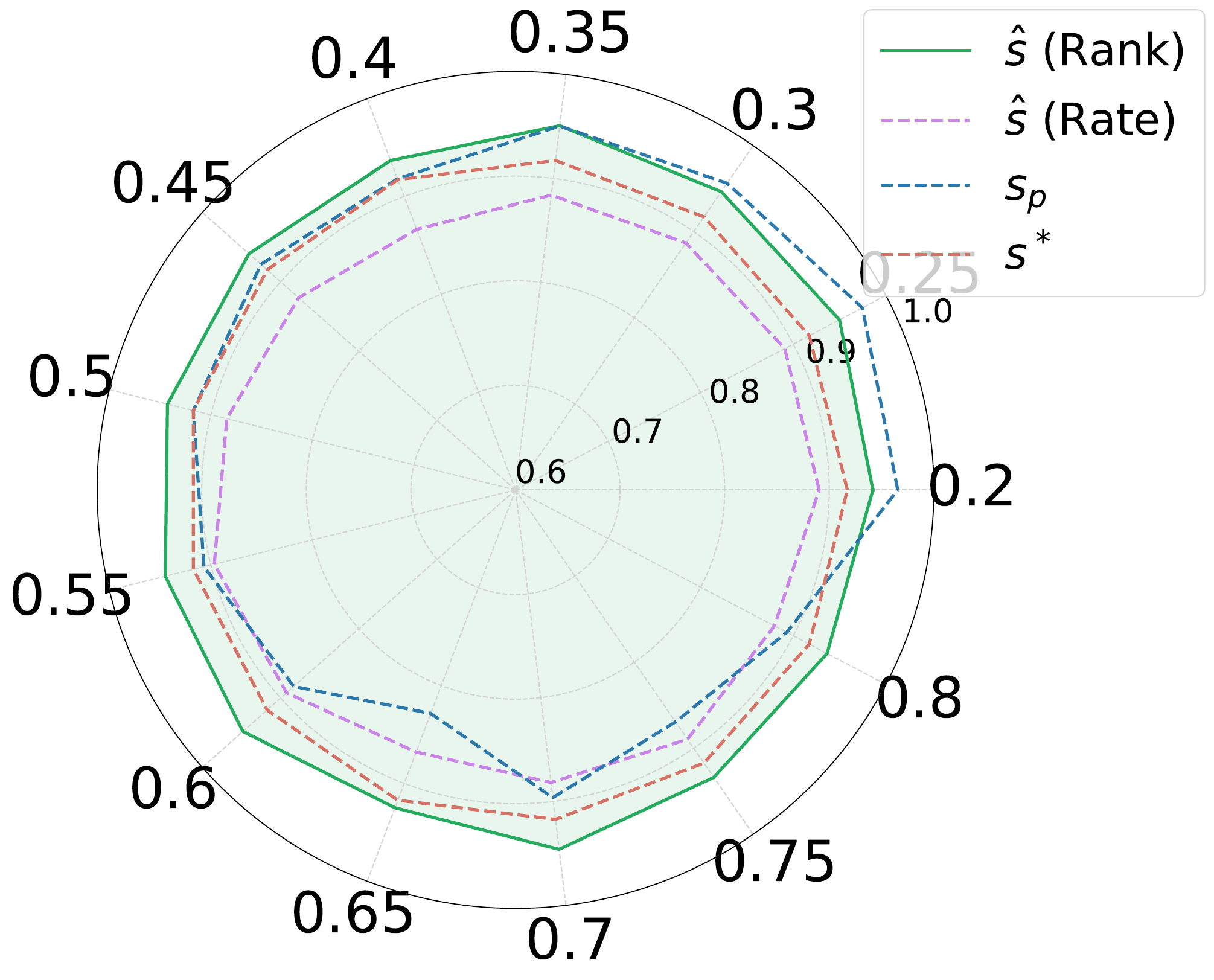}
    }
    \subfigure[WHITEOUT]{
        \includegraphics[width=0.23\textwidth]{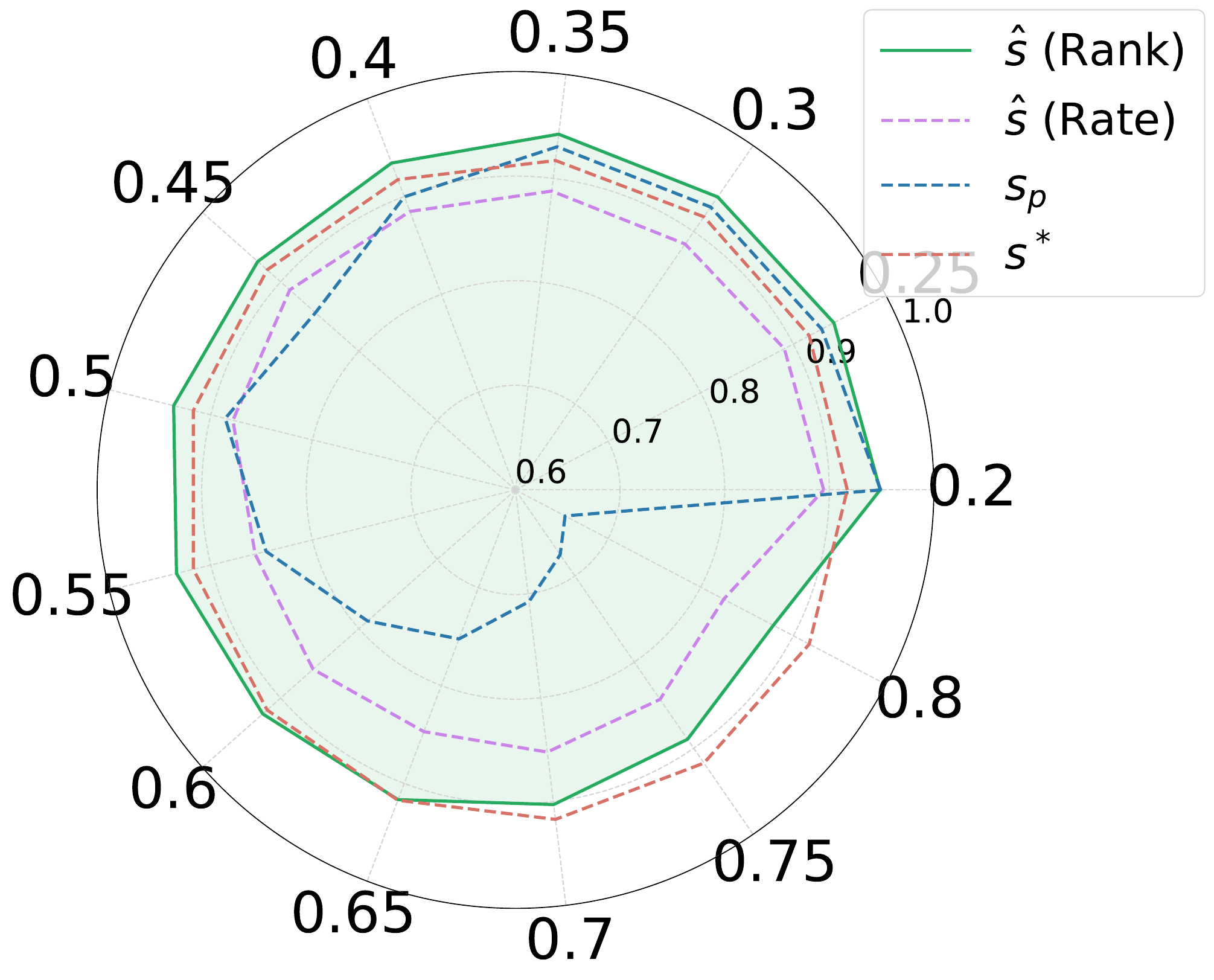}
    }
    \caption{Radar plots under different noise conditions (TCV).}
    \label{fig:radar_TCV}
\end{figure}

\begin{figure}[htbp]
    \centering
    \subfigure[BLUR]{
        \includegraphics[width=0.23\textwidth]{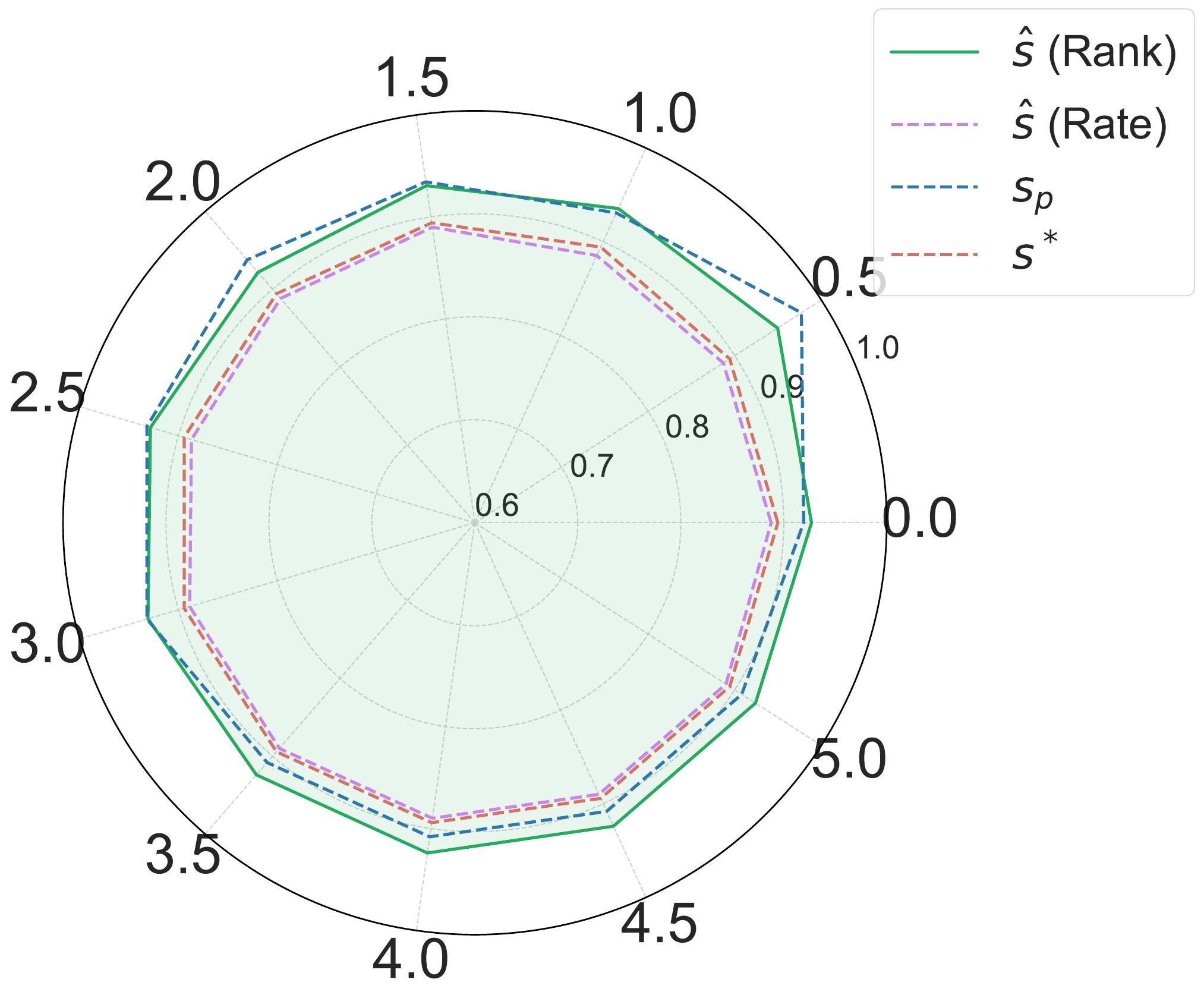}
    }
    \subfigure[LOCAL\_BLUR]{
        \includegraphics[width=0.23\textwidth]{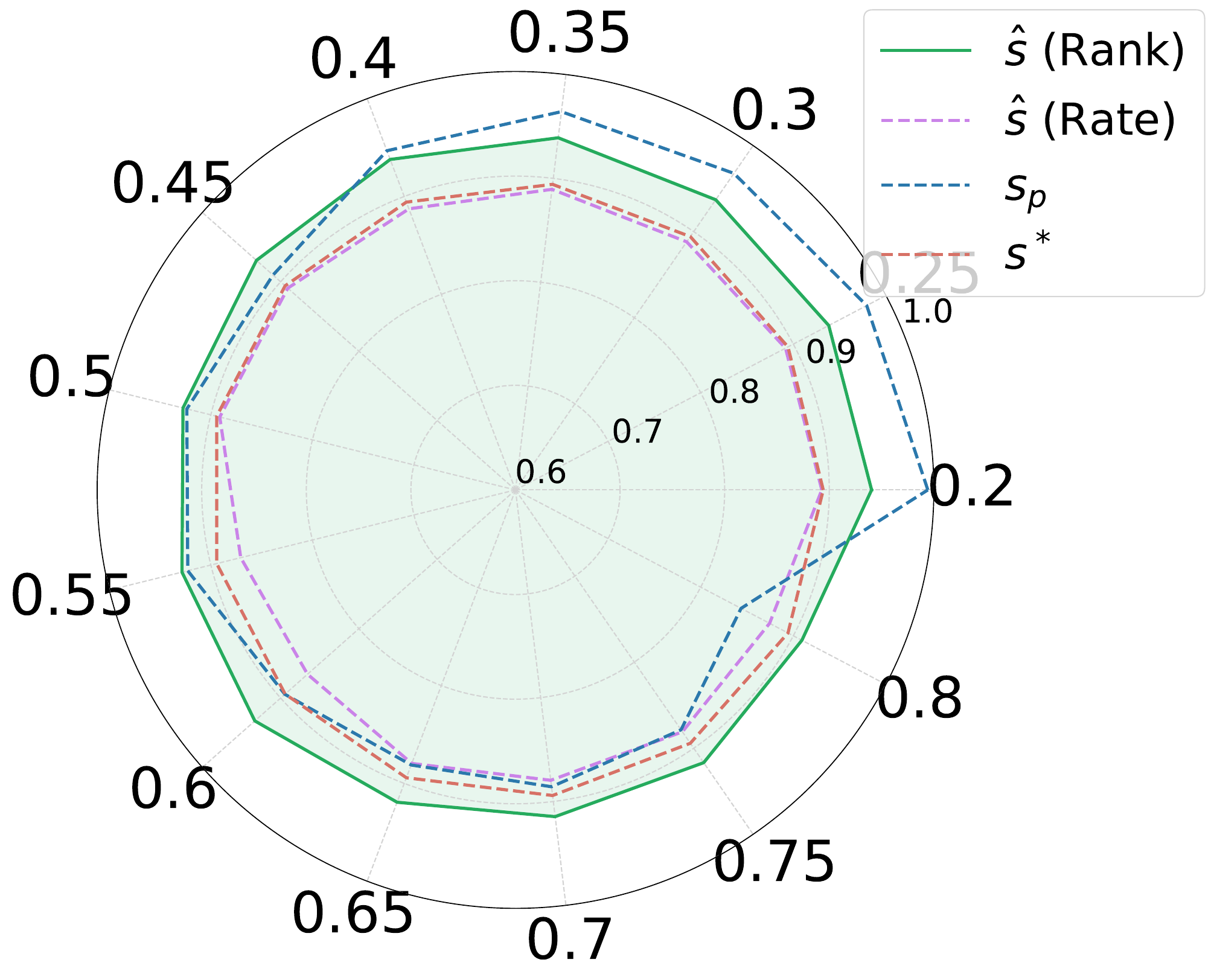}
    }
    \subfigure[LOCAL\_NOISE]{
        \includegraphics[width=0.23\textwidth]{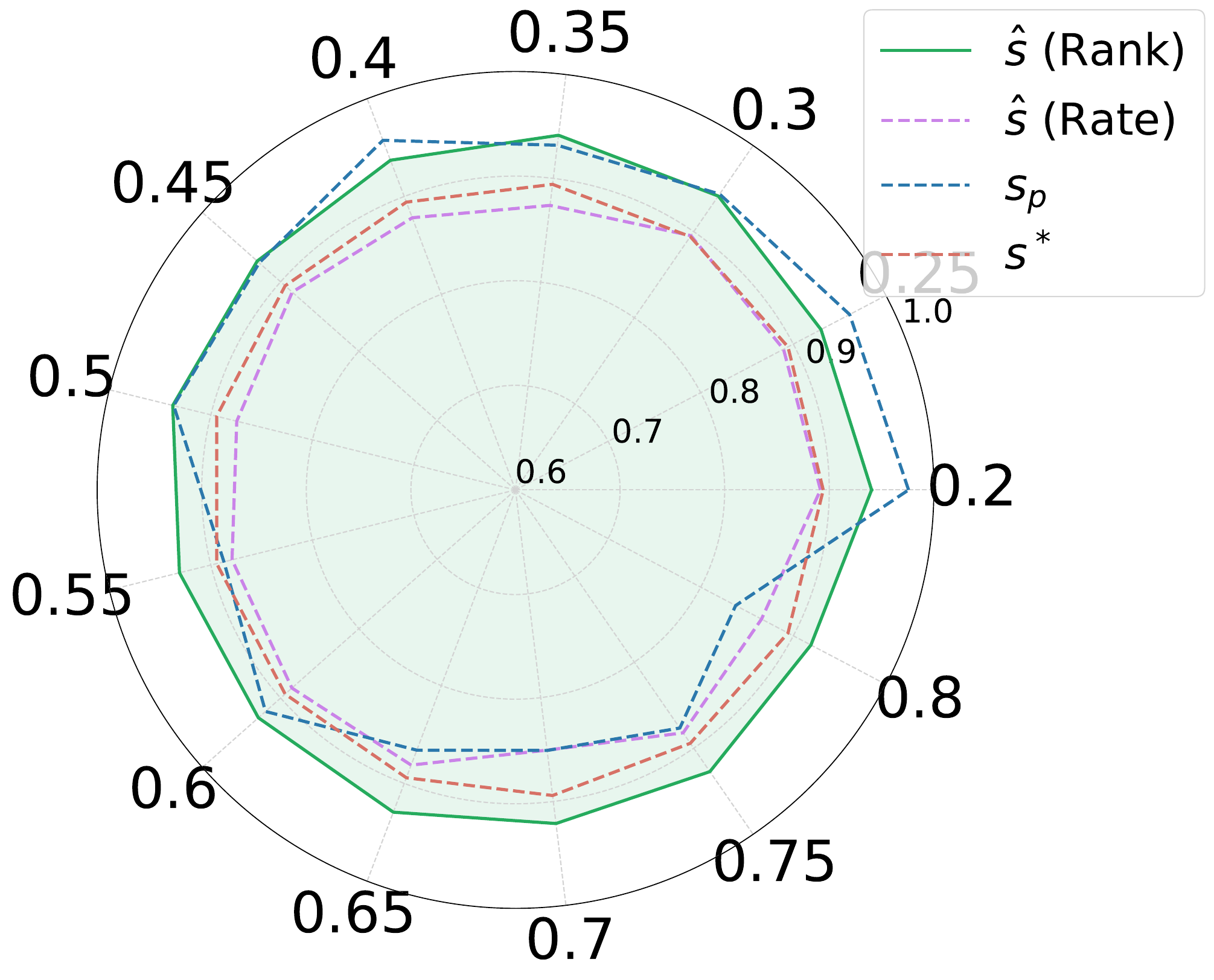}
    }
    \subfigure[WHITEOUT]{
        \includegraphics[width=0.23\textwidth]{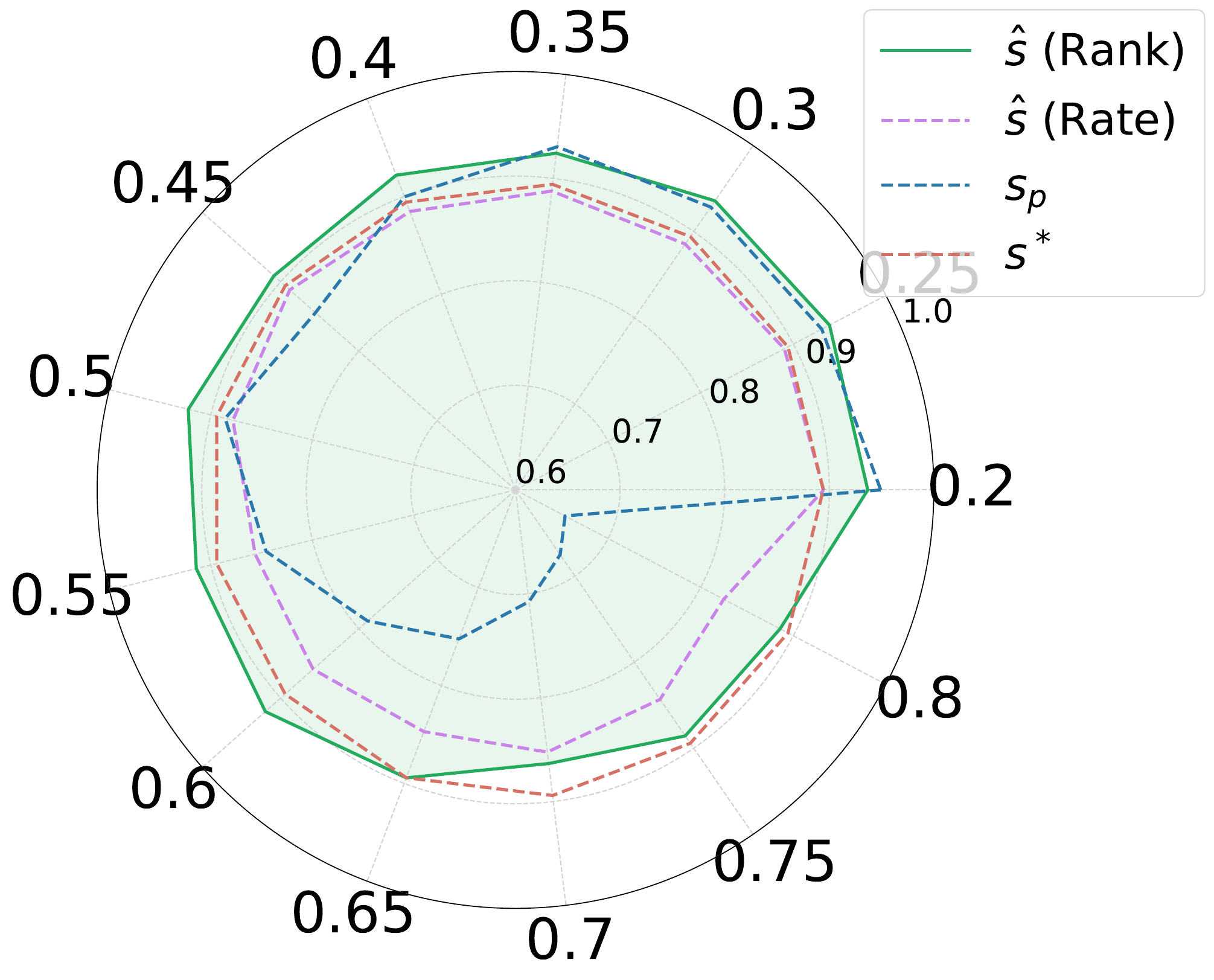}
    }
    \caption{Radar plots under different noise conditions (CrowdBT).}
    \label{fig:radar_CrowdBT}
\end{figure}

% \newpage

\end{document}